\documentclass{article}

\PassOptionsToPackage{numbers, sort&compress}{natbib}

\usepackage[final]{neurips_2024}

\usepackage[utf8]{inputenc} % allow utf-8 input
\usepackage[T1]{fontenc}    % use 8-bit T1 fonts
\usepackage{hyperref}       % hyperlinks
\usepackage{url}            % simple URL typesetting
\usepackage{booktabs}       % professional-quality tables
\usepackage{amsfonts}       % blackboard math symbols
\usepackage{nicefrac}       % compact symbols for 1/2, etc.
\usepackage{microtype}      % microtypography
\usepackage{xcolor}         % colors

\usepackage{amsmath}
\usepackage{amsthm}
\usepackage{amssymb}
\usepackage{natbib}
\usepackage{graphicx}
\usepackage{wrapfig}

\usepackage{algorithm}
\usepackage[noend]{algorithmic}

\usepackage{shortcuts}
\usepackage{xspace}

\usepackage{multirow}

\usepackage[compact]{titlesec}
\newcommand{\cvar}{\op{CVaR}}
\newcommand{\pit}{\pi_{\textsf{t}}}
\newcommand{\epsQComp}{\eps_{\normalfont\textsf{QComp}}}
\newcommand{\epsWReal}{\eps_{\normalfont\textsf{WReal}}}
\newcommand{\epsWComp}{\eps_{\normalfont\textsf{WComp}}}
\newcommand{\TV}{\normalfont\textsf{TV}}
\newcommand{\Trob}{\Tcal_{\normalfont\textsf{rob}}}
\newcommand{\QR}{\textsf{QR}}
\newcommand{\QRerr}{{\normalfont\text{err}}_{\normalfont\textsf{QR}}}
\newcommand{\RobustFQE}{{\small\normalfont{\textsf{RobustFQE}}}\xspace}
\newcommand{\RobustMIL}{{\small\normalfont{\textsf{RobustMIL}}}\xspace}
\newcommand*\samethanks[1][\value{footnote}]{\footnotemark[#1]}

\let\citet\cite
\let\citep\cite

\title{Efficient and Sharp Off-Policy Evaluation\\in Robust Markov Decision Processes}

\author{%
  Andrew Bennett\thanks{Alphabetical order.} \\
  Morgan Stanley\\
  \texttt{andrew.bennett@morganstanley.com} \\
   \And
   Nathan Kallus\samethanks\\
  Cornell University \\
  \texttt{kallus@cornell.edu}\\
  \AND
  Miruna Oprescu\samethanks\\
  Cornell University \\
  \texttt{amo78@cornell.edu} \\
  \And
  Wen Sun\samethanks\\
  Cornell University\\
  \texttt{ws455@cornell.edu}\\
  \And
  Kaiwen Wang\samethanks\\
  Cornell University\\
  \texttt{kw437@cornell.edu} \\
}

\begin{document}

\maketitle

\begin{abstract}
  We study the evaluation of a policy under best- and worst-case perturbations to a Markov decision process (MDP), using transition observations from the original MDP, whether they are generated under the same or a different policy. This is an important problem when there is the possibility of a shift between historical and future environments, \emph{e.g.} due to unmeasured confounding, distributional shift, or an adversarial environment. We propose a perturbation model that allows changes in the transition kernel densities up to a given multiplicative factor or its reciprocal, extending the classic marginal sensitivity model (MSM) for single time-step decision-making to infinite-horizon RL. We characterize the sharp bounds on policy value under this model -- \emph{i.e.}, the tightest possible bounds based on transition observations from the original MDP -- and we study the estimation of these bounds from such transition observations. We develop an estimator with several important guarantees: it is semiparametrically efficient, and remains so even when certain necessary nuisance functions, such as worst-case Q-functions, are estimated at slow, nonparametric rates. Our estimator is also asymptotically normal, enabling straightforward statistical inference using Wald confidence intervals. Moreover, when certain nuisances are estimated inconsistently, the estimator still provides valid, albeit possibly not sharp, bounds on the policy value. We validate these properties in numerical simulations. The combination of accounting for environment shifts from train to test (robustness), being insensitive to nuisance-function estimation (orthogonality), and addressing the challenge of learning from finite samples (inference) together leads to credible and reliable policy evaluation.
\end{abstract}

\section{Introduction}\label{sec:intro}

Offline policy evaluation (OPE) from historical data is crucial in domains where active, on-policy experimentation is costly, risky, unethical, or otherwise operationally infeasible. Relevant domains range from medicine to finance and recommendation systems. However, whenever historical data is used to study future behavior, there is a concern of non-stationarity -- shift between the environment generating the data (training environment) and the environment in which a policy will be deployed (test environment).
This may occur, \emph{e.g.}, due to general distributional shifts in the environment over time, unobserved confounding in the observed historical data, or adversarial elements of the environment (such as other agents) that may react when the agent is deployed.
While standard OPE in offline reinforcement learning (ORL) accounts for the change between the logging and evaluation policies, it may overlook the fact that the Markov decision process (MDP) too has changed. While this issue is particularly critical in high-stakes domains, it is broadly appealing to understand how value shifts across different environments in any application domain.

Robust MDPs \citep{iyengar2005robust, nilim2005robust} model unknown environments by allowing an adversary to choose from any one environment in a set. Therefore, they offer a natural model for unknown environment shifts by simply considering all environments to which we could possibly shift. A variety of work addresses questions such as planning in a known robust MDP \citep{wiesemann2013robust, mannor2016robust, goyal2023robust} as well as online learning \cite{badrinath2021robust, wang2021online}.
Here we focus on a purely statistical estimation question: given observations of transitions from some unknown transition kernel, we wish to estimate the worst-case (or best-case) value of a given evaluation policy in a robust MDP, defined by a set of MDPs whose transition functions are centered around the observed transition kernel. 

This setting captures the previously studied unconfounded robust OPE problem \citep{wang2024reliable}, where the observed transition kernel corresponds to an MDP, and the observed transitions are the result of applying some logging policy within this MDP. In such cases, the goal is to estimate policy values that are robust to future changes in the MDP dynamics. However, our setting is more general in that it also captures problems where the observed transitions are confounded by some unobserved variables, in which case they do \emph{not} correspond to observations from the transition kernel of an MDP. In this case, the robust MDP and the robust policy value estimates are designed to account for worst-case (or best-case) impact of this confounding bias. In either case, as in ORL, we emphasize that we do \emph{not} know the observational MDP, and can only access it via a sample of transitions.
Furthermore, even in the simple case with no unmeasured confounding, in a notable departure from standard ORL, the problem can be difficult even if the logging and evaluation policies are the same (the usually easy on-policy setting), since the policy can induce very different visitation distributions in the original and perturbed MDPs. 

Such robust offline evaluation from transition data was considered in recent work \citep{panaganti2022robust,bruns2023robust}.
We build on this recent work by focusing the question of statistically \emph{efficient} and \emph{robust} estimation of the \emph{sharp} bounds (\ie, the tightest possible given the data). Previous work focused on evaluation using only the Q-function under the worst-case environment (in some cases under a relaxation of the adversary, leading to loose bounds). Thus, any error in its estimation translates directly to error in evaluation.
In other words, flexible nonparametric modeling of this function can mean slow rates for estimated bounds and a lack semiparametric efficiency. Moreover, without a clear understanding of the noise in the estimates, we cannot add confidence bands to the bounds, leading to bounds that are too tight.

We address these challenges by developing an orthogonalized estimation method that combines several nuisance functions: the worst-case $Q$-function, the state-visitation frequency in the worst-case environment, and a threshold function characterizing the worst-case transition kernel. Our first key result is that, to first order, our estimator behaves as a sample average using the true values of these functions without having to estimate them at all, provided we just estimate them at certain slow nonparametric rates. This ensures we not only have a $\sqrt n$-rate of estimation even when nuisances are estimated more slowly, but also that our estimator is asymptotically normal. This allows for the construction of confidence bands on the bounds, providing assurance that the true bound is captured. We further show that our asymptotic variance is in fact the minimum variance among all regular and asymptotically linear (RAL) estimators, ensuring semiparametric efficiency. Our second key result is that even if we do not estimate some of the nuisance functions correctly, we are still consistent to sharp or valid bounds. That is, even when we are biased due to misestimation of nuisances, our bias (if any) only enlarges our bounds, so they remain valid. We illustrate these guarantees numerically. Collectively, these guarantees lend substantial credibility to the bounds generated by our method.

Our contributions are summarized as follows:
\begin{enumerate}[itemsep=0.3ex,topsep=0ex,partopsep=0ex,parsep=0ex]
    \item We provide novel algorithms and analysis for learning robust $Q$-functions (\cref{sec:robust-q}) and robust visitation density ratios (\cref{sec:robust-w}) under the function approximation setting.
    \item We derive the sharp and efficient estimator for the robust policy value, which is optimal in the local-minimax sense and is the gold standard in semiparametric estimation (\cref{sec:ortho-est}).
    \item We empirically validate the efficiency and sharpness of our approach (\cref{sec:experiments}).
\end{enumerate}

\subsection{Related Works}

\textbf{Unobserved Confounding in Sequential Decision-Making.}
OPE in robust MDPs is related to OPE bounds in confounded MDPs, where the behavior policy and the transition kernel are influenced by unobserved confounders.  The constraint \cref{eq:uncertainty-set} that defines our target robust MDP aligns with the Marginal Sensitivity Model (MSM) \citep{tan2006distributional} employed in sensitivity analysis for causal inference. Yet, unlike the MSM, which limits the ratio of policy densities, our approach directly constrains the ratio of the transition kernels. Our formulation can be viewed as a generalization of the MSM from traditional two-action no-horizon causal effects (where the constrains coincide) to multi-action infinite-horizon discounted MDPs, where the next state is the ``potential outcome''. In that sense, our model essentially serves as an outcome-based sensitivity model \citep{bonvini2022sensitivity}. This distinction is crucial as it enables our model to subsume the policy-based MSM in cases where the policy is confounded. Nonetheless, the reverse does not hold, and the policy-based MSM does not imply a transition kernel-based MSM for $A>2$. This point is further corroborated by \citet{bruns2023robust}, who explore the policy-based MSM within confounded MDPs and obtain \emph{non-sharp} identification bounds when $A > 2$. In contrast, our approach yields \emph{sharp} identification in general, regardless of the number of actions and without placing assumptions on the behavior policy, which may or may not be confounded.

\citet{bruns2021model} also considered an MSM-like model in the transition kernel but their formulation assumes $A=2$.
\citet{kallus2020confounding} operates under the setting of \citet{bruns2023robust} and required tabular states.
We note that all these works including ours considers \emph{i.i.d.} confounders at each step, which translates to a robust MDP with $(s,a)$-rectangularity and ensures that the worst-case problem is still an MDP rather than a POMDP.
The importance of this assumption was verified by \citet{namkoong2020off}, who showed that without it, the non-memoryless confounder can create exponential-in-horizon changes in value.

\textbf{Neyman Orthogonality and Semiparametric Efficient Estimation.} 
We leverage a body of research focusing on learning with nuisances functions (e.g., Q-functions) that we need to estimate from data but are not the primary target (e.g., policy value).  Much of this research \citep[among others]{chernozhukov2018double, foster2023orthogonal,van2011cross,semenova2021debiased,belloni2017program,chernozhukov2018generic} aims to identify Neyman-orthogonal estimators, which are first order orthogonal (insensitive) to nuisance errors. This literature is tightly linked to the semiparametric efficient estimation literature since Neyman-orthogonal scores can arise naturally from efficient influence functions \citep{ichimura2022influence,schick1986asymptotically}.  Going beyond the no-horizon causal inference setting, some explore such estimators in off-policy sequential-decisions contexts \citep{kallus2020double,lewis2021double,chernozhukov2022automatic,kennedy2019nonparametric,laan2003unified}. Notably, \citet{kallus2022efficiently} derive efficient influence functions and orthogonal estimation in standard, non-robust OPE in infinite-horizon RL, which coincides with our unconfounded no-uncertainty case ($\Lambda=1$).

Moving beyond point-identified settings, some works explore orthogonality and efficiency for partial identification and sensitivity analysis.
In the causal inference literature, efficient/orthogonal estimation in the no-horizon setting has been studied extensively under several sensitivity models \citep{dorn2021doubly, bonvini2022sensitivity, chernozhukov2022long, oprescu2023b}. Closest to our work is \citet{dorn2021doubly} who provide an orthogonal estimator and convergence rates under the MSM \citep{tan2006distributional}, which coincides with our setting under $\gamma=1$. In the sequential setting, \citep{namkoong2020off} considers confounding at a single time step under the MSM, but their estimator is not orthogonal when the quantile function is unknown. \citet{bruns2023robust} provide a fitted-Q-iteration learner with an orthogonalized loss function, but not orthogonal/efficient estimates of worst-case policy value.

\section{Preliminaries}\label{sec:prelim}
We consider an MDP with state space $\Scal$, action space $\Acal$, transition kernel $P(s'\mid s,a)$, reward function $r(s,a)\in[0,1]$ and initial state distribution $d_1\in\Delta(\Scal)$. We do not require $\Scal$ or $\Acal$ to be finite. We assume $r$ and $d_1$ are known for simplicity, and it is standard to extend our analysis to when they are unknown.
We are given a dataset $\Dcal$ of $n$ \emph{i.i.d.} tuples $(s_i,a_i,r_i,s_i')$ such that $(s_i,a_i)\sim\nu$, $s_i'\sim P(\cdot\mid s,a)$ and $r_i=r(s_i,a_i)$, where $\nu$ is an arbitrary data-generating distribution. 
For discount factor $\gamma\in[0,1)$, let the $Q$ function be the discounted cumulative rewards under a policy $\pi:\Scal\to\Acal$,
$Q_{\pi,P}(s,a) = \EE_{\pi,P}\bracks{ \sum_{t=0}^\infty \gamma^t r_t(s_t,a_t) \mid s_1=s, a_1=a}$.
Similarly, define the value function as $V_{\pi,P}(s)=Q_{\pi,P}(s,\pi)$, where we use the notation $f(s,\pi):=\EE_{a\sim\pi(s)}[f(s,a)]$ for any function $f:\Scal\times\Acal\to\RR$.

We are interested in estimating the value of a fixed target policy $\pit$ (a.k.a. evaluation policy) in an unobserved MDP with a feasible perturbed transition kernel $U$. We say $U$ is a feasible perturbation of the observed, nominal kernel $P$ if for all $s,a,s'$: we have
\begin{equation}\label{eq:uncertainty-set}
  \textstyle\Lambda^{-1}(s,a)\leq \frac{\diff U(s'\mid s,a)}{\diff P(s'\mid s,a)}\leq \Lambda(s,a)
\end{equation}
where $\Lambda(s,a)\in[1,\infty)$ is a sensitivity parameter chosen by the practitioner. On the extremes, $\Lambda=1$ corresponds to no-confounding (\ie, classic OPE setting) and $\Lambda=\infty$ corresponds to maximal-confounding (\ie, worst or best outcome).
We denote the set of all feasible perturbations of $P$ by $\Ucal(P)$, which is an $s,a$-rectangular set \citep{mannor2016robust}. 
We define the best- and worst-case $Q$ functions of $\pit$ as
\begin{equation}\textstyle
  Q^+(s,a) := \sup_{U\in\Ucal(P)}Q_{\pit,U}(s,a);
  \quad\quad Q^-(s,a) := \inf_{U\in\Ucal(P)}Q_{\pit,U}(s,a).\label{eq:robust-q-primal-def}
\end{equation}
Thus, the goal of this paper is to estimate the best- and worst-case value of $\pit$ at the initial state,
\begin{equation}
  V_{d_1}^\pm := (1-\gamma)\EE_{s_1\sim d_1}[V^\pm(s_1)]. \label{eq:robust-V-def}
\end{equation}
where $V^\pm(s)=\EE_{a\sim\pit(s)}[Q^\pm(s,a)]$ and the $\pm$ symbol signals that an equation should be read twice, once with $\pm=+$ and once with $\pm=-$.
For clarity, we focus the discussion in the main text on estimating the worst-case policy value, $V_{d_1}^-$. We provide a similar analysis for policy values under best-case perturbations ($V_{d_1}^+$) in \pref{app:best-case-results}.

Compared to standard OPE, robust OPE is more challenging since the best- and worst-case transition kernels $U^\pm$ are unobserved as our dataset $\Dcal$ is generated under $P$. For example, standard OPE is easy in the on-policy case \ie, if $\Dcal$ were generated by $\pit$, but our problem is still ``off-data'' and non-trivial.

\textbf{Discounted Visitation Distributions.}
For any transition kernel $U$, define the discounted visitation distribution of $\pit$ under $U$ as: $d^{\pit,\infty}_{d_1,U}(s):= (1-\gamma)\sum_{h=1}^\infty\gamma^{h-1}d^{\pit,h}_{d_1,U}(s)$, where $d^{\pit,h}_{d_1,U}(s)$ is the probability of reaching state $s$ in the Markov chain induced by $U$ and policy $\pit$ starting from $d_1(\cdot)$.
We use $d^{-,\infty}$ as shorthand for $d^{\pit,\infty}_{d_1,U^-}$, where $U^-$ denotes the worst-case kernel in $\Ucal(P)$.

\textbf{Bellman-type Operators.}
For any function $f:\Scal\times\Acal\to\RR$ and transition kernel $U$, recall the Bellman operator is defined as $\Tcal_U f(s,a):=r(s,a)+\gamma\EE_U[ f(s',\pit)\mid s,a ]$. For robust OPE, we define the following robust analog $\Trob^+f(s,a):=r(s,a)+\gamma\sup_{U\in\Ucal(P)}\EE_U[ f(s',\pit)\mid s,a ]$ and $\Trob^-f(s,a):=r(s,a)+\gamma\inf_{U\in\Ucal(P)}\EE_U[ f(s',\pit)\mid s,a ]$. 
Moreover, we define $\Jcal_U f(s,a) := \gamma\EE_U[ f(s',\pit)\mid s,a ]-f(s,a)$. For any linear operator $\Tcal$, also let $\Tcal'$ denote its adjoint: that is, for all $f,g\in L_2(\nu)$, $\langle f,\Tcal g\rangle = \langle\Tcal' f,g\rangle$, where $\langle\cdot,\cdot\rangle$ is the inner product in $L_2(\nu)$.

\begin{wrapfigure}{R}{0.40\textwidth}
  \vspace{-1em}
  \centering
  \includegraphics[width=0.40\textwidth]{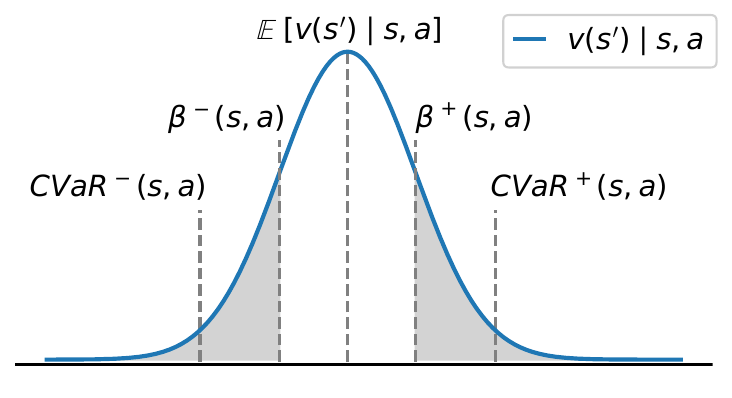}
  \vspace{-1em}
  \caption{Lower and upper CVaRs and quantiles ($\beta$) of $v(s')\mid s,a$ distribution.}
  \label{fig:dist-measures}
  \vspace{-3em}
\end{wrapfigure}

\textbf{Conditional Value-at Risk (CVaR).}
For a random variable $X$, its upper/lower CVaRs at level $\tau\in[0,1]$ is defined as the average outcome of the upper/lower $\tau$-fraction of cases, and are formally defined as follows \citep{rockafellar2002conditional}:
\begin{align*}\textstyle
  \cvar_\tau^+(X) &:= \min_{b\in\RR}\braces{b+\tau^{-1}\EE[(X-b)_+]},\\
  ~\cvar_\tau^-(X) &:= \max_{b\in\RR}\braces{b+\tau^{-1}\EE[(X-b)_-]},
\end{align*}
where $y_+:=\max(0, y)$ and $y_- := \min(0, y)$ for $y\in\RR$.
The optima are attained at the upper/lower $\tau$-th quantile of $X$ which we denote as $\beta_\tau^+(X)/\beta_\tau^-(X)$, \ie, 
\begin{equation*}\textstyle
  \cvar_\tau^+(X) := \beta^+_\tau(X) + \tau^{-1}\EE[ (X-\beta^+_\tau(X))_+],
  ~\cvar_\tau^-(X) := \beta^-_\tau(X) + \tau^{-1}\EE[ (X-\beta^-_\tau(X))_-].
\end{equation*} 
If $X$ has a cumulative distribution function (CDF) which is differentiable at $\beta_\tau^\pm(X)$, its CVaRs simplify to $\cvar_\tau^+(X) = \EE[X\mid X\geq \beta_\tau^+(X)]$ and $\cvar_\tau^-(X) = \EE[X\mid X\leq \beta_\tau^-(X)]$. 
In the paper, $\tau$ will often be set to $(\Lambda+1)^{-1}\in[0,0.5]$.
% For a CDF $F$ we define the lower $\alpha$-quantile $\beta^-_{\alpha,F}=\inf\{\beta:F(\beta)\geq\alpha\}$ and the upper $\alpha$-quantile $\beta^+_{\alpha,F}=\beta^-_{1-\alpha,F}$. 

\textbf{Notations.} 
We use $x\lesssim y$ to mean that $x\leq Cy$ holds for some universal constant $C$. The indicator function $\II[p]$ takes value $1$ if $p$ is true and $0$ otherwise. For a measure $\mu$, we let $\|f\|_\mu:=(\EE_\mu|f(X)|^2)^{1/2}$ denote the $L_2$ norm of $f$, provided it exists. When $\mu$ is clear from context, we also use $\|f\|_p:=(\EE |f(X)|^p)^{1/p}$ to denote the $L_p$ norm of $f$ and $\|f\|_{p,n}:=(\EE_n |f(X)|^p)^{1/p}$ to denote the empirical analog. 
For a data sample of size $n$, we define the empirical mean as $\EE_n[f(X)]=\frac{1}{n}\sum_{i=1}^n f(x_i)$. 
For a nuisance function $f$, we reserve $f^*$ as its true value and $\widehat{f}$ as the learned value from data.
Moreover, we employ $+$ and $-$ to denote functions corresponding to best- and worst-case bounds, respectively. See \pref{app:notation} for a comprehensive notation table.

\subsection{Background: Non-robust OPE}
We provide a quick primer on the double RL (DRL) estimator for classic OPE in non-robust MDPs \citep{kallus2020double}, which combines estimates of the $Q$-function and density ratio $w$ to achieve orthogonality, double robustness and semiparametric efficiency.
This sets the stage for our orthogonal estimator in \pref{sec:ortho-est}, which generalizes DRL to robust MDPs by incorporating the robust $Q$-function and density ratio in the worst-case MDP, as described in \pref{sec:robust-q} and \pref{sec:robust-w} respectively.

The DRL estimator involves two nuisances: (1) $q$, for which the oracle (true value) is the $Q$-function of the target policy $Q^{\pit}$, and (2) $w$, for which the oracle is the density ratio of the target policy's visitation distribution and the data distribution $w^{\pit} = \nicefrac{\diff d^{\pit,\infty}_{d_1,P}}{\diff\nu}$. In this section, let $\eta=(w,q)$ denote the DRL nuisances (outside this section, we use $\eta$ to denote our robust estimator's nuisances) and let $\eta^\star=(w^{\pit},Q^{\pit})$ denote their true values, then the recentered efficient influence function (EIF) of $V^{\pit}_{d_1}$ in non-robust MDPs is given by:
\begin{equation*}
    \textstyle\psi^{\textsf{DRL}}(s,a,s';w,q) = V^{\pit}_{d_1} + w(s,a)\cdot(r(s,a)+\gamma q(s',\pit)-q(s,a)).
\end{equation*} 
The DRL estimator uses cross-fitting to learn nuisances $\wh\eta^{[k]}$ on all data excluding the $k$-th fold $\Dcal^k$, for $k=1,2,\dots,K$ and estimates the OPE value via:
\begin{equation*}
    \textstyle\wh V^{\textsf{DRL}}_{d_1} = \frac1n\sum_{k=1}^K\sum_{(s,a,s')\in\Dcal^k}\psi^{\textsf{DRL}}(s,a,s';\wh\eta^{[k]}).
\end{equation*}
As we will see, this paves the way for the EIF of the robust value (\cref{thm:eif-main}) and our orthogonal estimator (\cref{alg:ortho-ope}).
There are two main guarantees for DRL: double robustness and semiparametric efficiency.
Let $r^w_{n}$ and $r^q_{n}$ be rate functions depending on $n=|\Dcal|$ such that $\|\wh q^{[k]}-Q^{\pit}\|_2\leq r^q_n$ and $\|\wh w^{[k]}-w^{\pit}\|_2\leq r^w_n$. Then, DRL enjoys $|\wh V^{\textsf{DRL}}_{d_1}-V^{\pit}_{d_1}|\leq O_p(n^{-1/2} + r^w_n r^q_n)$, which confers the algorithm double robustness properties. Moreover, if $\Sigma^{\textsf{ope}}$ is the efficiency bound (\ie, minimum achievable asymptotic variance among RAL estimators in nonparametric models for $(s,a,s')$), then $\sqrt{n}(\wh V^{\textsf{DRL}}_{d_1}-V^{\pit}_{d_1})\overset{d}{\to}\Ncal(0,\Sigma^{\textsf{ope}})$. We seek similar  guarantees for our orthogonal robust estimator. 

\section{Robust \texorpdfstring{$Q$}{Q}-Function Estimation with Fitted-\texorpdfstring{$Q$}{Q} Evaluation}\label{sec:robust-q}
In this section, we identify the robust $Q$-function using the robust Bellman equation and then derive convergence rates for iteratively minimizing the robust Bellman error. 

\subsection{Identification of the worst-case \texorpdfstring{$Q$}{Q}-function}
The robust worst-case $Q$-function of $\pit$, denoted as $Q^-$, satisfies the robust Bellman equation 
$Q^-(s,a) = \Trob^- Q^-(s,a), \forall s,a$ since the uncertainty set $\Ucal(P)$ factorizes over $s,a$ \citep{iyengar2005robust}.
While these equations may seem intractable due to the $\inf$ in the definition of $\Trob^-$, \citet{bruns2023robust} showed that $\Trob^-$ has a closed form solution in terms of the CVaR under the \emph{observed} kernel $P$.
\begin{restatable}{lemma}{identificationQ}\label{lem:identification-q}
Set $\tau(s,a)=\prns{\Lambda(s,a)+1}^{-1}$. Then, for any $q:\Scal\times\Acal\to\RR$,
\begin{align*}
    \Trob^- q(s,a) 
    &= r(s,a) + \gamma\Lambda^{-1}(s,a)\EE[ v(s')\mid s,a ]
    +\gamma(1-\Lambda^{-1}(s,a))\cvar^-_{\tau(s,a)}[ v(s')\mid s,a ],
\end{align*}
where $v(s')=\EE_{a'\sim\pit(s')}[q(s',a')]$, and $\EE,\cvar_\tau$ are under the observed kernel $P(\cdot\mid s,a)$.
\end{restatable}
\pref{lem:identification-q} implies that $Q^-$ is identified via the following equation of observable distributions: 
\begin{equation*}
    Q^-(s,a) = r(s,a) + \gamma\Lambda^{-1}(s,a)\EE[ Q^-(s',\pit)\mid s,a ]+\gamma(1-\Lambda^{-1}(s,a))\cvar^-_{\tau(s,a)}[ Q^-(s',\pit)\mid s,a ].
\end{equation*}
Under no confounding ($\Lambda(s,a)=1$), this recovers the classic Bellman equation.

\subsection{Estimating the Robust \texorpdfstring{$Q$}{Q}-Function with Robust FQE} \label{sec:q-est}
\begin{algorithm}[!t]
\caption{\RobustFQE: Iterative fitting for estimating $Q^-$ and $\beta^-_\tau$.}\label{alg:robust-fqe}
\begin{algorithmic}[1]
    \STATE \textbf{Input: } Number of iterations $M$, Dataset $\Dcal$ of size $n$, $Q$-function class $\Qcal$.
    \STATE Initialize $\wh v_0^-(s')=0.$
    \FOR{$i=1,2,\dots,M$}
        \STATE Set $\Dcal_i = \Dcal[ni/M:n(i+1)/M]$.
        \STATE On the first half of $\Dcal_i$, estimate the $\tau(s,a)$ lower quantile of $\wh v_{i-1}^-(s'),s'\sim P(\cdot\mid s,a)$. 
        \\ Let $\wh\beta^-_i(s,a)$ denote the learned lower quantiles from the quantile regression oracle $\QR$. \label{line:fqe-quantile-est}
        \STATE Using the second half of $\Dcal_i$, solve the empirical robust Bellman equation by minimizing squared prediction error for the pseudo-outcome: \label{line:fqe-regression-step}
        \begin{align*}\textstyle
            \textstyle\wh q_i^-&\textstyle\gets\argmin_{q\in\Qcal} \frac{1}{|D_i|/2}\sum_{(s,a,s')\in \Dcal_i[|\Dcal_i|/2+1:]}[(y^-(s,a,s')-q(s,a))^2], \,\,\,\text{where }
            \\&\textstyle y^-(s,a,s')= r(s,a)+\gamma\Lambda^{-1}(s,a)\wh v_{i-1}^-(s')
            +\gamma(1-\Lambda^{-1}(s,a))\\&\phantom{===}\times( \wh\beta_i^-(s,a)+\tau^{-1}(s,a)( \EE_{a'\sim\pit(s')}[(\wh q_i^-(s',a')-\wh\beta_i^-(s,a) )_-] ).
        \end{align*}
    \ENDFOR
    \STATE\textbf{Output: } $\wh q_M^-,\wh\beta_{M}^-$.
\end{algorithmic}
\end{algorithm}

In this section, we estimate $Q^-$ via an iterative fitting algorithm based on fitted Q-evaluation (FQE) \citep{munos2008finite}.
Our algorithm \RobustFQE (\pref{alg:robust-fqe}) proceeds for $M$ iterations with two main steps in each iteration $i$. First, in \pref{line:fqe-quantile-est}, we estimate the lower-quantile of $\wh v_{i-1}(s')\mid s,a$. Here, we assume access to an oracle $\QR$ for quantile regression, which is a well-established problem, allowing for the use of various existing algorithms.
Second, in \pref{line:fqe-regression-step}, we solve the tractable robust Bellman equation in \pref{lem:identification-q} with the CVaR term estimated by its orthogonal estimating equation with the learned quantiles \citep{olma2021nonparametric}. By orthogonally estimating $\cvar$, we achieve second-order dependence on the quantile estimation errors from the first step. Next, we minimize the mean squared error using a general function class, $\Qcal \subset \Scal \times \Acal \mapsto [0, (1-\gamma)^{-1}]$. % which can include linear functions or neural networks.

To enable convergence guarantees, we make two assumptions. 
First, we assume that the quantile regression oracle has a specific convergence rate, which can be guaranteed under certain smoothness conditions \citep{bhattacharya1990kernel,takeuchi2006nonparametric,meinshausen2006quantile,el2009local,cevid2020distributional,racine2017nonparametric,elie2022random}.
Distributional RL may also be modified to learn quantiles of the next state value and have shown benefits in practice \citep{dabney2018distributional,dabney2018implicit} and in theory \citep{wang2023benefits,wang2024more,wang2024central,ayoub2024switching}.

\begin{assumption}[QR Oracle]\label{asm:qr-oracle}
For any $v:\Scal\mapsto[0,(1-\gamma)^{-1}]$, let the true $\tau(s,a)$-quantile of $v(s'),s'\sim P(s,a)$ be denoted by $\beta^v_\tau(s,a)$. Given a dataset $\Dcal_{\QR}$, we assume $\QR$ outputs estimates $\wh\beta_v$ with bounded $\ell_\infty$ error: for any $\delta$, w.p. $1-\delta$, $\|\wh\beta_q-\beta_\tau^q\|_\infty<\QRerr(|\Dcal_{\QR}|,\delta)$.
\end{assumption}

The second assumption is completeness under the robust Bellman $\Trob^-$. Completeness is a standard assumption in algorithms based on temporal-difference learning and without it, fitted-Q can diverge or converge to suboptimal fixed points \citep{tsitsiklis1996analysis,kolter2011fixed}.
\begin{assumption}[Completeness]\label{asm:robust-bc}
For all $q\in\Qcal$, we have $\Trob^- q\in\Qcal$.
\end{assumption}
We note that the current proofs of \citet{panaganti2022robust,bruns2023robust} require a stronger completeness: $\Tcal_\beta q\in\Qcal$ for all $q\in\Qcal$ and feasible $\beta$. 
We circumvent the need for the stronger ``all-$\beta$'' completeness by bounding model misspecification of least squares regression with second order error in the quantile regression. 

Finally, we express our bounds with the critical radius $\eps_n^\Qcal$, a standard tool for deriving fast rates in statistics; see \pref{app:critical-radius} for a summary.
Also, we denote the standard concentrability coefficient with $C^-_{d_1}:=\nm{\nicefrac{\diff d^{-,\infty}_{\mu}}{\diff d_1}}_\infty$, a standard and necessary quantity for OPE.
\begin{restatable}{theorem}{robustFQE}\label{thm:robust-fqe}
Let $\eps_n^\Qcal$ denote the critical radius of $\Qcal$.
Under \cref{asm:qr-oracle,asm:robust-bc}, \RobustFQE ensures that for any $\delta\in(0,1)$, w.p. $1-\delta$, 
\begin{align*}
    &\textstyle\nm*{ \wh q^-_M - Q^- }_{d_1} \lesssim (1-\gamma)^{-2}\prns*{\sqrt{C^-_{d_1}}\cdot\eps_n^\Qcal+\QRerr^2(n/2M,\delta/2M)},~~\text{and}
    \\&\textstyle\abs{(1-\gamma)\EE_{d_1}[ \wh v_M^-(s_1) ]- V_{d_1}^-} \lesssim \gamma^M + (1-\gamma)^{-1}\prns*{\sqrt{C^-_{d_1}}\cdot\eps_n^\Qcal+\QRerr^2(n/2M,\delta/2M)}.
\end{align*}
\end{restatable}
For parametric classes (\eg, finite or linear), the critical radius converges at the standard $\wt\Ocal(n^{-1/2})$ rate.
Due to the orthogonal estimation of CVaR, we benefit from a favorable second-order dependence on $\QRerr$ which allows for quantile regression to converge at slower $\wt\Ocal(n^{-1/4})$ rates.
The main disadvantage of this direct approach is that it converges at a slow sub-$\sqrt{n}$ rate if $\eps_n^\Qcal$ converges at a sub-$\sqrt{n}$, \eg, $\eps_n^\Qcal$ converges at a $\wt\Ocal(n^{-1/4})$ rate if $\Qcal$ is nonparametric with metric entropy at most $1/t^2$ \citep{van2000asymptotic}. In \pref{sec:ortho-est}, we present an orthogonal estimator that is both robust to slower rates of $Q$ and achieves semiparametric efficiency. 

\section{Robust \texorpdfstring{$w$}{w}-Function Estimation with Minimax Learning}\label{sec:robust-w}

Before we present our orthogonal estimator, we study another essential nuisance function: the robust visitation density ratio, \ie, the robust $w$-function \citep{kallus2022efficiently,amortila2024harnessing}. In this section, we first identify the worst-case transition kernel $U^-$ in our uncertainty set $\Ucal(P)$. 
Then, we propose a minimax estimator \citep{uehara2021finite} for the robust $w$-function, an important nuisance function for our orthogonal estimator in \pref{sec:ortho-est}.

\paragraph{Identification of $U^-$.}
The robust transition kernel $U^-$ is defined as the feasible perturbed kernel that achieves the $\inf$ in the robust Bellman equation $Q^-(s,a)=\Trob^- Q^-(s,a)$.
Let $F^-(y\mid s,a)=P(V^-(s')\leq y\mid s,a)$ be the next-state pushforward measure of the robust value function $V^-$.
Then, $U^-$ is a convex combination of the nominal kernel $P$ and a reweighting of $P$ by an indicator function. 
\begin{restatable}{lemma}{identifcationU}\label{lem:identification-U-robust}
Suppose $F^-(\beta_{\tau}^-(s,a)\mid s,a)=\tau$, where $\beta^-_\tau(s,a)$ is the lower $\tau$-th quantile of $F^-(\cdot\mid s,a)$. Then,
\begin{equation}
  {U^-(s'\mid s,a)}/{P(s'\mid s,a)} = \Lambda^{-1}(s,a)+(1-\Lambda^{-1})\tau(s,a)^{-1}\II\bracks*{\prns*{V^-(s')-\beta^-_{\tau}(s,a)}\leq 0}. \label{eq:def-U-pm}
\end{equation}
\end{restatable}

The proof strategy decomposes $U^-$ into its nominal and perturbed components, leveraging the primal solution of $\cvar_\tau$ \citep{ang2018dual}; we formalize this in \pref{app:proof-identification-kernel}. 

\paragraph{Identification of $w^-$.} 
Using the identification of $U^-$ in \pref{lem:identification-U-robust}, we can now identify the robust $w$-function based on the Bellman flow equations in the worst-case MDP. 

The Bellman flow in the robust MDP is given by 
$
    d^{-,\infty}(s) = (1-\gamma)d_1(s)+ \gamma\EE_{\wt s\sim d^{-,\infty},\wt a\sim\pit(\wt s)}U^-(s\mid \wt s,\wt a). % \label{eq:robust-bellman-flow}
$
where $d^{-,\infty}(s)$ was defined in \pref{sec:prelim}. Thus, the robust visitation density, defined as $w^{-}(s):=\nicefrac{\diff d^{-,\infty}(s)}{\diff\nu(s)}$, satisfies the following moment condition for all $f:\Scal\mapsto\RR$:
\begin{equation}
    \EE[w^-(s)f(s)] =(1-\gamma)\EE_{d_1}[f(s_1)] + \gamma\EE\bracks*{w^-(s,a)\EE_{s'\sim U^-(s,a)}[f(s')] }, \label{eq:robust-w-moments}
\end{equation}
where we relaxed notation and defined $w^-(s,a):=w(s)\cdot\pit(a\mid s)/\nu(a\mid s)$.
As before, in the unconfounded base ($\Lambda=1$), this result recovers the classic Bellman flow.

\subsection{Estimating \texorpdfstring{$w^{-}$}{w-} with Robust Minimax Indirect Learning}\label{sec:w-minimax}
\begin{algorithm}[t!]
\caption{\RobustMIL: Minimax Estimation of $w^\pm$ with a Stabilizer}
\label{alg:robust-minimax}
\begin{algorithmic}[1]
\STATE\textbf{Input:} Dataset $\Dcal$, prior stage estimate $\wt\zeta$, function classes $\Wcal,\Fcal$, stabilizer weight $\lambda>0$. \\
\STATE Define weights
$
    \xi^-(s,a,s') := \Lambda^{-1}(s,a) + (1-\Lambda^{-1}(s,a))\tau^{-1}(s,a)\II\bracks*{ \wt\zeta(s,a,s')\leq 0}.
$
\STATE\textbf{Output: }
\begin{align}
    \wh w^- = \argmin_{w\in\Wcal}\max_{f\in\Fcal}\,
    &~\EE_n\bracks{w(s,a)\prns*{ \gamma\xi^-(s,a,s')f(s',\pit) - f(s,a) } + (1-\gamma)\EE_{d_1}f(s_1,\pit) }  \nonumber
    \\&- \lambda\|\gamma\xi^-(s,a,s';\wt\zeta)f(s',\pit) - f(s,a)\|_{2,n}^2 \label{eq:w-minimax-objective}
\end{align}
\end{algorithmic}
\end{algorithm}
We now propose a penalized minimax estimator for $w^-$ that generalizes the Minimax Indirect Learning (MIL) of \citet{uehara2021finite} to our robust MDP setting. Our estimator, \RobustMIL (\pref{alg:robust-minimax}), leverages a general function class $\Wcal\subset\Scal\times\Acal\mapsto\RR_+$ to approximately solve the moment equation in \pref{eq:robust-w-moments}. It does so by minimizing the difference between the left- and right-hand sides of the equation across a sufficiently large set of adversaries $f$ in a discriminator class $\Fcal\subset \Scal\times\Acal\mapsto\RR$.
Since $U^-$ is unknown, we approximate it via \pref{eq:def-U-pm} by plugging in a threshold $\wt\zeta(s,a,s')$ in the indicator function to approximate the true threshold $\zeta^-(s,a,s'):=V^-(s')-\beta_{\tau(s,a)}^-(s,a)$. This yields the minimax objective in \pref{eq:w-minimax-objective}, where we also allow for an optional regularization of the adversary's norm which can be useful for obtaining fast convergence rates.

We make the following assumptions for MIL \citep{uehara2021finite}.
The first is a regularity condition that (i) our function class has bounded outputs and (ii) $\zeta$ is continuously distributed around the threshold.
\begin{assumption}[Regularity]\label{assum:regular}
(i) $\sup_{w\in\Wcal\cup\{w^-\}}\|w\|_\infty<\infty$;  (ii) the marginal CDF of $V^-(s')-\beta^-(s,a)$, \ie, $F(y)=P(V^-(s')-\beta_{\tau(s,a)}^-(s,a)\leq y)$, is boundedly differentiable around $0$. 
\end{assumption}
If next-value distribution is discrete, we can use the discrete form of CVaR and (ii) can be removed.

The second is that the adversary class is rich enough to capture all projected errors under the adjoint of the operator $\Jcal_{U^-} f(s,a) := \gamma\EE_{U^-}[f(s',\pit)\mid s,a]-f(s,a)$. 
\begin{assumption}[$w^-$-realizability and completeness]\label{asm:w-realizability-completeness}
$w^-\in\Wcal$ and $\Jcal_{U^-}'(\Wcal-w^-)\subset\Fcal$.
\end{assumption}
We note that \cref{asm:w-realizability-completeness} is monotone in the function class size and can be satisfied by making the function class more expressive, \eg, increasing size of the neural net. Our algorithms are also robust to violations in \cref{asm:w-realizability-completeness}, which we show in \pref{app:proofs-minimax}.

We are now ready to state the main estimation result for $w^-$ in terms of the critical radius (\pref{app:critical-radius}) of the function class. 

\begin{restatable}{theorem}{minimaxFastRatesW}\label{thm:minimax-fast-rates-for-w}
Let $\eps_n^\Wcal$ denote the maximum critical radii of the following classes:
\begin{align*}
    &\Gcal_1 = \braces{ (s,a,s')\mapsto (f(s,a)-\gamma f(s',\pit)), f\in\Fcal },
    \\&\Gcal_2 = \braces{ (s,a,s')\mapsto (w(s,a)-w^-(s,a))( \gamma f(s',\pit)-f(s,a) ), f\in\Fcal,w\in\Wcal }.
\end{align*}
Under \cref{assum:regular,asm:w-realizability-completeness}, \RobustMIL ensures that for any $\delta$, w.p. $1-\delta$,
\begin{equation*}
    \textstyle\nm{\Jcal_{U^-}'(\wh w-w^-)}_2\lesssim \eps_n^\Wcal+\|\wt\zeta^- -\zeta^-\|_\infty+\sqrt{ {\log(1/\delta)}/{n} }.
\end{equation*}
\end{restatable}
As before, the critical radius $\eps_n^\Wcal$ converges at an $\wt\Ocal(n^{-1/2})$ rate for parametric classes.
Notably, our bounds degrade linearly w.r.t. the $\ell_\infty$ error in $\wt\zeta^-$ for estimating $\zeta^-$. For example, if $\wt\zeta(s,a,s')=\wh v(s')-\wh\beta(s,a)$ where $\wh v,\wh\beta$ are estimated with \RobustFQE, then the $\zeta$-error can be bounded by $\Ocal(\|\wh v-v^-\|_\infty+\|\wh\beta-\beta^-\|_\infty)$. 
We present the full proof in \pref{app:proofs-minimax}, where we also present a more general result that is robust to misspecifications to realizability and completeness (\cref{asm:w-realizability-completeness}).

\section{Orthogonal and Efficient Estimator for Robust Policy Value}\label{sec:ortho-est}

In this section, we propose an orthogonal estimator that is robust against errors in the nuisances (exhibiting only second-order sensitivity), achieves semiparametric efficiency, and enables inference. Our estimator is based on the efficient influence function (EIF) of $V^{-}_{d_1}$, which is the canonical gradient of a statistical estimand \citep{tsiatis2006semiparametric}.
The adoption of EIFs for developing efficient estimators is a broadly employed technique in causal inference \citep{chernozhukov2018double, kennedy2020towards}  and reinforcement learning \citep{jiang2016doubly, kallus2022efficiently}.

We define the collection of nuisance parameters by $\eta^{-} = (w^{-}, q^{-}, \beta^{-})$. The notation $\widehat{\eta}$ indicates that these functions are estimated from data, while the notation $\eta$ denotes their true values.

\begin{theorem}[(Recentered) Efficient Influence Function]\label{thm:eif-main}
    The (R)EIF of $V^{-}_{d_1}$ is given by:
    \begin{align*}
        &\textstyle\psi(s, a, s';\eta^{-}) = V^{-}_{d_1} + w^{-}(s,a)\big(r(s,a)+\gamma\rho^{-}(s,a,s'; v^{-}, \beta^{-}) -q^{-}(s,a)\big),\quad\text{where}
    \\
        &\textstyle\rho^{-}(s,a,s'; v^{-}, \beta^{-}) = \Lambda(s,a)^{-1}v^{-}(s') +  (1-\Lambda(s,a)^{-1})\big(\beta^{-}(s,a) + \tau^{-1}(v^{-}(s')-\beta^{-}(s,a))_{-}\big).
    \end{align*}
\end{theorem}
\begin{remark}
When $\Lambda=1$, there is no shift in the target environment, and the weight on the $\cvar$ term is zero. The (R)EIF then reduces to the (R)EIF in \citet{kallus2022efficiently} for regular OPE with an infinite horizon. As $\Lambda\rightarrow\infty$, the $\cvar$ term becomes predominant, with the quantile $\beta^{-}(s,a)$ taking extreme values. This yields the (novel) (R)EIF for the problem in \citet{du2022provably}, where the expected value term is replaced solely by a $\cvar$ component in the Bellman equation.
\end{remark}

The (R)EIF forms the basis of our orthogonal estimator. First, we note that $\EE[\psi(s,a,s';\eta^{-})]$ is an unbiased estimator of $V^{-}_{d_1}$. Furthermore, the expression for $\psi(s,a,s';\eta^{-})$ depends only on quantities $w^{-}, q^{-}, \beta^{-}$ which can be estimated from data. Thus, we can cast the expression $\EE[\psi(s,a,s';\eta^{-})]$ as a statistical estimand to be learned from the observed sample. This suggests a natural two-stage estimator that we summarize in \pref{alg:ortho-ope}. In the first stage, we estimate the nuisance parameters $\widehat{\eta}$ from the data with $K$-fold cross-fitting; in the second stage, these estimates are incorporated into the (R)EIF expression and we calculate the empirical average using the observed data. We summarize our procedure in \pref{alg:ortho-ope}.

\begin{algorithm}[t!]
\caption{Orthogonal Estimator for $V^{-}_{d_1}$}
\label{alg:ortho-ope}
\begin{algorithmic}[1]
    \STATE \textbf{Input:} Dataset $\Dcal$, number of splits $K$.
    \FOR{$k=1,2,\dots,K$}
        \STATE Use data $\Dcal \setminus \Dcal_k$ to learn $(q^{-, [k]},\beta^{-,[k]})$ with \pref{alg:robust-fqe} and $w^{-,[k]}$ with \pref{alg:robust-minimax}
        \STATE \textbf{for} $i=\lfloor (k-1)n/K\rfloor,\dots,\lfloor kn/K\rfloor-1$ \textbf{do} $\psi^{-}_i = \psi(s_i,a_i,s'_i, \widehat{\eta}^{-})$
    \ENDFOR
    \STATE\textbf{Output: } $\widehat{V}^{-}_{d_1} = \frac{1}{n}\sum_{i=1}^n \psi^{-}_i$.
    \end{algorithmic}
\end{algorithm}

The nuisance estimation is detailed in \cref{sec:q-est,sec:w-minimax}. The reliance on the EIF confers our estimator desirable statistical properties including a second order bias due to the nuisances, meaning the bias has a product structure with respect to the nuisance errors. Thus, this special structure orthogonalizes away the dependency on $\widehat{Q}^{-}$ errors which now only appear in second order. Furthermore, our estimator is semiparametrically efficient in the sense that under mild consistency assumptions, it achieves minimum variance among all regular and asymptotically linear (RAL) estimators. We provide theoretical justifications for these properties in the next section.

\subsection{Theoretical Guarantees of the Orthogonal Estimator}
We now characterize the theoretical properties of our orthogonal estimator. 
We consider the $K$-fold cross-fitted estimator in \pref{alg:ortho-ope} given by
\begin{equation*}\textstyle
    \widehat{V}^{-}_{d_1}= \frac{1}{n}\sum_{k=1}^K \sum_{(s,a,s')\in \Dcal^k} \psi(s,a,s'; \widehat{\eta}^{[k]}),
\end{equation*}
where nuisances $\widehat{\eta}^{[k]}, k\in[K]$ are trained on all data excluding the $k^\text{th}$ fold $\Dcal^k$. The following theorem outlines the theoretical guarantees of this estimator:

\begin{theorem}[Efficiency of $\widehat{V}_{d_1}^-$]\label{thm:efficiency}
    Let $r_{n, p}^w,r_{n,p}^q, r_{n,p}^{\beta}$ be functions of the same size $n=|\Dcal|$ such that $\|\Jcal_{U^-}'(\widehat{w}^{-, [k]}-w)\|_p\leq r_{n,p}^w$, $\|\widehat{q}^{-, [k]}-q\|_p\leq r_{n,p}^q$, and $\|\beta^{-, [k]}-\beta\|_p\leq r_{n,p}^{\beta}$ for any $k\in [K]$. Furthermore, assume that the regularity conditions in \pref{assum:regular} hold. Then:
    \begin{align}
        |\widehat{V}_{d_1}^--V^-_{d_1}| \lesssim O_p(n^{-1/2}) + O_p(r_{n,2}^wr_{n,2}^q+(r_{n, \infty}^{q})^2+(r_{n,\infty}^{\beta})^2)\tag{Rates}\label{eq:rates}
    \end{align}
    Furthermore, if $r_{n,2}^w \vee r_{n,2}^q=o_p(1)$, $r_{n,2}^w r_{n,2}^q=o_p(n^{-1/2})$, $r_{n,\infty}^{q}=o_p(n^{-1/4})$, and $r_{n,\infty}^{\beta}=o_p(n^{-1/4})$, then $\widehat{V}^-_{d_1}$ satisfies:
    \begin{align}
        \sqrt{n}(\widehat{V}^-_{d_1}-V^{-}_{d_1}) \xrightarrow{d} \Ncal(0, \Sigma) \tag{Normality \& Efficiency}\label{eq:efficiency},\quad\Sigma=\mathrm{Var}(\psi(s,a,s';\eta^-)).
    \end{align}
    Moreover, $\Sigma$ is the minimum achievable asymptotic variance among RAL estimators in the nonparametric model for $(s,a,s')$ (the efficiency bound).
\end{theorem}
We provide the intuition along with a detailed proof in \pref{app:efficiency-proof}. The first part of \pref{thm:efficiency} implies that as long as we estimate the nuisances at rates faster that $n^{-1/4}$, then we can learn $\widehat{V}^{-}_{d_1}$ at parametric rates.
The second part of \pref{thm:efficiency} states that under mild consistency assumptions, our estimator attains the efficiency bound and is asymptotically normal. That means, for example, we can construct asymptotically valid lower 95\%-confidence bound on $\widehat{V}^{-}_{d_1}$ by simply subtracting 1.64 times $\hat{\mathrm{se}}=\frac{1}{n}\prns*{\sum_{k=1}^K \sum_{(s,a,s')\in \Dcal^k} (\psi(s,a,s'; \widehat{\eta}^{[k]})-\widehat{V}^{-}_{d_1})^2}^{1/2}$. Then, we can be sure to have a bound on the worst-case RL policy value, accounting \emph{both} for potential environment shift and finite data. 
Finally, in \pref{app:validity-proofs}, we describe two settings when our orthogonal estimator remains valid even if some nuisances are \emph{inconsistent}, which is a desirable guarantee for sensitivity analysis \citep{dorn2023sharp}.

\paragraph{Bringing it all together.} We can instantiate \pref{thm:efficiency} with the nuisance estimators from the previous sections. First, use \RobustFQE to estimate $\wh q^-$ and $\wh\beta^-$, ensuring $\|\wh q^--Q^-\|_2\leq \Ocal(\eps^\Qcal_n+\QRerr^2)$. Under smoothness conditions (\pref{lem:smoothness}), the $L_2$ guarantee for $\wh q^-$ implies an $L_\infty$ guarantee for $\wh q^-$, which also ensures an $L_\infty$ guarantee for $\wh\beta^-$. 
This ensures $\max(\|\wh q^--Q^-\|_\infty,\|\wh\beta^--\beta^-\|_\infty)$ is well-controlled.
Then, we can set $\wt\zeta^-(s,a,s')=\wh q^-(s',\pit)-\wh\beta^-(s,a)$ and run \RobustMIL for estimating $\wh w^-$. By \pref{thm:minimax-fast-rates-for-w}, its projected-$L_2$ error is $\Ocal(\eps_n^\Wcal+\|\wh q^--Q^-\|_\infty+\|\wh\beta^--\beta^-\|_\infty)$. Therefore, the final rate via \pref{thm:efficiency} is $\Ocal((\eps^\Qcal_n+\QRerr^2)\cdot\eps^\Wcal_n + \|\wh q^--Q^-\|_\infty^2+\|\wh\beta^--\beta^-\|_\infty^2)$. 

\section{Empirical Evaluation}\label{sec:experiments}

We now provide a proof-of-concept empirical investigation to validate our theoretical findings. We experiment with our proposed methodology in a simple synthetic environment. First, we discuss our environment, followed by our approach for solving for the nuisances functions $\eta^-$. Then, we provide empirical results for our orthogonal estimator, and compare its performance to weighted or direct estimators using the $Q^-$ or $w^-$ nuisances only. The code for our experiments is open-sourced and available at \url{https://github.com/CausalML/adversarial-ope/}. 

\begin{figure}[!h]
    \centering
    \includegraphics[width=\textwidth]{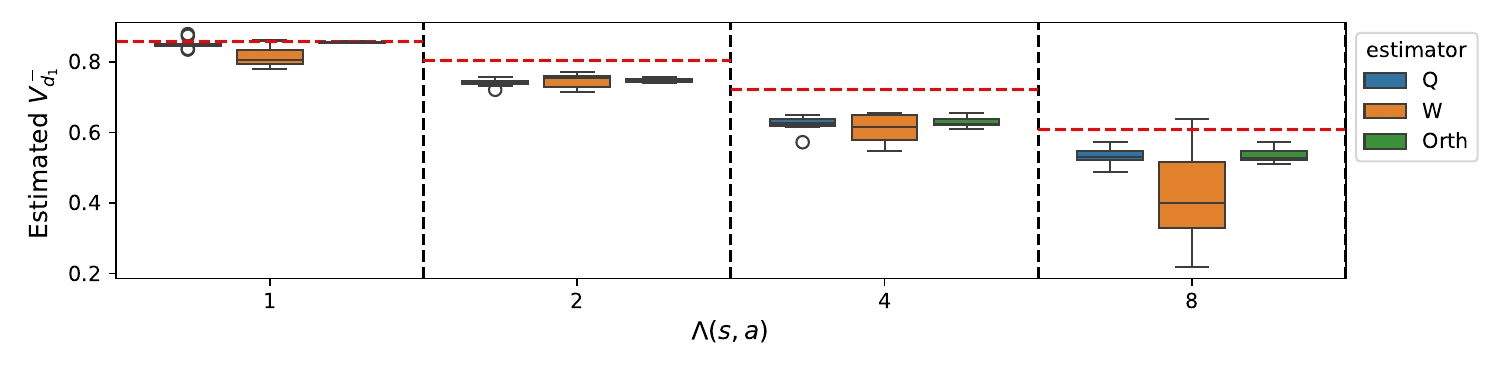}
    \vspace{1em}
    \begin{tabular}{cccc}
         \multirow{2}{*}{$\Lambda$} & \multicolumn{3}{c}{Mean squared error (MSE) to true worst-case policy value} \vspace{0.5em}\\
         & \textbf{Q} & \textbf{W} & \textbf{Orth} \\
         \hline
            1 & $.000240 \pm .000170$ & $.002722 \pm .002266$ & $\mathbf{.000005 \pm .000006}$ \\
            2 & $.004053 \pm .001329$ & $.003584 \pm .002311$ & $\mathbf{.003244 \pm .000600}$  \\
            4 & $.009799 \pm .005172$ & $.013862 \pm .009228$ & $\mathbf{.008721 \pm .002543}$  \\
            8 & $.006247 \pm .003980$ & $.052643 \pm .050839$ & $\mathbf{.005713 \pm .002730}$  \\
         \hline
    \end{tabular}
    \vspace{0.5em}
    \caption{Results of our synthetic data experiments. We show results for our three estimators on all four $\Lambda$ values, over our 10 experiment replications. \textbf{Above:} Box plot summarizing range of policy value estimates for each combination of estimator and $\Lambda$, with Horizontal red dashed lines showing the true worst-case policy values $V^-_{d_1}$. \textbf{Below:} Table summarizing the corresponding MSE of these estimators for the true worst-case policy value, along with one standard deviation errors.}
    \label{fig:results}
\end{figure}

\paragraph{Experimental Setup} We consider a synthetic MDP with a one-dimensional state and two actions, modeled after a simple control problem with non-deterministic dynamics. The task is to estimate the worst-case policy value $V^-_{d_1}$ of a fixed candidate policy $\pit$, across four different constant values of the sensitivity parameter: $\Lambda(s,a) \in \{1,2,4,8\}$.

\newpage
We considered three methods for estimating the robust value $V^-_{d_1}$:
\begin{enumerate}[noitemsep,topsep=0ex,partopsep=0ex,parsep=0ex,left=0em]
    \item \textbf{Q} (\RobustFQE): Direct method using the estimated robust quality function $\wh Q^-$ only.
    \item \textbf{W} (\RobustMIL): Importance-sampling method using the estimated robust density ratio $\wh w^-$ only.
    \item \textbf{Orth}: Our orthogonal estimator which combines the former two, as described in \cref{alg:ortho-ope}.
\end{enumerate}

We performed 10 replications of our experimental procedure, where for each replication we: (1) sampled a dataset of 20,000 tuples using a different fixed logging policy $\pi_b$; (2) fit the nuisance functions $Q^-$, $\beta^-$, and $w^-$ following the method outlined in \cref{alg:robust-fqe,alg:robust-minimax} for each $\Lambda$; and (3) estimated the corresponding robust policy value $V^-_{d_1}$ for all estimators using the fitted nuisances.

\paragraph{Results} We summarize our results in \cref{fig:results}.
We note that all of our estimators are consistently valid for all values of $\Lambda$ in our experiment. Notably, \textbf{Orth} consistently has the lowest mean squared error for the true worst-case policy value. In particular, incorporating the robust importance-sampling weights improves the \RobustFQE estimator $\textbf{Q}$, even though these importance-sampling weights by themselves (as in \textbf{W}) are much noisier estimators. This is consistent with our theory that the orthogonal estimator is semiparametrically efficient and insensitive to errors in the nuisance functions.

Full experimental details, including our MDP, target/logging policies, methodology for computing the true robust policy values $V^-_{d_1}$, and nuisance estimation, are provided in \cref{apx:experiment}.
Finally, we also performed an empirical evaluation in the real-world medical problem of sepsis management using the MIMIC-III dataset \citep{johnson2016mimic}. We detail these results in \cref{sec:medical-experiments}.

\section{Conclusion}
We consider the problem of infinite-horizon OPE in RL settings when there can be unknown, but bounded, shifts in the transition distribution compared to the transition distribution generating the data. This can arise due to unobserved confounding, where observed transitions do not reflect the true causal ones, non-stationarity in the environment, or adversarial environments. We propose a sensitivity model for such transition kernel shifts analogous to the classic MSM for static decision making, and provide theoretical guarantees for identifying and estimating the sharp (\ie, tightest possible) bounds on the best/worst-case policy value, as well as the corresponding robust $Q$-function and state density ratio functions. Our estimator for the best/worst-case policy value is orthogonal (insensitive to how the nuisance functions are estimated) and achieves semiparametric efficiency (attaining the best possible asymptotic variance). Finally, our estimator also supports inference, ensuring we can derive reliable bounds for the robust policy value even with finite data.

\subsubsection*{Acknowledgements}
We thank the anonymous reviewers for their valuable feedback and insightful suggestions. This material is based upon work supported by the National Science Foundation under Grant Numbers 1846210, IIS-2154711, CAREER 2339395, and by the U.S. Department of Energy, Office of Science, Office of Advanced Scientific Computing Research, under Award Number DE-SC0023112.

\bibliography{paper}

\begin{thebibliography}{81}
\providecommand{\natexlab}[1]{#1}
\providecommand{\url}[1]{\texttt{#1}}
\expandafter\ifx\csname urlstyle\endcsname\relax
  \providecommand{\doi}[1]{doi: #1}\else
  \providecommand{\doi}{doi: \begingroup \urlstyle{rm}\Url}\fi

\bibitem[Agarwal et~al.(2023)Agarwal, Song, Sun, Wang, Wang, and Zhang]{agarwal2023provable}
Alekh Agarwal, Yuda Song, Wen Sun, Kaiwen Wang, Mengdi Wang, and Xuezhou Zhang.
\newblock Provable benefits of representational transfer in reinforcement learning.
\newblock In \emph{The Thirty Sixth Annual Conference on Learning Theory}, pages 2114--2187. PMLR, 2023.

\bibitem[Amortila et~al.(2024)Amortila, Foster, Jiang, Sekhari, and Xie]{amortila2024harnessing}
Philip Amortila, Dylan~J Foster, Nan Jiang, Ayush Sekhari, and Tengyang Xie.
\newblock Harnessing density ratios for online reinforcement learning.
\newblock \emph{arXiv preprint arXiv:2401.09681}, 2024.

\bibitem[Ang et~al.(2018)Ang, Sun, and Yao]{ang2018dual}
Marcus Ang, Jie Sun, and Qiang Yao.
\newblock On the dual representation of coherent risk measures.
\newblock \emph{Annals of Operations Research}, 262:\penalty0 29--46, 2018.

\bibitem[Audibert and Tsybakov(2005)]{audibert2005fast}
Jean-Yves Audibert and Alexandre~B Tsybakov.
\newblock Fast learning rates for plug-in classifiers under the margin condition.
\newblock \emph{arXiv preprint math/0507180}, 2005.

\bibitem[Ayoub et~al.(2024)Ayoub, Wang, Liu, Robertson, McInerney, Liang, Kallus, and Szepesv{\'a}ri]{ayoub2024switching}
Alex Ayoub, Kaiwen Wang, Vincent Liu, Samuel Robertson, James McInerney, Dawen Liang, Nathan Kallus, and Csaba Szepesv{\'a}ri.
\newblock Switching the loss reduces the cost in batch reinforcement learning.
\newblock \emph{International Conference of Machine Learning}, 2024.

\bibitem[Badrinath and Kalathil(2021)]{badrinath2021robust}
Kishan~Panaganti Badrinath and Dileep Kalathil.
\newblock Robust reinforcement learning using least squares policy iteration with provable performance guarantees.
\newblock In \emph{International Conference on Machine Learning}, pages 511--520. PMLR, 2021.

\bibitem[Belloni et~al.(2017)Belloni, Chernozhukov, Fernandez-Val, and Hansen]{belloni2017program}
Alexandre Belloni, Victor Chernozhukov, Ivan Fernandez-Val, and Christian Hansen.
\newblock Program evaluation and causal inference with high-dimensional data.
\newblock \emph{Econometrica}, 85\penalty0 (1):\penalty0 233--298, 2017.

\bibitem[Bennett et~al.(2023)Bennett, Kallus, and Oprescu]{bennett2023low}
Andrew Bennett, Nathan Kallus, and Miruna Oprescu.
\newblock Low-rank mdps with continuous action spaces.
\newblock \emph{arXiv preprint arXiv:2311.03564}, 2023.

\bibitem[Bhattacharya and Gangopadhyay(1990)]{bhattacharya1990kernel}
Pallab~K Bhattacharya and Ashis~K Gangopadhyay.
\newblock Kernel and nearest-neighbor estimation of a conditional quantile.
\newblock \emph{The Annals of Statistics}, pages 1400--1415, 1990.

\bibitem[Bonvini et~al.(2022)Bonvini, Kennedy, Ventura, and Wasserman]{bonvini2022sensitivity}
Matteo Bonvini, Edward Kennedy, Valerie Ventura, and Larry Wasserman.
\newblock Sensitivity analysis for marginal structural models.
\newblock \emph{arXiv preprint arXiv:2210.04681}, 2022.

\bibitem[Brezis and Mironescu(2019)]{brezis2019sobolev}
Ha{\"\i}m Brezis and Petru Mironescu.
\newblock Where sobolev interacts with gagliardo--nirenberg.
\newblock \emph{Journal of functional analysis}, 277\penalty0 (8):\penalty0 2839--2864, 2019.

\bibitem[Bruns-Smith and Zhou(2023)]{bruns2023robust}
David Bruns-Smith and Angela Zhou.
\newblock Robust fitted-q-evaluation and iteration under sequentially exogenous unobserved confounders.
\newblock \emph{arXiv preprint arXiv:2302.00662}, 2023.

\bibitem[Bruns-Smith(2021)]{bruns2021model}
David~A Bruns-Smith.
\newblock Model-free and model-based policy evaluation when causality is uncertain.
\newblock In \emph{International Conference on Machine Learning}, pages 1116--1126. PMLR, 2021.

\bibitem[{\'C}evid et~al.(2020){\'C}evid, Michel, N{\"a}f, Meinshausen, and B{\"u}hlmann]{cevid2020distributional}
Domagoj {\'C}evid, Loris Michel, Jeffrey N{\"a}f, Nicolai Meinshausen, and Peter B{\"u}hlmann.
\newblock Distributional random forests: Heterogeneity adjustment and multivariate distributional regression.
\newblock \emph{arXiv preprint arXiv:2005.14458}, 2020.

\bibitem[Chang et~al.(2022)Chang, Wang, Kallus, and Sun]{chang2022learning}
Jonathan Chang, Kaiwen Wang, Nathan Kallus, and Wen Sun.
\newblock Learning bellman complete representations for offline policy evaluation.
\newblock In \emph{International Conference on Machine Learning}, pages 2938--2971. PMLR, 2022.

\bibitem[Chernozhukov et~al.(2018{\natexlab{a}})Chernozhukov, Chetverikov, Demirer, Duflo, Hansen, Newey, and Robins]{chernozhukov2018double}
Victor Chernozhukov, Denis Chetverikov, Mert Demirer, Esther Duflo, Christian Hansen, Whitney Newey, and James Robins.
\newblock {Double/debiased machine learning for treatment and structural parameters}.
\newblock \emph{The Econometrics Journal}, 21\penalty0 (1):\penalty0 C1--C68, 2018{\natexlab{a}}.
\newblock \doi{10.1111/ectj.12097}.

\bibitem[Chernozhukov et~al.(2018{\natexlab{b}})Chernozhukov, Demirer, Duflo, and Fernandez-Val]{chernozhukov2018generic}
Victor Chernozhukov, Mert Demirer, Esther Duflo, and Ivan Fernandez-Val.
\newblock Generic machine learning inference on heterogeneous treatment effects in randomized experiments, with an application to immunization in india.
\newblock Technical report, National Bureau of Economic Research, 2018{\natexlab{b}}.

\bibitem[Chernozhukov et~al.(2022{\natexlab{a}})Chernozhukov, Cinelli, Newey, Sharma, and Syrgkanis]{chernozhukov2022long}
Victor Chernozhukov, Carlos Cinelli, Whitney Newey, Amit Sharma, and Vasilis Syrgkanis.
\newblock Long story short: Omitted variable bias in causal machine learning.
\newblock Technical report, National Bureau of Economic Research, 2022{\natexlab{a}}.

\bibitem[Chernozhukov et~al.(2022{\natexlab{b}})Chernozhukov, Newey, Singh, and Syrgkanis]{chernozhukov2022automatic}
Victor Chernozhukov, Whitney Newey, Rahul Singh, and Vasilis Syrgkanis.
\newblock Automatic debiased machine learning for dynamic treatment effects and general nested functionals.
\newblock \emph{arXiv preprint arXiv:2203.13887}, 2022{\natexlab{b}}.

\bibitem[Chow et~al.(2015)Chow, Tamar, Mannor, and Pavone]{chow2015risk}
Yinlam Chow, Aviv Tamar, Shie Mannor, and Marco Pavone.
\newblock Risk-sensitive and robust decision-making: a cvar optimization approach.
\newblock \emph{Advances in neural information processing systems}, 28, 2015.

\bibitem[Dabney et~al.(2018{\natexlab{a}})Dabney, Ostrovski, Silver, and Munos]{dabney2018implicit}
Will Dabney, Georg Ostrovski, David Silver, and R{\'e}mi Munos.
\newblock Implicit quantile networks for distributional reinforcement learning.
\newblock In \emph{International conference on machine learning}, pages 1096--1105. PMLR, 2018{\natexlab{a}}.

\bibitem[Dabney et~al.(2018{\natexlab{b}})Dabney, Rowland, Bellemare, and Munos]{dabney2018distributional}
Will Dabney, Mark Rowland, Marc Bellemare, and R{\'e}mi Munos.
\newblock Distributional reinforcement learning with quantile regression.
\newblock In \emph{Proceedings of the AAAI conference on artificial intelligence}, volume~32, 2018{\natexlab{b}}.

\bibitem[Dorn and Guo(2023)]{dorn2023sharp}
Jacob Dorn and Kevin Guo.
\newblock Sharp sensitivity analysis for inverse propensity weighting via quantile balancing.
\newblock \emph{Journal of the American Statistical Association}, 118\penalty0 (544):\penalty0 2645--2657, 2023.

\bibitem[Dorn et~al.(2021)Dorn, Guo, and Kallus]{dorn2021doubly}
Jacob Dorn, Kevin Guo, and Nathan Kallus.
\newblock Doubly-valid/doubly-sharp sensitivity analysis for causal inference with unmeasured confounding.
\newblock \emph{arXiv preprint arXiv:2112.11449}, 2021.

\bibitem[Du et~al.(2022)Du, Wang, and Huang]{du2022provably}
Yihan Du, Siwei Wang, and Longbo Huang.
\newblock Provably efficient risk-sensitive reinforcement learning: Iterated cvar and worst path.
\newblock In \emph{The Eleventh International Conference on Learning Representations}, 2022.

\bibitem[Duan et~al.(2021)Duan, Jin, and Li]{duan2021risk}
Yaqi Duan, Chi Jin, and Zhiyuan Li.
\newblock Risk bounds and rademacher complexity in batch reinforcement learning.
\newblock In \emph{International Conference on Machine Learning}, pages 2892--2902. PMLR, 2021.

\bibitem[El~Ghouch and Genton(2009)]{el2009local}
Anouar El~Ghouch and Marc~G Genton.
\newblock Local polynomial quantile regression with parametric features.
\newblock \emph{Journal of the American Statistical Association}, 104\penalty0 (488):\penalty0 1416--1429, 2009.

\bibitem[Elie-Dit-Cosaque and Maume-Deschamps(2022)]{elie2022random}
Kevin Elie-Dit-Cosaque and V{\'e}ronique Maume-Deschamps.
\newblock Random forest estimation of conditional distribution functions and conditional quantiles.
\newblock \emph{Electronic Journal of Statistics}, 16\penalty0 (2):\penalty0 6553--6583, 2022.

\bibitem[Foster and Syrgkanis(2023)]{foster2023orthogonal}
Dylan~J Foster and Vasilis Syrgkanis.
\newblock Orthogonal statistical learning.
\newblock \emph{The Annals of Statistics}, 51\penalty0 (3):\penalty0 879--908, 2023.

\bibitem[Goyal and Grand-Clement(2023)]{goyal2023robust}
Vineet Goyal and Julien Grand-Clement.
\newblock Robust markov decision processes: Beyond rectangularity.
\newblock \emph{Mathematics of Operations Research}, 48\penalty0 (1):\penalty0 203--226, 2023.

\bibitem[Hines et~al.(2022)Hines, Dukes, Diaz-Ordaz, and Vansteelandt]{hines2022demystifying}
Oliver Hines, Oliver Dukes, Karla Diaz-Ordaz, and Stijn Vansteelandt.
\newblock Demystifying statistical learning based on efficient influence functions.
\newblock \emph{The American Statistician}, 76\penalty0 (3):\penalty0 292--304, 2022.

\bibitem[Howard and Matheson(1972)]{howard1972risk}
Ronald~A Howard and James~E Matheson.
\newblock Risk-sensitive markov decision processes.
\newblock \emph{Management science}, 18\penalty0 (7):\penalty0 356--369, 1972.

\bibitem[Ichimura and Newey(2022)]{ichimura2022influence}
Hidehiko Ichimura and Whitney~K Newey.
\newblock The influence function of semiparametric estimators.
\newblock \emph{Quantitative Economics}, 13\penalty0 (1):\penalty0 29--61, 2022.

\bibitem[Iyengar(2005)]{iyengar2005robust}
Garud~N Iyengar.
\newblock Robust dynamic programming.
\newblock \emph{Mathematics of Operations Research}, 30\penalty0 (2):\penalty0 257--280, 2005.

\bibitem[Jiang and Li(2016)]{jiang2016doubly}
Nan Jiang and Lihong Li.
\newblock Doubly robust off-policy value evaluation for reinforcement learning.
\newblock In \emph{International Conference on Machine Learning}, pages 652--661. PMLR, 2016.

\bibitem[Johnson et~al.(2016)Johnson, Pollard, Shen, Lehman, Feng, Ghassemi, Moody, Szolovits, Anthony~Celi, and Mark]{johnson2016mimic}
Alistair~EW Johnson, Tom~J Pollard, Lu~Shen, Li-wei~H Lehman, Mengling Feng, Mohammad Ghassemi, Benjamin Moody, Peter Szolovits, Leo Anthony~Celi, and Roger~G Mark.
\newblock Mimic-iii, a freely accessible critical care database.
\newblock \emph{Scientific data}, 3\penalty0 (1):\penalty0 1--9, 2016.

\bibitem[Kallus(2022)]{kallus2022s}
Nathan Kallus.
\newblock What's the harm? sharp bounds on the fraction negatively affected by treatment.
\newblock \emph{Advances in Neural Information Processing Systems}, 35:\penalty0 15996--16009, 2022.

\bibitem[Kallus and Uehara(2020)]{kallus2020double}
Nathan Kallus and Masatoshi Uehara.
\newblock Double reinforcement learning for efficient off-policy evaluation in markov decision processes.
\newblock \emph{The Journal of Machine Learning Research}, 21\penalty0 (1):\penalty0 6742--6804, 2020.

\bibitem[Kallus and Uehara(2022)]{kallus2022efficiently}
Nathan Kallus and Masatoshi Uehara.
\newblock Efficiently breaking the curse of horizon in off-policy evaluation with double reinforcement learning.
\newblock \emph{Operations Research}, 70\penalty0 (6):\penalty0 3282--3302, 2022.

\bibitem[Kallus and Zhou(2020)]{kallus2020confounding}
Nathan Kallus and Angela Zhou.
\newblock Confounding-robust policy evaluation in infinite-horizon reinforcement learning.
\newblock \emph{Advances in neural information processing systems}, 33:\penalty0 22293--22304, 2020.

\bibitem[Kallus et~al.(2022)Kallus, Mao, Wang, and Zhou]{kallus2022doubly}
Nathan Kallus, Xiaojie Mao, Kaiwen Wang, and Zhengyuan Zhou.
\newblock Doubly robust distributionally robust off-policy evaluation and learning.
\newblock In \emph{International Conference on Machine Learning}, pages 10598--10632. PMLR, 2022.

\bibitem[Kennedy(2019)]{kennedy2019nonparametric}
Edward~H Kennedy.
\newblock Nonparametric causal effects based on incremental propensity score interventions.
\newblock \emph{Journal of the American Statistical Association}, 114\penalty0 (526):\penalty0 645--656, 2019.

\bibitem[Kennedy(2020)]{kennedy2020towards}
Edward~H Kennedy.
\newblock Towards optimal doubly robust estimation of heterogeneous causal effects.
\newblock \emph{arXiv preprint arXiv:2004.14497}, 2020.

\bibitem[Kiani et~al.(2019)Kiani, Wang, and Xu]{kiani2019sepsis}
Amirhossein Kiani, Chris Wang, and Angela Xu.
\newblock Sepsis world model: A mimic-based openai gym" world model" simulator for sepsis treatment.
\newblock \emph{arXiv preprint arXiv:1912.07127}, 2019.

\bibitem[Kolter(2011)]{kolter2011fixed}
J~Kolter.
\newblock The fixed points of off-policy td.
\newblock \emph{Advances in Neural Information Processing Systems}, 24, 2011.

\bibitem[Kumar et~al.(2022)Kumar, Levy, Wang, and Mannor]{kumar2022efficient}
Navdeep Kumar, Kfir Levy, Kaixin Wang, and Shie Mannor.
\newblock Efficient policy iteration for robust markov decision processes via regularization.
\newblock \emph{arXiv preprint arXiv:2205.14327}, 2022.

\bibitem[Kumar et~al.(2023)Kumar, Derman, Geist, Levy, and Mannor]{kumar2023policy}
Navdeep Kumar, Esther Derman, Matthieu Geist, Kfir~Yehuda Levy, and Shie Mannor.
\newblock Policy gradient for rectangular robust markov decision processes.
\newblock In \emph{Thirty-seventh Conference on Neural Information Processing Systems}, 2023.
\newblock URL \url{https://openreview.net/forum?id=NLpXRrjpa6}.

\bibitem[Laan and Robins(2003)]{laan2003unified}
Mark~J Laan and James~M Robins.
\newblock \emph{Unified methods for censored longitudinal data and causality}.
\newblock Springer, 2003.

\bibitem[Leoni(2017)]{leoni2017first}
Giovanni Leoni.
\newblock \emph{A first course in Sobolev spaces}.
\newblock American Mathematical Soc., 2017.

\bibitem[Lewis and Syrgkanis(2021)]{lewis2021double}
Greg Lewis and Vasilis Syrgkanis.
\newblock Double/debiased machine learning for dynamic treatment effects.
\newblock In \emph{NeurIPS}, pages 22695--22707, 2021.

\bibitem[Mannor et~al.(2016)Mannor, Mebel, and Xu]{mannor2016robust}
Shie Mannor, Ofir Mebel, and Huan Xu.
\newblock Robust mdps with k-rectangular uncertainty.
\newblock \emph{Mathematics of Operations Research}, 41\penalty0 (4):\penalty0 1484--1509, 2016.

\bibitem[Meinshausen and Ridgeway(2006)]{meinshausen2006quantile}
Nicolai Meinshausen and Greg Ridgeway.
\newblock Quantile regression forests.
\newblock \emph{Journal of machine learning research}, 7\penalty0 (6), 2006.

\bibitem[Mnih(2013)]{mnih2013playing}
Volodymyr Mnih.
\newblock Playing atari with deep reinforcement learning.
\newblock \emph{arXiv preprint arXiv:1312.5602}, 2013.

\bibitem[Munos and Szepesv{\'a}ri(2008)]{munos2008finite}
R{\'e}mi Munos and Csaba Szepesv{\'a}ri.
\newblock Finite-time bounds for fitted value iteration.
\newblock \emph{Journal of Machine Learning Research}, 9\penalty0 (5), 2008.

\bibitem[Namkoong et~al.(2020)Namkoong, Keramati, Yadlowsky, and Brunskill]{namkoong2020off}
Hongseok Namkoong, Ramtin Keramati, Steve Yadlowsky, and Emma Brunskill.
\newblock Off-policy policy evaluation for sequential decisions under unobserved confounding.
\newblock \emph{Advances in Neural Information Processing Systems}, 33:\penalty0 18819--18831, 2020.

\bibitem[Nilim and El~Ghaoui(2005)]{nilim2005robust}
Arnab Nilim and Laurent El~Ghaoui.
\newblock Robust control of markov decision processes with uncertain transition matrices.
\newblock \emph{Operations Research}, 53\penalty0 (5):\penalty0 780--798, 2005.

\bibitem[Olma(2021)]{olma2021nonparametric}
Tomasz Olma.
\newblock Nonparametric estimation of truncated conditional expectation functions.
\newblock \emph{arXiv preprint arXiv:2109.06150}, 2021.

\bibitem[Oprescu et~al.(2023)Oprescu, Dorn, Ghoummaid, Jesson, Kallus, and Shalit]{oprescu2023b}
Miruna Oprescu, Jacob Dorn, Marah Ghoummaid, Andrew Jesson, Nathan Kallus, and Uri Shalit.
\newblock B-learner: Quasi-oracle bounds on heterogeneous causal effects under hidden confounding.
\newblock In \emph{International Conference on Machine Learning}, pages 26599--26618. PMLR, 2023.

\bibitem[Panaganti et~al.(2022)Panaganti, Xu, Kalathil, and Ghavamzadeh]{panaganti2022robust}
Kishan Panaganti, Zaiyan Xu, Dileep Kalathil, and Mohammad Ghavamzadeh.
\newblock Robust reinforcement learning using offline data.
\newblock \emph{Advances in neural information processing systems}, 35:\penalty0 32211--32224, 2022.

\bibitem[Racine and Li(2017)]{racine2017nonparametric}
Jeffrey~S Racine and Kevin Li.
\newblock Nonparametric conditional quantile estimation: A locally weighted quantile kernel approach.
\newblock \emph{Journal of Econometrics}, 201\penalty0 (1):\penalty0 72--94, 2017.

\bibitem[Rockafellar and Uryasev(2002)]{rockafellar2002conditional}
R~Tyrrell Rockafellar and Stanislav Uryasev.
\newblock Conditional value-at-risk for general loss distributions.
\newblock \emph{Journal of banking \& finance}, 26\penalty0 (7):\penalty0 1443--1471, 2002.

\bibitem[Schick(1986)]{schick1986asymptotically}
Anton Schick.
\newblock On asymptotically efficient estimation in semiparametric models.
\newblock \emph{The Annals of Statistics}, pages 1139--1151, 1986.

\bibitem[Schulman et~al.(2017)Schulman, Wolski, Dhariwal, Radford, and Klimov]{schulman2017proximal}
John Schulman, Filip Wolski, Prafulla Dhariwal, Alec Radford, and Oleg Klimov.
\newblock Proximal policy optimization algorithms.
\newblock \emph{arXiv preprint arXiv:1707.06347}, 2017.

\bibitem[Semenova and Chernozhukov(2021)]{semenova2021debiased}
Vira Semenova and Victor Chernozhukov.
\newblock Debiased machine learning of conditional average treatment effects and other causal functions.
\newblock \emph{The Econometrics Journal}, 24\penalty0 (2):\penalty0 264--289, 2021.

\bibitem[Takeuchi et~al.(2006)Takeuchi, Le, Sears, Smola, and Williams]{takeuchi2006nonparametric}
Ichiro Takeuchi, Quoc~V Le, Timothy~D Sears, Alexander~J Smola, and Chris Williams.
\newblock Nonparametric quantile estimation.
\newblock \emph{Journal of machine learning research}, 7, 2006.

\bibitem[Tan(2006)]{tan2006distributional}
Zhiqiang Tan.
\newblock A distributional approach for causal inference using propensity scores.
\newblock \emph{Journal of the American Statistical Association}, 101\penalty0 (476):\penalty0 1619--1637, 2006.

\bibitem[Tsiatis(2006)]{tsiatis2006semiparametric}
Anastasios~A Tsiatis.
\newblock \emph{Semiparametric theory and missing data}, volume~4.
\newblock Springer, 2006.

\bibitem[Tsitsiklis and Van~Roy(1996)]{tsitsiklis1996analysis}
John Tsitsiklis and Benjamin Van~Roy.
\newblock Analysis of temporal-diffference learning with function approximation.
\newblock \emph{Advances in neural information processing systems}, 9, 1996.

\bibitem[Uehara et~al.(2021)Uehara, Imaizumi, Jiang, Kallus, Sun, and Xie]{uehara2021finite}
Masatoshi Uehara, Masaaki Imaizumi, Nan Jiang, Nathan Kallus, Wen Sun, and Tengyang Xie.
\newblock Finite sample analysis of minimax offline reinforcement learning: Completeness, fast rates and first-order efficiency.
\newblock \emph{arXiv preprint arXiv:2102.02981}, 2021.

\bibitem[van~der Laan et~al.(2011)van~der Laan, Rose, Zheng, and van~der Laan]{van2011cross}
Mark~J van~der Laan, Sherri Rose, Wenjing Zheng, and Mark~J van~der Laan.
\newblock Cross-validated targeted minimum-loss-based estimation.
\newblock \emph{Targeted learning: causal inference for observational and experimental data}, pages 459--474, 2011.

\bibitem[Van~der Vaart(2000)]{van2000asymptotic}
Aad~W Van~der Vaart.
\newblock \emph{Asymptotic statistics}, volume~3.
\newblock Cambridge university press, 2000.

\bibitem[Wainwright(2019)]{wainwright2019high}
Martin~J Wainwright.
\newblock \emph{High-dimensional statistics: A non-asymptotic viewpoint}, volume~48.
\newblock Cambridge university press, 2019.

\bibitem[Wang et~al.(2024{\natexlab{a}})Wang, Gao, and Zha]{wang2024reliable}
Jie Wang, Rui Gao, and Hongyuan Zha.
\newblock Reliable off-policy evaluation for reinforcement learning.
\newblock \emph{Operations Research}, 72\penalty0 (2):\penalty0 699--716, 2024{\natexlab{a}}.

\bibitem[Wang et~al.(2023{\natexlab{a}})Wang, Kallus, and Sun]{wang2023near}
Kaiwen Wang, Nathan Kallus, and Wen Sun.
\newblock Near-minimax-optimal risk-sensitive reinforcement learning with cvar.
\newblock In \emph{International Conference on Machine Learning}, pages 35864--35907. PMLR, 2023{\natexlab{a}}.

\bibitem[Wang et~al.(2023{\natexlab{b}})Wang, Zhou, Wu, Kallus, and Sun]{wang2023benefits}
Kaiwen Wang, Kevin Zhou, Runzhe Wu, Nathan Kallus, and Wen Sun.
\newblock The benefits of being distributional: Small-loss bounds for reinforcement learning.
\newblock \emph{Advances in Neural Information Processing Systems}, 36, 2023{\natexlab{b}}.

\bibitem[Wang et~al.(2024{\natexlab{b}})Wang, Kallus, and Sun]{wang2024central}
Kaiwen Wang, Nathan Kallus, and Wen Sun.
\newblock The central role of the loss function in reinforcement learning.
\newblock \emph{arXiv preprint arXiv:2409.12799}, 2024{\natexlab{b}}.

\bibitem[Wang et~al.(2024{\natexlab{c}})Wang, Liang, Kallus, and Sun]{wang2024risk}
Kaiwen Wang, Dawen Liang, Nathan Kallus, and Wen Sun.
\newblock Risk-sensitive rl with optimized certainty equivalents via reduction to standard rl.
\newblock \emph{arXiv preprint arXiv:2403.06323}, 2024{\natexlab{c}}.

\bibitem[Wang et~al.(2024{\natexlab{d}})Wang, Oertell, Agarwal, Kallus, and Sun]{wang2024more}
Kaiwen Wang, Owen Oertell, Alekh Agarwal, Nathan Kallus, and Wen Sun.
\newblock More benefits of being distributional: Second-order bounds for reinforcement learning.
\newblock \emph{International Conference of Machine Learning}, 2024{\natexlab{d}}.

\bibitem[Wang and Zou(2021)]{wang2021online}
Yue Wang and Shaofeng Zou.
\newblock Online robust reinforcement learning with model uncertainty.
\newblock \emph{Advances in Neural Information Processing Systems}, 34:\penalty0 7193--7206, 2021.

\bibitem[Wiesemann et~al.(2013)Wiesemann, Kuhn, and Rustem]{wiesemann2013robust}
Wolfram Wiesemann, Daniel Kuhn, and Ber{\c{c}} Rustem.
\newblock Robust markov decision processes.
\newblock \emph{Mathematics of Operations Research}, 38\penalty0 (1):\penalty0 153--183, 2013.

\bibitem[Xu et~al.(2023)Xu, Gao, and He]{pmlr-v202-xu23d}
Wenhao Xu, Xuefeng Gao, and Xuedong He.
\newblock Regret bounds for {M}arkov decision processes with recursive optimized certainty equivalents.
\newblock In Andreas Krause, Emma Brunskill, Kyunghyun Cho, Barbara Engelhardt, Sivan Sabato, and Jonathan Scarlett, editors, \emph{Proceedings of the 40th International Conference on Machine Learning}, volume 202 of \emph{Proceedings of Machine Learning Research}, pages 38400--38427. PMLR, 23--29 Jul 2023.
\newblock URL \url{https://proceedings.mlr.press/v202/xu23d.html}.

\end{thebibliography}
\bibliographystyle{plainnat}

%%%%%%%%%%%%%%%%%%%%%%%%%%%%%%%%%%%%%%%%%%%%%%%%%%%%%%%%%%%%%%%%%%%%%%%%%%%%%%%
%%%%%%%%%%%%%%%%%%%%%%%%%%%%%%%%%%%%%%%%%%%%%%%%%%%%%%%%%%%%%%%%%%%%%%%%%%%%%%%
% APPENDIX
%%%%%%%%%%%%%%%%%%%%%%%%%%%%%%%%%%%%%%%%%%%%%%%%%%%%%%%%%%%%%%%%%%%%%%%%%%%%%%%
%%%%%%%%%%%%%%%%%%%%%%%%%%%%%%%%%%%%%%%%%%%%%%%%%%%%%%%%%%%%%%%%%%%%%%%%%%%%%%%
\newpage
\appendix
\onecolumn
\setlength{\parindent}{0pt}
\setlength{\parskip}{\baselineskip}
\begin{center}\LARGE
\textbf{Appendices}
\end{center}

\section{Notations}\label{app:notation}

{\renewcommand{\arraystretch}{1.3}% for the vertical padding
\begin{table}[h!]
    \centering
      \caption{List of Notations} \vspace{0.3cm}
    \begin{tabular}{l|l}
    $\Scal, \Acal$ & State and action spaces. \\
    $\Delta(S)$ & The set of distributions supported by set $S$. \\
    $d_1$ & The initial state distribution. \\
    $\Lambda(s,a)$ & Tolerance parameter for kernel shift at $(s,a)$. Takes values $[1,\infty]$. \\
    $\tau(s,a)$ & $\tau(s,a) = \frac{1}{1+\Lambda(s,a)}\in[0,\frac12]$. \\
    $V^\pm,Q^\pm$ & Robust value and quality functions of the target policy $\pit$. \\
    $f(s,\pi)$ & $f(s,\pi):=\EE_{a\sim\pi(s)}[f(s,a)]$. \\
    $U^\pm(s'\mid s,a)$ & Robust transition kernel which attains the best- or worst-case value. \\
    $\Tcal_U,\Trob^\pm$ & Bellman operator under $U$ and the robust Bellman operators. \\
    $\Jcal_U$ & $\Jcal_U f(s,a) := \gamma\EE_U[ f(s',\pit)\mid s,a ]-f(s,a)$ \\
    $\beta_\tau^\pm(s,a)$ & The upper $\tau$-th quantile of $V^+(s')$ and lower $\tau$-th quantile of $V^-(s')$, $s'\sim P(s,a)$.\\
    $d^{\pit,\infty}_{d_1,U}$ & The $\gamma$-discounted average visitation of $\pit$ under MDP with transition $U$ starting from $d_1$.\\
    $d^{\pm,\infty}$ & $d^{\pm,\infty}=d^{\pit,\infty}_{d_1,U^\pm}$. \\
    $\nu(s),\nu(s,a)$ & Data generating distribution. $\nu(s)$ marginalizes over actions. \\
    $w^{\pm}$ & $w^{\pm}=\nicefrac{\diff d^{\pm,\infty}}{\diff\nu}$. This is valid both as a function of $s$ or $(s,a)$. \\
    $\omega(s,a)$ & $\omega(s,a) = \frac{\pit(a\mid s)}{\nu(a\mid s)}$. \\
    $x_+, x_-$ & $\max(0, x),\min (0, x)$ respectively, for $x\in \RR$.\\
    $x\lesssim y$ & $x\leq Cy$ for some constant $C$.\\
    $\EE_n$ & Empirical average over $n$ samples. \\
    $\magd{f}_p$ & $L^p$ norm, $(\EE|f(X)|^p)^{1/p}$.\\
    $f^\star$ & True (oracle) value of a parameter or function $f$.\\
    $f, \bar{f}$ & Putative value of a parameter or function $f$.\\
    $\wh f$ & Estimated value of a parameter or function $f$.
    \end{tabular}
    \label{tab:notation}
\end{table}
}

\section{Results for Policy Evaluation Under Best-Case Perturbations}\label{app:best-case-results} 
In this section, we present analogous results for the best-case perturbation under the uncertainty set, corresponding to the supremum case of \cref{eq:robust-q-primal-def}.
We derive a similar orthogonal estimator with the properties outlined in \pref{thm:efficiency}, following the same reasoning presented in the main text. 

\paragraph{$Q^+$ Identification and Estimation.} We present the results of \pref{lem:identification-q} for $\Trob^+$:
\begin{align*}
    \Trob^+ q(s,a) 
    &= r(s,a) + \gamma\Lambda^{-1}(s,a)\EE[ v(s')\mid s,a ]
    +\gamma(1-\Lambda^{-1}(s,a))\cvar^+_{\tau(s,a)}[ v(s')\mid s,a ].
\end{align*}
Next, applying \pref{asm:qr-oracle} and \pref{asm:robust-bc} to $\Trob^+$, we derive from \pref{thm:robust-fqe} for $Q^-$ that:
\begin{align*}
    &\nm*{ \wh q^+_M - Q^+ }_{d_1} \lesssim (1-\gamma)^{-2}\prns*{\sqrt{C^+_{d_1}}\cdot\eps_n^\Qcal+\QRerr^2(n/2M,\delta/2M)},~~\text{and}
    \\&\abs{(1-\gamma)\EE_{d_1}[ \wh v_M^+(s_1) ]- V_{d_1}^+} \lesssim \gamma^M + (1-\gamma)^{-1}\prns*{\sqrt{C^+_{d_1}}\cdot\eps_n^\Qcal+\QRerr^2(n/2M,\delta/2M)}.
\end{align*}

\paragraph{$w^+$ Identification and Estimation.} We first state the identification result for $U^-$ as in \pref{lem:identification-U-robust}:
\begin{equation*}
  {U^+(s'\mid s,a)}/{P(s'\mid s,a)} = \Lambda^{-1}(s,a)+(1-\Lambda^{-1})\tau(s,a)^{-1}\II\bracks*{\prns*{V^+(s')-\beta^+_{\tau}(s,a)}\geq 0}.
\end{equation*}
Then, under \pref{assum:regular} and \pref{asm:w-realizability-completeness} formulated for $U^+$, the minimax rates from \pref{thm:minimax-fast-rates-for-w} are given by:
\begin{equation*}
    \nm{\Jcal_{U^+}'(\wh w-w^+)}_2\lesssim \eps_n^\Wcal+\|\wt\zeta^+-\zeta^+\|_\infty+\sqrt{ {\log(1/\delta)}/{n} }.
\end{equation*}

\begin{algorithm}[t!]
\caption{Orthogonal Estimator for $V^{+}_{d_1}$}
\label{alg:ortho-ope-best-case}
\begin{algorithmic}[1]
    \STATE \textbf{Input:} Dataset $\Dcal$, number of splits $K$.
    \FOR{$k=1,2,\dots,K$}
        \STATE Use data $\Dcal \setminus \Dcal_k$ to learn $(q^{+, [k]},\beta^{+,[k]})$ with \pref{alg:robust-fqe} and $w^{+,[k]}$ with \pref{alg:robust-minimax}
        \STATE \textbf{for} $i=\lfloor (k-1)n/K\rfloor,\dots,\lfloor kn/K\rfloor-1$ \textbf{do} $\psi^{+}_i = \psi(s_i,a_i,s'_i, \widehat{\eta}^{+})$
    \ENDFOR
    \STATE\textbf{Output: } $\widehat{V}^{+}_{d_1} = \frac{1}{n}\sum_{i=1}^n \psi^{+}_i$.
    \end{algorithmic}
\end{algorithm}

\paragraph{Orthogonal and Efficient Estimator for $V^+_{d_1}$.} Let the set of nuisance parameters be denoted by $\eta^{+} = (w^{+}, q^{+}, \beta^{+})$. Then, the (recentered) efficient influence function (R)EIF (see \pref{thm:eif-main}) for in $V^+_{d_1}$ is formulated as: 
\begin{align*}
        &\psi(s, a, s';\eta^{+}) = V^{+}_{d_1} + w^{+}(s,a)\big(r(s,a)+\gamma\rho^{+}(s,a,s'; v^{+}, \beta^{+}) -q^{+}(s,a)\big),\quad\text{where}
    \\
        & \rho^{+}(s,a,s'; v^{+}, \beta^{+}) = \Lambda(s,a)^{-1}v^{+}(s') +  (1-\Lambda(s,a)^{-1})\big(\beta^{+}(s,a) + \tau^{-1}(v^{+}(s')-\beta^{+}(s,a))_{+}\big).
\end{align*}
Using this (R)EIF, the orthogonal estimator for $V^+_{d_1}$ is presented in \pref{alg:ortho-ope-best-case}. We now restate \pref{thm:efficiency} for $\wh V^+_{d_1}$:
\begin{theorem}[Efficiency of $\widehat{V}^+_{d_1}$]\label{thm:efficiency-best-case}
    Let $r_{n, p}^w,r_{n,p}^q, r_{n,p}^{\beta}$ be functions of $n=|\Dcal|$ such that $\|\Jcal_{U^+}'(\widehat{w}^{+, [k]}-w^*)\|_p\leq r_{n,p}^w$, $\|\widehat{q}^{+, [k]}-q^*\|_p\leq r_{n,p}^q$, and $\|\beta^{+, [k]}-\beta^*\|_p\leq r_{n,p}^{\beta}$ for any $k\in [K]$. Furthermore, assume that the regularity conditions in \pref{assum:regular} hold. Then:
    \begin{align}
        |\widehat{V}_{d_1}^+-V_{d_1}| \lesssim O_p(n^{-1/2}) + O_p(r_{n,2}^wr_{n,2}^q+(r_{n, \infty}^{q})^2+(r_{n,\infty}^{\beta})^2)\tag{Rates}\label{eq:rates-for-plus}
    \end{align}
    Furthermore, if $r_{n,2}^w \vee r_{n,2}^q=o_p(1)$, $r_{n,2}^w r_{n,2}^q=o_p(n^{-1/2})$, $r_{n,\infty}^{q}=o_p(n^{-1/4})$, and $r_{n,\infty}^{\beta}=o_p(n^{-1/4})$, then $\widehat{V}^+_{d_1}$ satisfies:
    \begin{align}
        \sqrt{n}(\widehat{V}^+_{d_1}-V_{d_1}) \xrightarrow{d} \Ncal(0, \Sigma) \tag{Normality \& Efficiency}\label{eq:efficiency-for-plus},\quad\Sigma=\mathrm{Var}(\psi(s,a,s';\eta^+)).
    \end{align}
    Moreover, $\Sigma$ is the minimum achievable asymptotic variance among RAL estimators in the nonparametric model for $(s,a,s')$ (the efficiency bound).
\end{theorem}

\section{Additional Related Works}
\paragraph{Robust MDPs.}
There is a rich literature on Robust MDPs \citep{iyengar2005robust,wiesemann2013robust,mannor2016robust,goyal2023robust} with $s,a$-rectangular uncertainty sets,
but these foundational works assumed knowledge of the transition kernel.
Recently, learning-based robust MDP algorithms have been proposed for uncertainty sets under the total variation \citep{panaganti2022robust,kumar2023policy} and more generally $L_p$ balls \citep{kumar2022efficient}.
These $L_p$ uncertainty sets are additive in nature, \ie, the adversary adds or subtracts a vector in the $\ell_p$ ball to $P(\cdot\mid s,a)$, whereas our uncertainty set is multiplicative in nature, \ie, the adversary can multiply or divide a bounded factor and is more commonly used in causal inference to model unobserved confounding.
In the contextual bandit setting, \citep{kallus2022doubly} also derived efficiency bounds for robust OPE where both state distribution and reward distributions may shift -- their work is however restricted to the one-step bandit setting while our full RL setting is more challenging.

\paragraph{Risk-Sensitive RL.}
Risk-sensitive RL is the problem of optimizing the risk measure of cumulative rewards \citep{howard1972risk} and is tightly related to robust MDPs \citep{chow2015risk}. 
For example, as we proved in \pref{lem:identification-q}, the MSM uncertainty set is indeed equivalent to risk-sensitive RL with the dynamic risk measure $\Lambda \EE + (1-\Lambda)\cvar_\tau$.
We note that efficient online RL algorithms have been proposed for similar measures \citet{du2022provably,pmlr-v202-xu23d}. Static risk-sensitive RL also modifies the Bellman equations in an augmented MDP \citep{wang2023near,wang2024risk}.
Our focus is on deriving the optimal \emph{off-policy evaluation} estimators for the problem, which involves a different set of challenges such as deriving the efficiency bound and ensuring sharpness guarantees even when nuisances are estimated slowly.

\section{Additional Technical Details}

\subsection{Higher Order Norms via Smoothness}

For any $x\in \RR^+$, define $\lfloor x \rfloor$ as the greatest integer that is strictly less than $x$, and let  $x$ and $\{x\}=x-\lfloor x\rfloor$ represent the fractional part. Thus, we obtain the distinct decomposition $x = \lfloor x \rfloor + \{x\}$, where $\lfloor x \rfloor \in \NN$ and $\{x\} \in (0,1]$.

\begin{definition}[$\alpha$-smooth functions]\label{def:a-smooth}
   Given $\alpha\in(0, \infty)$ and $\Xcal \subseteq \RR^m$, $f:\Xcal\rightarrow\RR$ is an $\alpha$-smooth function if (1) the mixed derivatives up to $\lfloor \alpha \rfloor$-order exist and are bounded; and (2) all $\lfloor \alpha \rfloor$-order derivatives are $\{\alpha\}$-H\"older continuous \citep{leoni2017first}. 
\end{definition}

\begin{lemma}[$L^\infty$ Bound for $\alpha$-Smooth Functions]\label{lem:smoothness} Let $f:\Xcal\rightarrow\RR, \Xcal \subseteq \RR^m$ be an $\alpha$-smooth function as in \pref{def:a-smooth}. Then, if $\Xcal$ is $\RR^m$, a half-space or a bounded Lipschitz domain in $\RR^m$, there exists a constant $C$ such the following inequality holds:
\begin{align*}
    \|f\|_\infty \leq C\|f\|_p^{\frac{p\alpha}{p\alpha + m}}.
\end{align*}
\end{lemma}
\begin{proof}
    This lemma is a direct application of the fractional Gagliardo-Nirenberg interpolation inequality (Theorem 1 in \citet{brezis2019sobolev}) from the functional analysis literature. For a more comprehensive exposition on this result, see Appendix A.1 in \citet{bennett2023low}. 
\end{proof}

\subsection{Localized Rademacher Complexity and Critical Radius}\label{app:critical-radius}
Here, we recap the localized Rademacher complexity and critical radius which is a standard complexity measure for obtaining fast rates for squared loss \citep{wainwright2019high}.
Let $\Gcal$ be a class of functions $g:\Zcal\to\RR$. 
Given $n$ datapoints $z_1,z_2,\dots,z_n$, the empirical localized Rademacher complexity is:
\begin{align*}
    \Rcal_n(\eps,\Gcal):=\EE_\sigma\bracks{ \sup_{g\in\Gcal:\|g\|_n\leq \eps}\frac1n\sum_{i=1}^n\epsilon_ig(z_i) },
\end{align*}
where $\EE_\sigma$ is expectation over $n$ independent Rademacher random variables $\sigma_1,\sigma_2,\dots,\sigma_n$, \ie, $\EE_\sigma[\cdot] = \frac{1}{2^n}\sum_{\sigma\in\{-1,1\}^n}[\cdot]$. Note that when $\eps=\infty$, there is no localization and $\Rcal_n(\infty,\Gcal)$ reduces to the vanilla Rademacher complexity. 
Let $C:=\sup_{g\in\Gcal}\|g\|_\infty$ be the envelope of $\Gcal$. Then, the critical radius of $\Gcal$ with $n$, called $\eps_n$, is the smallest $\eps$ that satisfies $\Rcal_n(\eps,\Gcal)\leq\nicefrac{\eps^2}{C}$.

Unless otherwise stated, we will posit that $\Gcal$ is star-shaped: there exists $g_0\in\Gcal$ such that for all $g\in\Gcal$ and $\alpha\in[0,1]$, we have $\alpha g_0 + (1-\alpha)g\in\Gcal$. If not, we can replace $\Gcal$ by its star-hull, \ie, the smallest star-shaped set containing $\Gcal$. We will also posit that $\Gcal$ is symmetric for simplicity.

The critical radius is a well-studied quantity in statistics \citep{wainwright2019high} and also recently in RL \citep{duan2021risk,uehara2021finite}.
For example if $\Gcal$ has $d$ VC-subgraph dimension, then w.p. $1-\delta$, $\eps_n\leq\Ocal(\sqrt{d\log n/n})$. 
For nonparametric models with metric entropy at most $1/t^\beta$, the critical radius can also be bounded by $\Ocal(n^{-1/(\max(2+\beta,2\beta))})$ \citep{uehara2021finite}, \eg, is $\Ocal(n^{-1/4})$ if $\beta=2$.

\section{Proofs for Identification Results}

\subsection{Identification of robust \texorpdfstring{$Q$}{Q}}
\identificationQ*
\begin{proof}
Consider the uncertainty set in $\Tcal_{\textsf{rob}}$ where the constraint on $U$ (\pref{eq:uncertainty-set}) can be rewritten as:
\begin{align*}
  \textstyle 0 \leq \frac{U(s'\mid s,a)-\Lambda^{-1}(s,a)P(s'\mid s,a)}{P(s'\mid s,a)} \leq \Lambda(s,a)-\Lambda^{-1}(s,a).
\end{align*}
Therefore, we can write $U(s'\mid s,a) = \Lambda^{-1}(s,a)P(s'\mid s,a)+(1-\Lambda^{-1})G(s'\mid s,a)$ where we define $G(s'\mid s,a) := \frac{U(s'\mid s,a)-\Lambda^{-1}(s,a) P(s'\mid s,a)}{1-\Lambda^{-1}(s,a)}$.
Thus, the constraints on $G$ are that $G(\cdot\mid s,a)\ll P(\cdot\mid s,a)$ and $\|\frac{\diff G(s'\mid s,a)}{\diff P(s'\mid s,a)}\| \leq \Lambda(s,a)+1$.
Setting $\tau(s,a)=\frac{1}{\Lambda(s,a)+1}$, we can apply the primal form of $\cvar$ \citep{dorn2021doubly,ang2018dual} to obtain
\begin{align*}
  \inf_{G\ll P: \|\frac{\diff G(\cdot\mid s,a)}{\diff P(\cdot\mid s,a)}\|_\infty\leq \tau^{-1}(s,a) }\EE_G[f(s')] = \cvar_{\tau(s,a)}^-\bracks{ f(s')\mid s,a }.
\end{align*}
Therefore, the supremum in $\Tcal_{\textsf{rob}}$ can be expressed as $\Lambda^{-1}(s,a)$ times the expectation under nominal $P$ and $(1-\Lambda^{-1}(s,a))$ times the above CVaR expression, which finishes the proof of the $-$ case.

For the $+$ case, we can simply use $\sup$ instead of $\inf$ and upper CVaR instead of lower CVaR.
\end{proof}

\subsection{Identification of robust kernel and visitation}\label{app:proof-identification-kernel}

\identifcationU*
\begin{lemma}\label{lem:identification-U-robust-general}
Fix any $v:\Scal\to\RR$ and define the pushforward $F_{v}(y\mid s,a)=P(v(s')\leq y\mid s,a)$. Suppose $F_{v}(\beta_{\tau,F_{v}(\cdot\mid s,a)}^\pm(s,a)\mid s,a)=\frac12\pm(\frac12-\tau)$, where $\beta_{\tau,F_{v}}^\pm$ is the upper/lower $\tau$-quantile of $F_v$.
Then, $\sup_{U\in\Ucal(P)}\EE_U[v(s')\mid s,a]=\EE_{s'\sim U_v^+(s,a)}[v(s')]$ and $\inf_{U\in\Ucal(P)}\EE_U[v(s')\mid s,a]=\EE_{s'\sim U_v^-(s,a)}[v(s')]$, where 
\begin{equation*}
  {U^\pm_v(s'\mid s,a)}/{P(s'\mid s,a)} = \Lambda^{-1}(s,a)+(1-\Lambda^{-1})\tau(s,a)^{-1}\II\bracks*{\pm\prns*{v(s')-\beta^\pm_{\tau,F_{v}(\cdot\mid s,a)}(s,a)}\geq 0}.
\end{equation*}
\end{lemma}
\begin{proof}
We start with some intuitions.
First, if the CDF of $v(s')$ is differentiable $\beta^+_\tau(s,a)$, then $\cvar_{\tau}^+(v(s')\mid s,a)=\EE[ v(s')\mid f(s')\geq\beta^+_\tau(s,a), s, a ]$ and the result follows immediately from \pref{lem:identification-q} by noticing that the form of $U^+$ exactly recovers the convex combination of expectation and CVaR.
Alternatively, one can use the closed form solution of the primal CVaR as derived in \citep{ang2018dual} to obtain the result.

We now provide a formal proof.
Fix any $s,a$ and let $\tau=\tau(s,a)$. Fix any function $v(s')\in\mathbb{R}$. We want to show that the worst-case $U^+ = \arg\max _ {U\in\mathcal{U}(P)} \mathbb{E} _ U[v(s')\mid s,a]$ has a closed form expression as shown in line 725. By the proof of Lemma 3.1 above, we can rewrite $U^+(s'\mid s,a) = \Lambda^{-1}(s,a)P(s'\mid s,a) + (1-\Lambda^{-1}(s,a)) G^+(s'\mid s,a)$, where $G^+ = \arg\max _ {G\ll P: |dG(\cdot\mid s,a)/dP(\cdot\mid s,a)| _ \infty\leq \tau^{-1}(s,a)}\mathbb{E} _ G[v(s')]$. Thus, it suffices to simplify $G^+$. To do so, we invoke the premise that the CDF of $v(s')$ is differentiable at $\beta^+ _ \tau$, i.e. $F _ v(\beta^+ _ {\tau,F _ v}(s,a)\mid s,a)=1-\tau$. This implies that the CVaR is exactly the conditional expectation of the $1-\tau(s,a)$-fraction of best outcomes, i.e. $\text{CVaR}^+ _ \tau(v(s')\mid s,a) = \mathbb{E}[v(s')\mid v(s')\geq\beta^+ _ \tau(s,a),s,a]$, which in turn is equal to $\tau^{-1}\mathbb{E}[v(s')\mathbb{I}[v(s')\geq\beta^+_\tau(s,a)]\mid s,a]$. Thus, $G^+(s'\mid s,a) = \tau^{-1}P(s'\mid s,a) \mathbb{I}[v(s')\geq\beta^+ _ \tau(s,a)]$. This concludes the proof for the $+$ case. The proof for the $-$ case follows identical steps.
\end{proof}

\section{Proofs for Robust FQE}
We prove a more general result with approximate completeness, which shows that \pref{thm:robust-fqe} is robust to approximate completeness.
\begin{assumption}[Approximate Completeness]\label{asm:approx-robust-bc}
$\max_{q\in\Qcal}\min_{g\in\Qcal}\|g-\Tcal^\pm_{\cvar}q\|_{\nu}\leq\epsQComp$.
\end{assumption}
\begin{theorem}
Assume \pref{asm:approx-robust-bc}. Under the same setup as \pref{thm:robust-fqe}, we have 
\begin{align*}
    \nm{ \wh q^\pm_K - Q^\pm }_\mu \lesssim \frac{1}{(1-\gamma)^2}\prns*{\sqrt{C^\pm_\mu}\cdot\prns{\eps_n^\Qcal+\epsQComp}+\QRerr^2(n/2K,\delta/2K)},
\end{align*}
and
\begin{align*}
    \abs{V_{d_1}^\pm - (1-\gamma)\EE_{d_1}[ \wh q_K^\pm(s_1,\pit) ]} \lesssim \gamma^K + \frac{1}{1-\gamma}\prns*{\sqrt{C^\pm_\mu}\cdot\prns{\eps_n^\Qcal+\epsQComp}+\QRerr^2(n/2K,\delta/2K)}.
\end{align*}
\end{theorem}

\begin{proof}
Let $U^\pm$ denote the worst-case kernel that satisfies $V^\pm_{d_1} = (1-\gamma)\EE_{d_1}V^{\pit}_{U^\pm}(s_1)$. Then,
\begin{align*}
    V^\pm_{d_1}-(1-\gamma)\EE_{d_1}[\wh q^\pm_K(s_1,\pit)]
    &= (1-\gamma)\EE_{d_1}[V^{\pit}_{U^\pm}(s_1) - \wh q_K(s_1,\pit)]
    \\&= \EE_{d^{\pi,\infty}_{U^\pm}}[ \Tcal^{\pit}_{U^\pm} \wh q_K(s,a) - \wh q_K(s,a) ] \tag{\pref{lem:performance-difference}}
    \\&\leq \frac{4}{1-\gamma}\max_{k=1,2,\dots}\nm{ \wh q_k - \Tcal^{\pit}_{U^\pm} \wh q_{k-1} }_{d^{\pit,\infty}_{U^\pm}}+\gamma^{K/2}.\tag{\pref{lem:fqe-unrolling}}
\end{align*}
Consider any $k=1,2,\dots$.
By definition of $U^\pm$, we have
\begin{align*}
    \nm{ \wh q_k - \Tcal^{\pit}_{U^\pm} \wh q_{k-1} }_{d^{\pit,\infty}_{U^\pm}}
    =\nm{ \wh q_k - \Tcal_{\beta^\star_k}^\pm \wh q_{k-1} }_{d^{\pm,\infty}}, \tag{by def of $U^\pm$}
\end{align*}
where $\beta^\star_k(s,a)$ is the true quantile of $\wh v_{k-1}(s')$.
Denote $q^\star_k := \Trob^\pm\wh q_{k-1}$ and let $\beta_k^\star$ be the true upper/lower quantile of $\wh q_{k-1}$.
Recall the population loss function is
\begin{align*}
    L_k(q,\beta)
    &:= \EE\bracks{ \prns{ y_k^\beta(s,a,s')-q(s,a) }^2 }
    \\y_k^\beta(s,a,s')&=r(s,a)+\gamma\Lambda^{-1}(s,a)\wh v_{k-1}(s')
    \\  &+ \gamma(1-\Lambda^{-1}(s,a))\prns{ \beta(s,a)+\tau^{-1}(s,a)\prns{ \wh v_{k-1}(s')-\beta(s,a) }_\pm }.
\end{align*}
The empirical loss $\wh L_k(q,\beta)$ is if $\EE$ is replaced by $\EE_n$.
Note that $\wh q_k=\argmin_{q\in\Qcal}\wh L_k(q,\wh\beta_k)$.

\paragraph{Nonparametric Least Squares with Model Misspecification.}
We will directly invoke \citet[Theorem 13.13]{wainwright2019high},
which gives a fast rate for misspecified least squares with general nonparametric classes.
We now bound the misspecification.
Recall that at the $k$-th iteration, our regression Bayes-optimal is $\EE[y^{\wh\beta_k}_k(s,a,s')\mid s,a] = \Tcal_{\wh\beta_k}\wh q_{k-1}(s,a)$.
By \pref{lem:margin-guarantees}, we know this is close to $\Tcal_{\beta^\star_k}\wh q_{k-1}(s,a)$ with second order errors in $\beta$: for any $\mu$, we have
\begin{equation*}
    \nm{ \Tcal_{\wh\beta_k}^\pm\wh q_{k-1}-\Tcal_{\beta^\star_k}^\pm\wh q_{k-1} }_{d^{\pm,\infty}_\mu} \lesssim \|\wh\beta_k-\beta^\star_k\|_\infty^2.
\end{equation*}
Finally, by approximate completeness (\pref{asm:approx-robust-bc}), there exists $g\in\Qcal$ such that $\|\Tcal_{\beta^\star_k}\wh q_{k-1}(s,a)-g\|\leq\epsQComp$.
Putting this together: for any $k$, there exists a $g\in\Qcal$ such that
\begin{align*}
    \|g-\Tcal_{\wh\beta_k}\wh q_{k-1}(s,a)\|_{d^{\pm,\infty}_\mu}
    &\leq\|g-\Tcal_{\beta_k^\star}\wh q_{k-1}(s,a)\|_{d^{\pm,\infty}_\mu} + \|\Tcal_{\beta_k^\star}\wh q_{k-1}(s,a)-\Tcal_{\wh\beta_k}\wh q_{k-1}(s,a)\|_{d^{\pm,\infty}_\mu}
    \\&\leq\sqrt{C^\pm_\mu}\cdot\epsQComp + \|\wh\beta_k-\beta^\star_k\|_\infty^2.
\end{align*}
Therefore, \citet[Theorem 13.13]{wainwright2019high} (and concentration of least squares) certifies that:
\begin{align*}
    \nm{ \wh q_k-\Tcal_{\wh\beta_k}\wh q_{k-1} }_{d^{\pm,\infty}} \lesssim \sqrt{C^\pm_\mu}\cdot\prns{\epsQComp + \eps_n} + \|\wh\beta_k-\beta^\star_k\|_\infty^2.
\end{align*}
Therefore, we have proven:
\begin{align*}
    \nm{ \wh q_k-\Tcal_{\beta^\star_k}^\pm\wh q_{k-1} }_{d^{\pm,\infty}_\mu} 
    &\leq \nm{ \wh q_k-\Tcal_{\wh\beta_k}^\pm\wh q_{k-1} }_{d^{\pm,\infty}_\mu} + \nm{ \Tcal_{\wh\beta_k}^\pm\wh q_{k-1}-\Tcal_{\beta^\star_k}^\pm\wh q_{k-1} }_{d^{\pm,\infty}_\mu} 
    \\&\lesssim \sqrt{C^\pm_\mu}\cdot\prns{\epsQComp + \eps_n} + \|\wh\beta_k-\beta^\star_k\|_\infty^2.
\end{align*}
This concludes the proof.
\end{proof}

\begin{lemma}[Performance Difference]\label{lem:performance-difference}
For any $\pi$, transition kernel $P$, and function $f:\Scal\times\Acal\to\RR$, we have
\begin{align*}
    V^{\pi}_{P} - \EE_{s\sim d_1}[f(s,\pi)] = \frac{1}{1-\gamma}\EE_{d^{\pi,\infty}_{P}}[ \Tcal^{\pi}_Pf(s,a) - f(s,a) ].
\end{align*}
\end{lemma}
\begin{proof}
See Lemma C.1 of \citep{chang2022learning}.
\end{proof}

\begin{lemma}[Unrolling]\label{lem:fqe-unrolling}
For any $\pi$, transition kernel $P$, and functions $f_0,f_1,\dots,f_K:\Scal\times\Acal\to\RR$ satisfying $f_0(s,a)=0$, we have $\nm{f_K-\Tcal^{\pi}_P f_K}_{d^{\pi,\infty}_P}\leq \frac{4}{1-\gamma}\max_{k=1,2,\dots}\nm{ f_k - \Tcal^\pi_P f_{k-1} }_{d^{\pi,\infty}_P}+\gamma^{K/2}$.
\end{lemma}
\begin{proof}
See Lemma C.2 of \citep{chang2022learning}.
\end{proof}

\section{Proofs for Robust Minimax Algorithm}\label{app:proofs-minimax}
\begin{assumption}[Approximate $W$-realizability and completeness]\label{asm:approximate-w-realizability-and-completeness}
Assume the following hold for $\Wcal$ and $\Fcal$:
\\(A) Approximate realizability: $\min_{w\in\Wcal}\nm{ \Jcal_{U^\pm}(w^\pm-w) }_2\leq\epsWReal$; 
\\(B) Approximate completeness: $\max_{w\in\Wcal}\min_{f\in\Fcal}\nm{f-\Jcal_{U^\pm}'(w-w^\pm)}_2\leq\epsWComp$.
\end{assumption}

We prove a more general result with approximate realizability and completeness, which implies \pref{thm:minimax-fast-rates-for-w} that is robust to misspecification in its assumptions.
\begin{theorem}
Under \pref{asm:approximate-w-realizability-and-completeness} and the same setup as \pref{thm:minimax-fast-rates-for-w}, we have
\begin{align*}
    \nm{\Jcal_{U^\pm}'(\wh w-w^\pm)}_2\lesssim \eps_n^\Wcal+\|\wt\zeta^\pm-\zeta^\pm\|_\infty+\sqrt{ \frac{\log(1/\delta)}{n} }+\epsWReal+\epsWComp.
\end{align*}
\end{theorem}
\begin{proof}
For this proof, we focus on the worst-case kernel $P^\star$ of the form $\frac{P^\star(s'\mid s,a)}{P(s'\mid s,a)}=\tau^{-1}(s,a)\II\bracks*{\zeta^\star(s,a,s')\leq 0}$ where $\zeta^\star(s,a,s')=V^-(s')-\beta^-(s,a)$. This corresponds to the pure CVaR case of $\Trob^-$; the $\EE$ part is identical to standard non-robust RL so we omit it. The best-case kernel $U^+$ can be handled similarly. 
Let $\wh P(s'\mid s,a)$ denote our estimated robust kernel, which satisfies $\frac{\wh P(s'\mid s,a)}{P(s'\mid s,a)}=\tau^{-1}(s,a)\II\bracks*{\wh\zeta(s,a,s')\leq 0}$, where $\wh\zeta(s,a,s')$ is the given prior stage estimate of $\zeta^\star(s,a,s') = V^-(s')-\beta^-(s,a)$.

The key and only difference between our \pref{alg:robust-minimax} and the MIL algorithm ($\wh w_{\text{mil}}$) of \citet{uehara2021finite} is that our next-state samples are importance weighted with $\xi^\pm(s,a,s')$, which is the density ratio of the estimated robust kernel $\wh P(s'\mid s,a)$ and the nominal kernel $P(s'\mid s,a)$. 
Note also that $\xi^\pm(s,a,s')\leq \tau^{-1}(s,a)<\infty$, and hence $\abs*{\EE_n[ \zeta(s,a,s')f(s') ]-\EE_{s,a\sim\nu,s'\sim\wh P(s,a)}[f(s')]}\lesssim \sqrt{\log(1/\delta)/n}$ w.p. $1-\delta$. 
Therefore, up to $\Ocal(\sqrt{\log(1/\delta)/n})$ errors, our \pref{alg:robust-minimax} can be viewed as MIL applied to the MDP with kernel $\wh P$.

To invoke the result of \citet[Theorem 6.1]{uehara2021finite} (in MDP with kernel $\wh P$), we need to show that its assumptions are met by bounding the model misspecification, \ie, Eq.\,(6) and Appendix C of \citet{uehara2021finite}. Note that these misspecifications are w.r.t. the MDP with kernel $\wh P$, since this is the MDP in which we're applying Theorem 6.1 of \citet{uehara2021finite}. 
Specifically, the two errors we need to bound are, 
(A) approximate realizability: $\eps_A=\min_{w\in\Wcal}\|\Jcal_{\wh P}'(w_{\wh P} - w)\|_2$; 
and (B) approximate completeness: $\eps_B = \max_{w\in\Wcal}\min_{f\in\Fcal}\|f-\Jcal_{\wh P}'(w-w_{\wh P})\|_2$ where recall that $\Jcal_P$ is the linear operator defined as $\Jcal_P f(s,a) := \gamma\EE_P[f(s',\pit)\mid s,a]-f(s,a)$ and $\Jcal_P'$ is the adjoint.

\paragraph{Bounding misspecifications by $\|\wh\zeta-\zeta^\star\|_\infty$.}
Since $\zeta^\star(s,a,s')$ has a marginal CDF that's boundedly differentiable around $0$ (\ie, (ii) of \pref{assum:regular}), \citet[Lemma 3]{kallus2022s} implies that $\zeta^\star(s,a,s')$ satisfies a $1$-margin (\pref{def:margin}). Hence, \pref{lem:margin-guarantees} and the continuity of $\zeta^\star(s,a,s')$ implies that 
\begin{align*}
    &\Pr\prns{\II\bracks*{\wh\zeta(s,a,s')\leq 0}\neq \II\bracks*{\zeta^\star(s,a,s')\leq 0}} 
    \\&=\Pr\prns{(\II\bracks*{\wh\zeta(s,a,s')\leq 0}\neq \II\bracks*{\zeta^\star(s,a,s')\leq 0}), \zeta^\star(s,a,s')\neq 0} 
    \lesssim \|\wh\zeta-\zeta^\star\|_\infty,
\end{align*}
Thus, for any $v:\Scal\to\RR$,
\begin{align*}\EE\abs{(\EE_{\wh P}-\EE_{P^\star})[ v(s')\mid s,a ]} &\leq
\EE[\tau^{-1}(s,a)\prns*{\II\bracks*{\wh\zeta(s,a,s')\leq 0}\neq \II\bracks*{\zeta^\star(s,a,s')\leq 0}}\cdot\abs{v(s')}]
\\&\lesssim 
\|v\|_\infty\cdot \Pr\prns{\II\bracks*{\wh\zeta(s,a,s')\leq 0}\neq \II\bracks*{\zeta^\star(s,a,s')\leq 0}} 
\\&\lesssim \|v\|_\infty\|\wh\zeta-\zeta^\star\|_\infty,
\end{align*}
or equivalently
\begin{equation}
    \EE\|\wh P(\cdot\mid s,a)-P^\star(\cdot\mid s,a)\|_{\TV}\lesssim \|\wh\zeta-\zeta^\star\|_\infty. \label{eq:tv-p-bound}
\end{equation}

Equipped with \pref{eq:tv-p-bound}, we can now bound the following two types of errors: (i) $\langle f,(\Tcal_{P^\star}-\Tcal_{\wh P})g\rangle$, and (ii) $\langle w_{\hat P}-w_{P^\star},h\rangle$, where $f,g:\Scal\times\Acal\to\RR$ and $h:\Scal\to\RR$, and $\Tcal_{P}$ and $w_{P}$ are the Bellman operator and visitation density of target policy $\pit$ in the MDP with kernel $P$.

For (i):
\begin{align*}
    \abs{\langle f,(\Jcal_{P^\star}-\Jcal_{\wh P}) g\rangle}
    &=\abs{\EE[ f(s,a)\prns{ \gamma(\EE_{P^\star}-\EE_{\wh P})[g(s',\pit)\mid s,a] } ]}
    \\&\leq \gamma\|f\|_\infty \EE\abs{ (\EE_{P^\star}-\EE_{\wh P})[g(s',\pit)\mid s,a] }
    \\&\lesssim \gamma\|f\|_\infty\|g(\cdot,\pit)\|_\infty \|\wh\zeta-\zeta^\star\|_\infty.
\end{align*}
For (ii):
\begin{align*}
    \langle w_{\wh P}-w_{P^\star},h\rangle
    &=\EE[ (w_{\wh P}(s)-w_{P^\star}(s))h(s) ]
    \\&\leq \|h\|_\infty \|d_{\wh P}-d_{P^\star}\|_{\TV}
    \\&\leq \|h\|_\infty \frac{\gamma}{1-\gamma} \EE_{d_{P^\star}}\|\wh P(\cdot\mid s,a)-P^\star(\cdot\mid s,a)\|_{\TV} \tag{\pref{eq:pdl-for-visitations}}
    \\&\lesssim C\|h\|_\infty \frac{\gamma}{1-\gamma} \EE\|\wh P(\cdot\mid s,a)-P^\star(\cdot\mid s,a)\|_{\TV} \tag{\pref{assum:regular}(i)}
    \\&\lesssim C\|h\|_\infty \frac{\gamma}{1-\gamma} \|\wh\zeta-\zeta^\star\|_{\infty},
\end{align*}
where $C=\nm*{ \nicefrac{\diff d^{P^\star}}{\diff\nu} }_{\infty}<\infty$.

For approximate realizability ($\eps_A$): for any $w\in\Wcal$, we have
\begin{align*}
    &\|\Jcal_{\wh P}'(w_{\wh P} - w)\|_2
    \\&\leq \|(\Jcal_{\wh P}-\Jcal_{P^\star})'(w_{\wh P}-w)\|_2 + \|\Jcal_{P^\star}'(w_{\wh P}-w_{P^\star})\|_2 + \|\Jcal_{P^\star}'(w^\star-w)\|_2
    \\&= \langle w_{\wh P}-w, (\Jcal_{\wh P}-\Jcal_{P^\star}) g_1\rangle + \langle w_{\wh P}-w_{P^\star}, \Jcal_{P^\star} g_2\rangle + \|\Jcal_{P^\star}'(w^\star-w)\|_2
    \\&\lesssim \nm*{ \wh\zeta-\zeta^\star }_{\infty} + \|\Jcal_{P^\star}'(w^\star-w)\|_2
\end{align*}
where $g_1=\prns*{(\Jcal_{P^\star}-\Jcal_{\wh P})'(w_{\wh P}-w)}/\nm*{(\Jcal_{P^\star}-\Jcal_{\wh P})'(w_{\wh P}-w)}_2$, $g_2 = \prns*{\Jcal_{P^\star}'(w_{\wh P}-w_{P^\star})}/\nm*{\Jcal_{P^\star}'(w_{\wh P}-w_{P^\star})}_2$.
The last inequality uses (i) and (ii) with the fact that $\|g_1\|_\infty <\infty$ and $\|g_2\|_\infty<\infty$ as the $w$ terms are bounded by our premise.
Therefore, taking min over $w$ and using \pref{asm:approximate-w-realizability-and-completeness}, we have $\eps_A\lesssim\|\wh\zeta-\zeta^\star\|_{\infty}+\epsWReal$.

For approximate completeness ($\eps_B$): for any $w\in\Wcal$ and $f\in\Fcal$, we have
\begin{align*}
    &\|f-\Jcal_{\wh P}'(w-w_{\wh P})\|_2
    \\&\leq \|f-\Jcal'_{P^\star}(w-w_{P^\star})\|_2 + \|(\Jcal_{P^\star}-\Jcal_{\wh P})'(w-w_{P^\star})\|_2 + \|\Jcal_{P^\star}'(w_{\wh P}-w_{P^\star})\|_2
    \\&\lesssim \|f-\Jcal'_{P^\star}(w-w_{P^\star})\|_2 + \nm*{ \wh\zeta-\zeta^\star }_\infty,
\end{align*}
for the same reason as $\eps_A$ as the error terms are the same. 
Thus, $\eps_B\lesssim \|\wh\zeta-\zeta^\star\|_\infty + \epsWComp$.

In sum, we have shown that the misspecification is at most $\Ocal(\|\wh\zeta-\zeta^\star\|_\infty+\epsWReal+\epsWComp)$. Therefore, \citet[Theorem 6.1 and Appendix C]{uehara2021finite} ensures that w.p. $1-\delta$, our learned $\wh w$ satisfies,
\begin{align*}
    \nm{ \Jcal_{\wh P}'(\wh w- w_{\wh P}) }_2\lesssim \eps_n^\Wcal+\|\wh\zeta-\zeta^\star\|_\infty+\epsWReal+\epsWComp + \sqrt{\log(1/\delta)/n}.
\end{align*}

\paragraph{Concluding the proof.}
The final step is to translate the above guarantee to $\nm*{ \Jcal_{P^\star}'(\wh w- w_{P^\star}) }_2$.
The following shows that the switching cost is $\Ocal(\|\wh\zeta-\zeta^\star\|_\infty)$ as before:
\begin{align*}
    &\nm*{ \Jcal_{P^\star}'(\wh w -w_{P^\star}) }_2
    \\&\leq\nm*{ (\Jcal_{P^\star}-\Jcal_{\wh P})' (\wh w- w_{P^\star}) }_2 + \nm*{ \Jcal_{\wh P}' (\wh w- w_{\wh P}) }_2 + \nm*{ \Jcal_{\wh P}' (w_{\wh P}- w_{P^\star}) }_2
    \\&\lesssim \eps_n^\Wcal+\|\wh\zeta-\zeta^\star\|_\infty+\epsWReal+\epsWComp+\sqrt{\log(1/\delta)/n}.
\end{align*}
This concludes the proof.
\end{proof}

\begin{lemma}[Visitation performance-difference]\label{lem:performance-diff-visitation}
Let $P,U:\Scal\to\RR_+$ be non-negative measures, which should be thought of as transitions in a discounted Markov chain.
Assume $U$ satisfies $\sum_{s'} U(s'\mid s)\leq 1$.
Define $d_U = (1-\gamma)\sum_{h=1}^\infty \gamma^{h-1}d_U^h$, where $d_U^h = \int_{s_1,s_2,\dots,s_{h-1}}d_1(s_1)U(s_2\mid s_1)\dots U(s\mid s_{h-1})\diff s_{1:h-1}$.
Assume the same for $P$.

Let $\Fcal\subset\Scal\to\RR$ be a function class that satisfies $f\in\Fcal\implies g(s)=\EE_{s'\sim P(s)}[f(s')]\in\Fcal$, \ie, closed under projection with $P$.
Then, define the integral (probability) metric $\|P-U\|_\Fcal:=\sup_{f\in\Fcal}\abs{ (\EE_P-\EE_U)[f(s)] }$.
Then we have,
\begin{align}
    \nm{d_P-d_U}_\Fcal\leq \frac{\gamma}{1-\gamma}\EE_{d_U}\nm{P(\cdot\mid s)-U(\cdot\mid s)}_{\Fcal}. \label{eq:tv-performance-diff}
\end{align}
\end{lemma}
\begin{proof}

Recall Bellman's flow, which is $d_P(s) = (1-\gamma)d_1(s) + \gamma\EE_{\wt s\sim d_P}P(s\mid\wt s)$.
Fix any $f\in\Fcal$. The initial state distributions cancel, so we have,
\begin{align*}
    &\abs{ \prns{\EE_{d_P}-\EE_{d_U}}[f(s)] }
    \\&=\abs{ \gamma\EE_{\wt s\sim d_P}\EE_{s\sim P(\cdot\mid \wt s)}[f(s)]- \gamma\EE_{\wt s\sim d_U}\EE_{s\sim U(\cdot\mid \wt s)}[f(s)] }
    \\&\leq\abs{ \gamma\EE_{\wt s\sim d_P}\EE_{s\sim P(\cdot\mid \wt s)}[f(s)]- \gamma\EE_{\wt s\sim d_U}\EE_{s\sim P(\cdot\mid \wt s)}[f(s)] }
    \\&+\abs{ \gamma\EE_{\wt s\sim d_U}\EE_{s\sim P(\cdot\mid \wt s)}[f(s)]- \gamma\EE_{\wt s\sim d_U}\EE_{s\sim U(\cdot\mid \wt s)}[f(s)] }
    \\&\leq\gamma \abs{ \prns{ \EE_{\wt s\sim d_P}-\EE_{\wt s\sim d_U} }[ \EE_{s\sim P(\cdot\mid \wt s)}f(s) ] }
    +\gamma\EE_{\wt s\sim d_U}\abs{ \prns{\EE_{s\sim P(\cdot\mid \wt s)}-\EE_{s\sim U(\cdot\mid \wt s)}}[f(s)] }.
\end{align*}
Thus, taking supremum over $\Fcal$, we have
\begin{align*}
    &\|d_P-d_U\|_{\Fcal}
    \\&\leq \gamma\sup_{f\in\Fcal}\abs{(\EE_{\wt s\sim d_P}-\EE_{\wt s\sim d_U})[\EE_{s\sim P(\wt s)}f(s)]}+\gamma\EE_{\wt s\sim d_U}\sup_{f\in\Fcal}\abs{\prns{\EE_{s\sim P(\cdot\mid \wt s)}-\EE_{s\sim U(\cdot\mid \wt s)}}[f(s)] }
    \\&= \gamma\|d_P-d_U\|_{\Fcal}+\gamma\EE_{\wt s\sim d_U}\|P(\cdot\mid \wt s)-U(\cdot\mid \wt s)\|_{\Fcal}.  \tag{$\Fcal$ closed under $P$-projection}
\end{align*}
Rearranging terms finishes the proof.
\end{proof}

If $\Fcal$ is the class of functions with $\|f\|_\infty\leq 1$, then this recovers the TV distance, which gives,
\begin{equation}
    \|d_P-d_U\|_{\TV} \leq \frac{\gamma}{1-\gamma}\EE_{d_U}\|P(\cdot\mid s)-U(\cdot\mid s)\|_{\TV}. \label{eq:pdl-for-visitations}
\end{equation}
This generalizes Lemma E.3 of \citet{agarwal2023provable} to infinite horizon.

\section{Proofs and Additional Details for the Orthogonal Estimator}\label{app:efficiency-proof}

\subsection{Intuition for \pref{thm:efficiency}}\label{app:intuition-for-efficiency}
We provide some intuition for the results in \pref{thm:efficiency}. Consider the $V^-$ bound and let us decouple the indicator $\I{v(s')-\beta(s,a)\leq 0}$ that appears implicitly in the $(v^-(s')-\beta^-(s,a))_-$ notation of \pref{thm:eif-main}. We augment the set of nuisances with $\zeta(s,a,s')=v^-(s')-\beta^-(s,a)$ such that $(v^-(s')-\beta^-(s,a))_- = (v^-(s')-\beta^-(s,a))\I{\zeta(s,a,s')\leq 0}$. We state the following lemma (which we elaborate upon in \cref{lem:sharpness-correct-zeta-w,lem:sharpness-correct-q-beta} in the Appendix):
\begin{lemma}[Double sharpness with correct $\zeta^\star$] Let $\EE\bracks{ \psi(s,a,s';q,w,\beta,\zeta^\star) }$ be the expectation of the (R)EIF with an arbitrary nuisance set $\eta=(w,q,\beta)$, but where the indicator $\II[v^-(s')\leq \beta^-(s,a)]$ has been replaced with the correct indicator $\II[\zeta^\star(s,a,s')\leq 0]$. Then:
\begin{align*}
    V^-_{d_1} & = \EE\bracks{ \psi(s,a,s';q,w^\star,\beta^\star,\zeta^\star) }
     = \EE\bracks{ \psi(s,a,s';q^\star,w,\beta^\star,\zeta^\star) }
\end{align*}
\end{lemma}
This lemma implies that if $\beta^-=(\beta^*)^-$ and $\zeta = \zeta^*$, then the estimator $\widehat{V}^-_{d_1}$ has a property known as ``double-robustness" \cite{kennedy2020towards} or ``double-sharpness" \cite{dorn2021doubly} in $q$ and $w$, meaning the bias vanishes when either $q$ or $w$ is consistent. Moreover, the convergence rate would be $O_p(r_{n,2}^wr_{n,2}^q)$. This condition holds provided that $\beta$ and $\zeta$ are correctly specified. However, estimation errors in $\beta$ introduce an additional $O_p\left((r_{n,\infty}^{\beta})^2\right)$ term, reflecting that $\beta$ is first-order optimal for the $\cvar$ component. Additionally, discrepancies between $\zeta$ and $\zeta^*$ contribute an extra $O_p\left((r_{n,\infty}^{q})^2\right)$ to the error. While this discussion gives some insight into how we achieve the results in \pref{thm:efficiency}, we provide a a rigorous analysis in the next section.

\subsection{Preliminaries}

For this proof, our focus will be on $\widehat{V}^-_{d_1}$. The argument for $\widehat{V}^+_{d_1}$ is analogous, following a symmetric approach. To improve the clarity of our exposition, we will omit the $-$ and $\tau$ indices, assuming their presence is clear from the context.

For simplicity, we assume that $n$ is a multiple of $K$ such that $n=Kn_K$, where $n_K$ is the size of a fold. We let $\EE_n, \EE_{k}$ denote the empirical averages over the entire sample and the $k^\text{th}$ fold, respectively. Recall that we use $\wh\eta=(\wh w, \wh q, \wh\beta)$ and $\eta^*=(w^*, q^*, \beta^*)$ to denote the estimated and oracle nuisances, respectively.

We further suppress the dependency on $s,a$ in $\Lambda$ and $\tau$ and we write the $\rho$ term in \pref{thm:eif-main} as
\begin{align}
  \rho(s,a,s';v,\beta)
  &= (1-\lambda)v(s') + \lambda\prns{ \beta(s,a)+\tau^{-1}(v(s')-\beta(s,a))_- }. \label{eq:rho-risk-def}
\end{align}
We justify this by noting that the analysis holds regardless of whether $\lambda$ and $\tau$ depend on $s,a$. Sometimes, it will be useful to decouple the indicator $\I{v(s')-\beta(s,a)\leq 0}$ implicit in the definition of $\rho$. In this case, we augment the set of nuisances with $\zeta(s,a,s')=v(s')-\beta(s,a)$ and write $\rho$ as
\begin{align}
  \rho(s,a,s';v,\beta,\zeta)=(1-\lambda)v(s') + \lambda\prns{ \beta(s,a)+\tau^{-1}(v(s')-\beta(s,a))\I{\zeta(s,a,s')\leq 0} }. \label{eq:rho-risk-with-zeta}
\end{align}
Similarly define $\psi(\cdot;w,q,\beta,\zeta)$ with the $\rho(\cdot;v,\beta,\zeta)$.

\subsection{Auxiliary Lemmas}

\begin{definition}[Margin Condition]\label{def:margin}
    A function $f:\Xcal \rightarrow \RR$ of some random variable $X$ is said to satisfy the margin condition with sharpness $\alpha\in[0,\infty]$ (or more succinctly, an $\alpha$-margin) if there exist a fixed constant $c>0$ such that
    \begin{align*}
        \forall t>0: P(0< |f(X)|\leq t) \leq c t^\alpha .
    \end{align*}
\end{definition}
If $f(X)$ is either zero or bounded away from zero almost surely, then $f$ satisfies an infinite margin, \ie, $\alpha=\infty$ \citep[Lemma 2]{kallus2022s}.
If $f(X)$ is continuously distributed in a neighborhood around $0$, \ie, its CDF is boundedly differentiable on $(-\eps,0)\cup(0,\eps)$ for some $\eps>0$, then $f$ has a $1$-margin \citep[Lemma 3]{kallus2022s}.

\begin{lemma}[Margin Guarantees]\label{lem:margin-guarantees} For any $f:\Xcal\rightarrow\RR$ satisfying $\alpha$-margin (\pref{def:margin}), $p\in[1,\infty]$, and any $g:\Xcal\rightarrow\RR$, the following statements hold for some constant $C>0$:
\begin{align}
        \EE[(\II[g(X)\leq 0] - \II[f(X)\leq 0]) f(X)] & \leq C\|f-g\|_p^{\frac{p(1+\alpha)}{p+\alpha}},\label{eq: margin g1}\\
        P[\II[g(X)\leq 0] \neq \II[f(X)\leq 0], f(X)\neq 0] & \leq C\|f-g\|_p^{\frac{p\alpha}{p+\alpha}},\label{eq: margin g2}
    \end{align}
    where $\nm{\cdot}_p$ is the $L^p$ norm and we set $\infty t/\infty=t$ in the exponents.
\end{lemma}
The proof of \cref{eq: margin g1} for any $p\in[1,\infty]$ and of \cref{eq: margin g2} for $p=\infty$ is given in \citet[Lemmas 5.1 and 5.2]{audibert2005fast}. The proof of \cref{eq: margin g2} for $p<\infty$ is given in \citet[Lemma 5]{kallus2022s}.

\begin{lemma}[Sharpness with correct $q^\star$ and $\beta^\star$]\label{lem:sharpness-correct-q-beta}
$\frac{1}{n}\sum_{(s,a,s')\sim\Dcal} \psi(s,a,s';w, q, \beta)$ is an unbiased estimator of $V^\star_{d_1}$ when $q = q^\star, \beta = \beta^\star$, \ie,
\begin{align*}
  (1-\gamma)\EE_{d_1}v^\star(s_1) = \EE\bracks{ \psi(s,a,s';w, q^\star, \beta\star) }.
\end{align*}
\end{lemma}
\begin{proof}
Since $q^\star$ and $\beta^\star$ are correct, the robust Bellman equation holds, and so for every $s,a$,
\begin{align*}
  \EE\bracks{ (1-\lambda)v^\star(s')+\lambda(\beta^{\star}(s,a)+\tau^{-1}(v^{\star}(s')-\beta^{\star}(s,a))_-) \mid s,a} = 0.
\end{align*}
Thus, multiplying by any $w$ does not change the fact that the debiasing term in $\psi$ has expectation zero.
Since we have $v^\star$, the first term in $\psi$ is exactly the estimand, which concludes the proof.
\end{proof}

\begin{lemma}[Sharpness with correct $w^*$ and $\zeta^*$]\label{lem:sharpness-correct-zeta-w} $\frac{1}{n}\sum_{(s,a,s')\sim\Dcal} \psi(s,a,s';w, q, \beta, \zeta)$ is an unbiased estimator of $V^\star_{d_1}$ when $w = w^\star, \zeta = \zeta^\star$, \ie,
\begin{align*}
    (1-\gamma)\EE_{d_1}v^\star(s_1) = \EE\bracks{ \psi(s,a,s';q,w^\star,\beta,\zeta^\star)}
\end{align*}
\end{lemma}
\begin{proof}
Let $P^\star$ denote the robust transition kernel and let $d^\star$ denote the robust visitation measure under $\pi$, which satisfies:
for all functions $f$,
\begin{align*}
  \EE_{d^\star}[f(s,a)] = (1-\gamma)\EE_{{d_1}}f(s,\pi) + \gamma\EE_{\wt s,\wt a\sim d^\star,s\sim P^\star(s,a)}[f(s,\pi)].
\end{align*}
Since $\zeta^\star$ is correct, for any $v,s,a$, we have
\begin{align*}
  &\EE_{s'\sim P(s,a)}\bracks{(1-\lambda)v(s')+\lambda\left(\beta(s,a)+\tau^{-1}(v(s')-\beta(s,a))\I{\zeta^\star(s,a,s')\leq 0}\right)}
  \\&=\EE_{s'\sim P(s,a)}\bracks{(1-\lambda)v(s')+\lambda\tau^{-1}v(s')\I{\zeta^\star(s,a,s')\leq 0}} \tag{$\bigstar$}
  \\&=\EE_{s'\sim P^\star(s,a)}\bracks{ v(s') }, \tag{\pref{lem:identification-U-robust}}
\end{align*}
where in $\bigstar$ we used $\EE_{s'\sim P(s,a)}\bracks{\beta(s,a)\prns{1-\tau^{-1}\I{\zeta^\star(s,a,s')\leq 0}}}=\beta(s,a)\prns{1-\tau^{-1}\tau} = 0$. That is, for all function $f$, we have
\begin{align*}
  &(1-\gamma)\EE_{d_1}v(s_1)+\EE\bracks{ w^\star(s,a)\prns{ r(s,a)+\gamma\rho(s,a,s';v,\beta,\zeta^\star)-q(s,a) } }
  \\&=(1-\gamma)\EE_{d_1}v(s_1)+\EE_{s,a\sim d^\star}\bracks{ r(s,a)+\gamma\rho(s,a,s';v,\beta,\zeta^\star)-q(s,a) }
  \\&=\EE_{s,a\sim d^\star}[r(s,a)] + (1-\gamma)\EE_{d_1}v(s_1)+\EE_{s,a\sim d^\star}\bracks{\gamma\EE_{s'\sim P^\star(s,a)}[v(s')]-q(s,a)}
  \\&= \EE_{s,a\sim d^\star}[r(s,a)] \tag{robust Bellman flow}
  \\&= (1-\gamma)\EE_{d_1}v^\star(s_1).
\end{align*}
This concludes the proof.
\end{proof}

\subsection{Proof of \ref{eq:rates}}\label{app:rates-proof}

The estimation error is given by:
\begin{align*}
    |\widehat{V}_{d_1}-V^*_{d_1}|& = \left|\frac{1}{K}\sum_{k=1}^K  \EE_{k}[\psi(s,a,s'; \widehat{\eta}^{[k]})] - V^*_{d_1}\right|\leq \frac{1}{K}\sum_{k=1}^K \left| \EE_{k}[\psi(s,a,s'; \widehat{\eta}^{[k]})] - V^*_{d_1}\right|
\end{align*}
We wish need to bound $\left|\EE_{k}[\psi(s,a,s'; \widehat{\eta}^{[k]})] - V^*_{d_1}\right|$. We have that:
\begin{align*}
    \left|\EE_{k}[\psi(s,a,s'; \widehat{\eta}^{[k]})] - V^*_{d_1}\right| & \leq \left|\EE_{k}[\psi(s,a,s'; \widehat{\eta}^{[k]})] -  \EE[\psi(s,a,s'; \widehat{\eta}^{[k]})]\right| + \left|\EE[\psi(s,a,s'; \widehat{\eta}^{[k]})]- V^*_{d_1} \right|
\end{align*}
The first term is $O_p(n^{-1/2})$ by the CLT. We are now interested in bounding the second term:
\begin{align}
  \eps(\wh\eta) := \abs{ \EE\bracks{ \psi(s,a,s';\wh\eta) } - V^*_{d_1} }.
\end{align}
where we dropped the $[k]$ indicator without loss of generality. We further decompose $\eps(\wh\eta)$ into two error terms, $\eps_A$ and $\eps_B$, as follows:
\begin{align*}
  \eps(\wh\eta)
  &= \abs{ \EE\bracks{ \psi(s,a,s';\wh q,\wh w,\wh\beta) } - \EE\bracks{\psi(s,a,s';\wh q,w^\star,\wh\beta,\zeta^\star)} } \tag{ \pref{lem:sharpness-correct-zeta-w} }
  \\& \leq \abs{ \EE\bracks{ \psi(s,a,s';\wh q,\wh w,\wh\beta) } - \EE\bracks{\psi(s,a,s';\wh q,\wh w,\wh\beta,\zeta^\star)} }\tag{$\eps^A$}
  \\& \quad + \abs{ \EE\bracks{\psi(s,a,s';\wh q,\wh w,\wh\beta,\zeta^\star)} - \EE\bracks{\psi(s,a,s';\wh q,w^\star,\wh\beta,\zeta^\star)} } \tag{$\eps^B$}.
\end{align*}

\paragraph{Bounding $\eps^A$: Error from the incorrect indicator $\zeta$.}
\begin{align*}
  \eps_A
  &=\gamma\lambda\tau^{-1}\EE\wh w(s,a)\prns{ \wh v(s')-\wh\beta(s,a) }\prns{\I{\wh v(s')-\wh\beta(s,a)\leq 0}-\I{v^\star(s')-\beta^\star(s,a)\leq 0}}
  \\&\leq C\gamma\lambda\tau^{-1} \EE\prns{ \wh v(s')-\wh\beta(s,a) }\prns{\I{\wh v(s')-\wh\beta(s,a)\leq 0}-\I{v^\star(s')-\beta^\star(s,a)\leq 0}} \tag{\pref{assum:regular}}
  \\&\lesssim \EE\prns{ \wh v(s')-\wh\beta(s,a) }\prns{\I{\wh v(s')-\wh\beta(s,a)\leq 0}-\I{v^\star(s')-\beta^\star(s,a)\leq 0}}
\end{align*}
We break these terms down as follows:
\begin{align}
\EE&\prns{ \wh v(s')-\wh\beta(s,a) }\prns{\I{\wh v(s')-\wh\beta(s,a)\leq 0}-\I{v^\star(s')-\beta^\star(s,a)\leq 0}} \nonumber
\\=&\EE\prns{  v^\star(s')-\beta^\star(s,a) }\prns{\I{\wh v(s')-\wh\beta(s,a)\leq 0}-\I{v^\star(s')-\beta^\star(s,a)\leq 0}}\label{eq:eps_A1}\tag{$\eps^A_1$}
\\&+\EE\prns{  \wh v(s') - \wh \beta(s,a) - v^\star(s') + \beta^\star(s,a)}\prns{\I{\wh v(s')-\wh\beta(s,a)\leq 0}-\I{v^\star(s')-\beta^\star(s,a)\leq 0}}\label{eq:eps_A2}\tag{$\eps^A_2$}.
\end{align}
We first bound $\eps^A_1$. 
\pref{assum:regular} implies
\begin{align*}
    P(0< |v^\star(s')-\beta^\star(s,a)|\leq t) \leq c''t, \; \forall t\in [0,c'),\quad P(|v^\star(s')-\beta^\star(s,a)|=0)=0,
\end{align*}
where $c'<1$ is the min of 1 and the given neighborhood of zero and $c''\geq1$ is the max of 1 and the bound on the density in that neighborhood. This implies a margin condition with $\alpha=1$ and $c=c''/c'$.
 
 We can instantiate the first part of \pref{lem:margin-guarantees} with $f(X)=v^\star(s')-\beta^\star(s,a)$, $g(X)=\wh v(s')-\wh\beta(s,a)$ and obtain
 \begin{align*}
     \eps^A_1 & \lesssim \magd{v^\star(s')-\beta^\star(s,a) - \wh v(s')+\wh\beta(s,a)}_p^{\frac{2p}{p+1}}\\
     & \leq \magd{\wh v(s') - v^\star(s')}_p^{\frac{2p}{p+1}} + \magd{\wh\beta(s,a) - \beta^\star(s,a)}_p^{\frac{2p}{p+1}}.
 \end{align*}
To bound $\eps^A_2$, first write
\begin{align*}
&\abs{\EE\prns{  \wh v(s') - \wh \beta(s,a) - v^\star(s') + \beta^\star(s,a)}\prns{\I{\wh v(s')-\wh\beta(s,a)\leq 0}-\I{v^\star(s')-\beta^\star(s,a)\leq 0}}}
\\
& \leq \magd{\wh v(s') - \wh \beta(s,a) - v^\star(s') + \beta^\star(s,a)}_p \\
& \quad \cdot \PP\prns{\I{\wh v(s')-\wh\beta(s,a)\leq 0}\neq\I{v^\star(s')-\beta^\star(s,a)\leq 0}}^{(p-1)/p}. \tag{Holder's inequality}
\end{align*}
We can bound $\PP\prns{\I{\wh v(s')-\wh\beta(s,a)\leq 0}\neq \I{v^\star(s')-\beta^\star(s,a)\leq 0}}$ using the second part of \pref{lem:margin-guarantees} such that
\begin{align*}
\eps^A_2 &\lesssim \magd{\wh v(s') - \wh \beta(s,a) - v^\star(s') + \beta^\star(s,a)}_p \magd{\wh v(s') - \wh \beta(s,a) - v^\star(s') + \beta^\star(s,a)}^{\frac{p-1}{p+1}}\\
& = \magd{\wh v(s') - \wh \beta(s,a) - v^\star(s') + \beta^\star(s,a)}_p^{\frac{2p}{p+1}}\\
& \leq \magd{\wh v(s') - v^\star(s')}_p^{\frac{2p}{p+1}} + \magd{\wh\beta(s,a) - \beta^\star(s,a)}_p^{\frac{2p}{p+1}}.
\end{align*}
Putting the $\eps^A_1$ and $\eps^A_2$ together, we have
\begin{align*}
    \eps_A & \lesssim  \magd{\wh v(s') - v^\star(s')}_p^{\frac{2p}{p+1}} + \magd{\wh\beta(s,a) - \beta^\star(s,a)}_p^{\frac{2p}{p+1}} \tag{when $p\in[1,\infty)$}\\
    & \lesssim \magd{\wh v(s') - v^\star(s')}_\infty^{2} + \magd{\wh\beta(s,a) - \beta^\star(s,a)}_\infty^{2}. \tag{when $p=\infty$}
\end{align*}

\paragraph{Bounding $\eps^B$: Error with correct indicator but wrong nuisances.}
Now we focus on bounding $\eps^B$.
\begin{align*}
  &\eps_B = \EE\bracks{\psi(s,a,s';\wh q,\wh w,\wh\beta,\zeta^\star)} - \EE\bracks{\psi(s,a,s';\wh q,w^\star,\wh\beta,\zeta^\star)}
  \\&= \EE\prns{\wh w(s,a)-w^\star(s,a)}\prns{ r(s,a)+\gamma\rho(s,a,s';\wh v,\wh\beta,\zeta^\star)-\wh q(s,a) }
  \\&= \EE\prns{\wh w(s,a)-w^\star(s,a)}\prns{ r(s,a)+\gamma\rho(s,a,s';\wh v,\wh\beta,\zeta^\star)-\wh q(s,a) }
  \\&- \EE\prns{\wh w(s,a)-w^\star(s,a)}\prns{ r(s,a)+\gamma\rho(s,a,s';v^\star,\beta^\star)- q^\star(s,a) } \tag{\pref{lem:sharpness-correct-q-beta}}
  \\&= \EE\prns{\wh w(s,a)-w^\star(s,a)}\prns{ \wh q(s,a)-q^\star(s,a) + \gamma(\rho(s,a,s';\wh v,\wh\beta,\zeta^\star)-\rho(s,a,s';v^\star,\beta^\star)) }.
\end{align*}
In the \pref{lem:sharpness-correct-q-beta} step, we used
\begin{align*}
  0 = (1-\gamma)\EE_{d_1}v^\star(s_1) - \EE\bracks{\psi(s,a,s';q^\star,\wh w,\beta^\star)} = (1-\gamma)\EE_{d_1}v^\star(s_1) - \EE\bracks{\psi(s,a,s';q^\star,w^\star,\beta^\star)}.
\end{align*}
Finally, note that
\begin{align*}
  &\rho(s,a,s';\wh v,\wh\beta,\zeta^\star)-\rho(s,a,s';v^\star,\beta^\star)
  \\&=(1-\lambda)(\wh v(s')-v^\star(s')) + \lambda\tau^{-1}(\wh v(s')-v^\star(s'))\I{\zeta^\star(s,a,s')\leq 0} \\
  & \quad + \lambda(\wh\beta(s,a)-\beta^\star(s,a))\prns{ 1-\tau^{-1}\I{\zeta^\star(s,a,s')\leq 0} }.
\end{align*}
Due to continuity of the CDF of $v^\star(s')$ at $\beta^\star(s,a)$ for all $s,a$, we have $\Pr(\zeta^\star(s',s,a)\leq 0\mid s,a)=\tau$ and so the last term vanishes.
Thus, we're left with a quantity that is at most $\lesssim (\wh v(s')-v^\star(s'))$.
Therefore,
\begin{align*}
  \eps_B &\lesssim
  \EE\prns{\wh w(s,a)-w^\star(s,a)}\prns{ \Jcal_{U^\pm}\prns{\wh q(s,a)-q^\star(s,a)} } 
  \\& \leq \|\Jcal_{U^\pm}'(\wh w-w^\star)\|_2 \|\wh q-q^\star\|_2. \tag{Holder's inequality}
\end{align*}
Putting everything together, we obtain the desired rates:
\begin{align*}
    |\wh{V}_{d_1}-V^*_{d_1}|& \lesssim O_p(n^{-1/2}) +  \|\Jcal_{U^\pm}'(\wh w-w^\star)\|_2 \|\wh q-q^\star\|_2 +  \magd{\wh v - v^\star}_p^{\frac{2p}{p+1}} + \magd{\wh\beta - \beta^\star}_p^{\frac{2p}{p+1}} \\
    & =  O_p(n^{-1/2}) + O_p\prns{r_n^wr_n^q+(r_{n, p}^{q})^\frac{2p}{p+1}+(r_{n, p}^{\beta})^\frac{2p}{p+1}} \tag{when $p\in[1,\infty)$}\\
    & \lesssim O_p(n^{-1/2}) + \|\Jcal_{U^\pm}'(\wh w-w^\star)\|_2 \|\wh q-q^\star\|_2 + \magd{\wh v - v^\star}_\infty^{2} + \magd{\wh\beta - \beta^\star}_\infty^{2} \\
    & =  O_p(n^{-1/2}) + O_p\prns{r_n^wr_n^q+(r_{n, \infty}^{q})^2+(r_{n, \infty}^{\beta})^2}.\tag{when $p=\infty$}
\end{align*}

\subsection{Proof of \ref{eq:efficiency}}

In this part of the theorem, we let:
\begin{align*}
    \widetilde{V}_{d_1} = \frac{1}{K}\sum_{k=1}^K \EE_{k}[\psi(s,a,s'; \eta^*)]
\end{align*}
Then, we can write the following equality:
\begin{align*}
    \sqrt{n}(\widehat{V}_{d_1}-V^*_{d_1}) = \sqrt{n}(\widehat{V}_{d_1}-\widetilde{V}_{d_1}) + \underbrace{\sqrt{n}(\widetilde{V}_{d_1}-V^*_{d_1})}_{\xrightarrow{d} \Ncal(0, \Sigma)}
\end{align*}
The second term converges in distribution to $\Ncal(0, \Sigma)$ from the CLT and the fact that $\psi$ is the efficient influence function. Thus, it remains to show that the first term is $o_p(1)$. We decompose the first term as follows:
\begin{align}
    \sqrt{n}(\widehat{V}_{d_1}-\widetilde{V}_{d_1}) & = \sqrt{n}\frac{1}{K}\sum_{k=1}^n \left(\EE[\psi(s,a,s'; \widehat{\eta}^{[k]})] - \EE[\psi(s,a,s'; \eta^*)]\right) \label{eq:term1}\\
    & \quad + \sqrt{n}\frac{1}{K}\sum_{k=1}^n \underbrace{(\EE_{k}-\EE)[\psi(s,a,s'; \widehat{\eta}^{[k]}) - \psi(s,a,s'; \eta^*)]}_{\varepsilon_k} \label{eq:term2}
\end{align}
In \pref{eq:term1}, we have that $|\EE[\psi(s,a,s'; \widehat{\eta}^{[k]})] - \EE[\psi(s,a,s'; \eta^*)]|$ is bounded as in \pref{eq:rates}. Given the theorem's assumption about the nuisance rates, this term is $o_p(n^{-1/2})$ and \pref{eq:term1} is $o_p(1)$. We now seek to control the $\varepsilon_k$ term in \pref{eq:term2}. Letting $\Dcal_k$ represent the samples in the $k^\text{th}$ fold, we leverage sample splitting to show that the mean of $\varepsilon_k\mid \Dcal_k$ is $0$:
\begin{align*}
    \EE[\varepsilon_k\mid \Dcal_k] & = \EE[\EE_{k}[\psi(s,a,s'; \widehat{\eta}^{[k]}) - \psi(s,a,s'; \eta^*)] - \EE[\psi(s,a,s'; \widehat{\eta}^{[k]}) - \psi(s,a,s'; \eta^*)]\mid \Dcal_k]\\
    & = 0
\end{align*}
where we consider $\widehat{\eta}^{[k]}$ fixed with respect to the second expectation. The result follows from the fact that $\widehat{\eta}^{[k]}$ does not depend on $\Dcal_k$.
Then, we can invoke Chebyshev’s inequality to obtain the following bound:
\begin{align*}
    P\left(\frac{\varepsilon_k}{\Var[\varepsilon_k\mid \Dcal_k]^{1/2}}\geq \epsilon \bigg\lvert \Dcal_k\right)\leq \frac{1}{\epsilon^2}, \; \forall \epsilon > 0
\end{align*}
Thus, we have shown that $\varepsilon_k\mid \Dcal_k = O_p(\Var[\varepsilon_k\mid \Dcal_k]^{1/2})=O_p(n^{-1/2}\EE[(\psi(s,a,s'; \widehat{\eta}^{[k]}) - \psi(s,a,s'; \eta^*))^2\mid \Dcal_k]^{1/2})$. Here, we used the fact that $n_K=n/K$ (the size of $\Dcal_k$) and that $K$ is a fixed integer that doesn't grow with $n$. Moreover, $\varepsilon_k$ has $0$ conditional mean. 

For the remainder of the analysis, we leave the conditioning on $\Dcal_k$ implicit for simplicity. To bound $\EE[(\psi(s,a,s'; \widehat{\eta}^{[k]}) - \psi(s,a,s'; \eta^*))^2\mid \Dcal_k]^{1/2}=\|\psi(s,a,s'; \widehat{\eta}^{[k]}) - \psi(s,a,s'; \eta^*)\|_2$, we use similar notation and techniques as in \pref{app:rates-proof}:
\begin{align*}
    \|\psi(s,a,s'; \widehat{\eta}^{[k]}) - \psi(s,a,s'; \eta^*)\|_2 &
    \leq \|\psi(s,a,s'; \widehat{q}, \widehat{w}, \widehat{\beta}) - \psi(s,a,s'; \widehat{q}, \widehat{w}, \widehat{\beta}, \zeta^*)\|_2 \tag{$\sigma_1$}\\
    & \quad + \|\psi(s,a,s'; \widehat{q}, \widehat{w}, \widehat{\beta}, \zeta^*) - \psi(s,a,s'; q^*, w^*, \beta^*, \zeta^*)\|_2 \tag{$\sigma_2$}
\end{align*}
where we invoked Cauchy-Schwarz for the $L_2$ norm. We bound $\sigma_2$ as follows:
\begin{align*}
    \sigma_2 &\leq \|\psi(s,a,s'; \widehat{q}, \widehat{w}, \widehat{\beta}) - \psi(s,a,s'; q^*, \widehat{w}, \widehat{\beta}, \zeta^*)\|_2 \tag{$\sigma_{2a}$}\\
    & \quad + \|\psi(s,a,s'; q^*, \widehat{w}, \widehat{\beta}, \zeta^*) - \psi(s,a,s'; q^*, \widehat{w}, \beta^*, \zeta^*)\|_2 \tag{$\sigma_{2b}$}\\
    & \quad +  \|\psi(s,a,s'; q^*, \widehat{w}, \beta^*, \zeta^*) - \psi(s,a,s'; q^*, w^*, \beta^*, \zeta^*)\|_2\tag{$\sigma_{2c}$}\\
    & \leq \|\widehat{v}-v^*\|_2 + \gamma(1-\lambda)\|\widehat{w}\|_2\|\widehat{v}-v^*\|_2+\gamma\lambda\tau^{-1}\|\widehat{w}\|_2\|\widehat{v}-v^*\|_2 + \|\widehat{w}\|_2\|\widehat{q}-q^*\|_2\tag{$\sigma_{2a}$}\\
    & \quad + \gamma \lambda \|\widehat{w}\|_2\|\widehat{\beta}-\beta^*\|_2 + \gamma \lambda \tau^{-1}\|\widehat{w}\|_2\|\widehat{\beta}-\beta^*\|_2 \tag{$\sigma_{2b}$}\\
    & \quad + \|\widehat{w}-w^*\|_2 \left(\|r\|_2 + \gamma(1-\lambda)\|v^*\|_2 + \gamma\lambda \|\beta^*\|_2 + \gamma\lambda\tau^{-1}\|v^*-\beta^*\|_2\right) \tag{$\sigma_{2c}$}
\end{align*}
Given our rate assumptions, our boundedness assumptions for $\widehat{w}$, the implicit boundedness of $q^*,v^*, w^*, \beta^*$, as well as the ordering of the $L_2$ and $L_\infty$ norms, $\sigma_2$ is $o_p(1)$. We now bound the $\sigma_1$ term:
\begin{align*}
    \sigma_2 = \gamma\lambda\tau^{-1}\left\|\wh{w}(s,a)(\wh{v}(s')-\wh{\beta}(s,a))(\II[\wh{v}(s') \leq \wh{\beta}(s,a)] - \II[v^*(s')\leq\beta^*(s,a)])\right\|_2
\end{align*}
There are two cases in which the difference of indicators is non-zero:
\begin{align*}
\begin{cases}
    \widehat{v}(s')\leq\widehat{\beta}(s,a) \text{ and } v^*(s') > \beta^*(s,a)  \Rightarrow \II[\wh{v}(s') \leq \wh{\beta}(s,a)] - \II[v^*(s')\leq\beta^*(s,a)] = 1\\
    \widehat{v}(s')>\widehat{\beta}(s,a) \text{ and } v^*(s') \leq \beta^*(s,a) \Rightarrow \II[\wh{v}(s') \leq \wh{\beta}(s,a)] - \II[v^*(s')\leq\beta^*(s,a)] = -1
\end{cases}
\end{align*}
In the first case, $\widehat{v}(s')-\widehat{\beta}(s,a)\leq 0, \beta^*(s,a)- v^*(s')< 0$ and thus
\[
|(\widehat{v}(s')-\widehat{\beta}(s,a))(\II[\widehat{v}(s')\leq \widehat{\beta}(s,a)] - \II[v^*(s')\leq \beta^*(s,a)])|\leq |\widehat{v}(s')-\widehat{\beta}(s,a)+\beta^*(s,a)- v^*(s')|.\]
In the second case, $\widehat{v}(s')-\widehat{\beta}(s,a) > 0, \beta^*(s,a)- v^*(s')\leq 0$ and \[
|(\widehat{v}(s')-\widehat{\beta}(s,a))(\II[\widehat{v}(s')\leq \widehat{\beta}(s,a)] - \II[v^*(s')\leq \beta^*(s,a)])|\leq |\widehat{v}(s')-\widehat{\beta}(s,a)+\beta^*(s,a)- v^*(s')|.\] Going back to $\sigma_1$, we have:
\begin{align*}
    \sigma_2 & \leq \gamma\lambda\tau^{-1}\|\widehat{w}\|_2
    \|\widehat{v}(s')-\widehat{\beta}(s,a)+\beta^*(s,a)- v^*(s'))\|_2\\
    & \leq \gamma\lambda\tau^{-1}\|\widehat{w}\|_2 (\|\widehat{v}-v^*\|_2 + \|\widehat{\beta}-\beta^*\|_2)
\end{align*}
By out theorem's assumptions, this term is also $o_p(1)$. Putting $\sigma_1$ and $\sigma_2$ together, we have that $\|\psi(s,a,s'; \widehat{\eta}^{[k]}) - \psi(s,a,s'; \eta^*)\|_2$ is $o_p(1)$ and $\varepsilon_k\mid \Dcal_k$ is $o_p(n^{-1/2})$. By the bounded convergence theorem, this implies that $\varepsilon_k$ is also $o_p(n^{-1/2})$. Then, the term in \ref{eq:term2} is $o_p(1)$, which further means that $\sqrt{n}(\widehat{V}_{d_1}-\widetilde{V}_{d_1})=o_p(1)$. Our proof is now complete.

\section{Derivation of the Efficient Influence Function} \label{sec:eif-proof}

We use the $\eps$-contamination approach of \citet{hines2022demystifying} to derive an influence function (IF) for our estimand $V_{d_1}^-$. The proof for $V_{d_1}^+$ follows symmetrically. We note that since our tangent space is the whole space as it factorizes in the trivial way (as in \citep[Page 54]{kallus2022efficiently}), the IF we derive is actually the efficient influence function (EIF).

Let $P(s,a,s')$ denote the data distribution.
Consider the $\eps$-contamination $P_\eps(s,a,s')=(1-\eps)P(s,a,s')+\eps\delta(\bar s,\bar a,\bar s')$, where $\delta(\bar z)$ is the dirac delta at $\bar z$, \ie, $\delta(\bar z)$ has infinite mass at $\bar z$ and $0$ mass elsewhere.
Let $V^-_\eps$ denote the robust value function under the transition kernel $P_\eps(s'\mid s,a)$. Omitting the $\eps$ subscript means $\eps=0$.
The IF of $V_{d_1}^-$ is then given by
\begin{align*}
    \frac{\diff}{\diff\eps}(1-\gamma)\EE_{d_1}V^-_\eps(s_1)|_{\eps=0}.
\end{align*}
We dedicate the rest of this section towards this goal, which will be obtained in \pref{thm:if-estimand}.

\begin{lemma}
\begin{align*}
    \frac{\diff}{\diff\eps}P_\eps(s'\mid s,a)|_{\eps=0} = \frac{\delta(\bar s,\bar a)}{P(s,a)}\prns{ \delta(\bar s') - P(s'\mid s,a) }.
\end{align*}
\end{lemma}
\begin{proof}
Use the fact $P_\eps(s'\mid s,a) = \frac{P_\eps(s,a,s')}{P_\eps(s,a)} = \frac{(1-\eps)P(s,a,s')+\eps\delta(\bar s,\bar a,\bar s')}{(1-\eps)P(s,a)+\eps\delta(\bar s,\bar a)}$ and take derivative.
\end{proof}

\begin{lemma}[IF of conditional expectation]\label{lem:product-rule}
For any $s,a$ and $f_\eps$,
\begin{align*}
    \frac{\diff}{\diff\eps}\EE_{P_\eps}[ f_\eps(s')\mid s,a ]|_{\eps=0} = \frac{\delta(\bar s,\bar a)}{P(s,a)}\prns{ f(\bar s') - \EE_P[ f(s')\mid s,a ] } + \EE_{P}\left[\frac{\diff}{\diff\eps}f_\eps(s')|_{\eps=0}\mid s,a\right],
\end{align*}
where $f=f_0$.
\end{lemma}
\begin{proof}
\begin{align*}
\frac{\diff}{\diff\eps}\EE_{P_\eps}[ f_\eps(s')\mid s,a ]|_{\eps=0}
&= \sum_{s'} f(s') \frac{\diff}{\diff\eps}P_\eps(s'\mid s,a)|_{\eps=0} + \sum_{s'} \frac{\diff}{\diff\eps}f_\eps(s')|_{\eps=0} P(s'\mid s,a)
\\&= \frac{\delta(\bar s,\bar a)}{P(s,a)}\prns{ f_0(\bar s') - \EE_P[ f_0(s')\mid s,a ] } + \EE_{P}\left[\frac{\diff}{\diff\eps}f_\eps(s')|_{\eps=0}\mid s,a\right],
\end{align*}
\end{proof}

\begin{lemma}[IF of conditional CVaR]\label{lem:cvar-if}
For any $\tau,s,a$ and $f_\eps$,
\begin{align*}
    \frac{\diff}{\diff\eps}\cvar_{\tau,P_\eps}[ f_\eps(s')\mid s,a ]|_{\eps=0}
    &= \frac{\delta(\bar s,\bar a)}{P(s,a)}\prns{ \beta_\tau(s,a)+\tau^{-1}(f(\bar s')-\beta_\tau(s,a))_- - \cvar_\tau(f(s')\mid s,a) }
    \\&+ \EE_P\bracks{ \tau^{-1}\I{f(s')\leq\beta_\tau(s,a)} \frac{\diff}{\diff\eps}f_\eps(s')|_{\eps=0}\mid s,a } ,
\end{align*}
where $f=f_0$ and $\beta_\tau(s,a)$ be the $(1-\tau)$-th quantile of $f(s'),s'\sim P(s,a)$.
\end{lemma}
\begin{proof}
\begin{align}
    &\frac{\diff}{\diff\eps}\cvar_{P_\eps}[ f_\eps(s')\mid s,a ]|_{\eps=0}
    \\&=\frac{\diff}{\diff\eps}\min_b \EE_{P_\eps}\bracks{b+\tau^{-1}(f_\eps(s')-b)_-\mid s,a } |_{\eps=0}
    \\&=\frac{\diff}{\diff\eps} \EE_{P_\eps}\bracks{\beta_\tau(s,a)+\tau^{-1}(f_\eps(s')-\beta_\tau(s,a))_-\mid s,a } |_{\eps=0},
\end{align}
where the last equality is due to Danskin's theorem and the fact that $\beta_\tau(s,a)$ is the maximizer of the CVaR dual form at $\eps=0$.
Continuing, let $g_\eps(s';s,a) := \beta_\tau(s,a)+\tau^{-1}(f_\eps(s')-\beta_\tau(s,a))_-$, so
\begin{align*}
    &\frac{\diff}{\diff\eps}\EE_{P_\eps}\bracks{g_\eps(s';s,a)\mid s,a}
    \\&=\frac{\delta(\bar s,\bar a)}{P(s,a)}\prns{ g(\bar s';s,a) - \EE_P\bracks{g(s',s,a)\mid s,a} } + \EE_P\bracks{ \frac{\diff}{\diff\eps}g_\eps(s';s,a)|_{\eps=0}\mid s,a } \tag{\pref{lem:product-rule}}
    \\&=\frac{\delta(\bar s,\bar a)}{P(s,a)}\prns{ g(\bar s';s,a) - \cvar_\tau(f(s')\mid s,a)} + \EE_P\bracks{ \tau^{-1}\I{ f(s')\leq\beta_\tau(s,a) } \frac{\diff}{\diff\eps}f_\eps(s')|_{\eps=0}\mid s,a }.
\end{align*}
This concludes the proof.
\end{proof}

We now prove the key ``one-step forward'' lemma.
\begin{lemma}[One-Step Forward]\label{lem:one-step-forward}
For any state distribution $\nu(s)$, we have
\begin{align*}
    & \EE_{s\sim\nu}\bracks{ \frac{\diff}{\diff\eps}V^-_\eps(s)|_{\eps=0} }\\
    &= \frac{\nu(\bar s)\pi(\bar a\mid\bar s) }{ P(\bar s,\bar a) }\big( r(\bar s,\bar a) + \gamma\prns{ (1-\lambda)V^-(\bar s')+ \lambda \prns{\beta_\tau(\bar s,\bar a)+\tau^{-1}(V^-(\bar s')-\beta_\tau(\bar s,\bar a))_-} }\\
    & \qquad\qquad\qquad\quad -Q^-(\bar s,\bar a) \big)
    \\&\quad +\gamma \EE_{s\sim\nu}\bracks{ \EE_{\pi,P}\bracks{ \prns{(1-\lambda)+\lambda\tau^{-1}\I{ V^-(s')\leq \beta_\tau(s,a) }} \frac{\diff}{\diff\eps}V^-_{\eps}(s')|_{\eps=0} \mid s} }.
\end{align*}
\end{lemma}
\begin{proof}
For any $s_1$, we have
\begin{align*}
    &\frac{\diff}{\diff\eps} V^-_{\eps}(s_1)
    \\&=\frac{\diff}{\diff\eps}\EE_{a_1\sim\pi(s_1)}\bracks{ r(s_1,a_1) + \gamma((1-\lambda)\EE_{P_\eps}\bracks{V^-_\eps(s_2)\mid s_1,a_1}+\lambda\cvar_{\tau,P_\eps}\bracks{ V^-_{\eps}(s_2)\mid s_1,a_1 }}_{\eps=0}
    \\&=\gamma\EE_{a_1\sim\pi(s_1)}\bracks{ (1-\lambda)\frac{\diff}{\diff\eps}\EE_{\tau,P_\eps}\bracks{ V^-_\eps(s_2)\mid s_1,a_1 }|_{\eps=0} + \frac{\diff}{\diff\eps}\cvar_{\tau,P_\eps}\bracks{ V^-_{\eps}(s_2)\mid s_1,a_1 }|_{\eps=0}}
    \\&=\gamma(1-\lambda)\EE_{a_1\sim\pi(s_1)}\bracks{\frac{\delta(\bar s,\bar a)}{P(s_1,a_1)}\prns{V^-(\bar s')-\EE_P\bracks{V^-(s_2)\mid s_1,a_1} } }
    \\& +\gamma(1-\lambda)\EE_{a_1\sim\pi(s_1)}\EE_{P}\bracks{ \frac{\diff}{\diff\eps}V^-_{\eps}(s_2)|_{\eps=0}\mid s_1,a_1 }
    \\&+\gamma\lambda\EE_{a_1\sim\pi(s_1)}\bracks{\frac{\delta(\bar s,\bar a)}{P(s_1,a_1)}\prns{ \beta_\tau(s_1,a_1)+\tau^{-1}(V^-(\bar s')-\beta_\tau(s_1,a_1))_--\cvar_\tau(V^-(s_2)\mid s_1,a_1) }  }
    \\&+\gamma\lambda\EE_{a_1\sim\pi(s_1)}\EE_{P}\bracks{ \tau^{-1}\I{ V^-(s_2)\leq \beta_\tau(s_1,a_1) }\frac{\diff}{\diff\eps}V^-_{\pi,P_\eps}(s_2) }.
\end{align*}
Taking expectation over $s_1\sim\nu$, we have
\begin{align*}
    \EE_{s\sim\nu}\bracks{ \frac{\diff}{\diff\eps}V^-_\eps(s)|_{\eps=0} }
    &=\gamma\frac{\nu(\bar s)\pi(\bar a\mid\bar s) }{ P(\bar s,\bar a) }\bigg( (1-\lambda)V^-(\bar s') + \lambda\prns{ \beta_\tau(\bar s,\bar a)+\tau^{-1}(V^-(\bar s')-\beta_\tau(\bar s,\bar a))_- }
    \\&- \prns{(1-\lambda)\EE\bracks{V^-(s')\mid\bar s,\bar a}+\lambda\cvar_\tau(V^-(s')\mid\bar s,\bar a)}\bigg)
    \\&+\gamma\EE_{s\sim\nu}\bracks{ \EE_{\pi,P}\bracks{ \prns{(1-\lambda)+\lambda\tau^{-1}\I{ V^-(s')\leq \beta_\tau(s,a) } } \frac{\diff}{\diff\eps}V^-_\eps(s')|_{\eps=0} \mid s} }.
\end{align*}
Finally recall that $V^-$ satisfies the Bellman equation, so
\begin{align*}
    (1-\lambda)\EE\bracks{V^-(s')\mid\bar s,\bar a}+\lambda\cvar_\tau(V^-(s')\mid\bar s,\bar a) = Q^-(\bar s,\bar a)-r(\bar s,\bar a).
\end{align*}
This concludes the proof.
\end{proof}

Equipped with our main one-step lemma, we can now unroll it an infinite number of steps to derive the IF of our estimand.
\begin{theorem}[IF of Estimand]\label{thm:if-estimand}
Let us denote
\begin{align*}
    g(\bar s,\bar a,\bar s'):=r(\bar s,\bar a)+\gamma\prns{ (1-\lambda)V^-(\bar s')+\lambda\prns{\beta_\tau(\bar s,\bar a)+\tau^{-1}(V^-(\bar s')-\beta_\tau(\bar s,\bar a))_- } }.
\end{align*}
Then, we have
\begin{align*}
    \EE_{d_1}\bracks{ \frac{\diff}{\diff\eps}V^-_\eps(s_1)|_{\eps=0} } = \frac{d^{\pi,\infty}_{\text{rob}}(\bar s,\bar a)}{P(\bar s,\bar a)}g(\bar s,\bar a,\bar s').
\end{align*}
\end{theorem}
\begin{proof}
Let $d_h$ denote the $h$-th step visitation in the robust MDP, with transition $P_{\text{rob}}$ satisfying $\frac{P_{\text{rob}}(s'\mid s,a)}{P(s'\mid s,a)} = (1-\lambda)+\lambda\tau^{-1}\I{ V^-(s')\leq\beta_\tau(s,a) }$. Then notice that the final term of \pref{lem:one-step-forward} is exactly $\EE_{s\sim\nu}\bracks{ \EE_{\pi,P_{\text{rob}}}\bracks{\frac{\diff}{\diff\eps}V^-_\eps(s')|_{\eps=0} \mid s}}$.
Therefore,
\begin{align*}
    &\EE_{d_1}\bracks{ \frac{\diff}{\diff\eps}V^-_\eps(s_1)|_{\eps=0} }
    \\&= \frac{d_1(\bar s)\pi(\bar a\mid \bar s)}{P(\bar s,\bar a)}g(\bar s,\bar a,\bar s') + \gamma\EE_{s_2\sim d_2}\bracks{ \frac{\diff}{\diff\eps}V^-_\eps(s_2)|_{\eps=0} }
    \\&= \frac{d_1(\bar s)\pi(\bar a\mid\bar s)}{P(\bar s,\bar a)}g(\bar s,\bar a,\bar s') + \gamma\frac{d_2(\bar s)\pi(\bar a\mid\bar s)}{P(\bar s,\bar a)}g(\bar s,\bar a,\bar s') + \gamma^2\EE_{s_3\sim d_3}\bracks{ \frac{\diff}{\diff\eps}V^-_\eps(s_3)|_{\eps=0} }.
\end{align*}
Iterating the process, we have
\begin{align*}
    \EE_{d_1}\bracks{ \frac{\diff}{\diff\eps}V^-_\eps(s_1)|_{\eps=0} } = \sum_{h=1}^\infty\gamma^{h-1}\frac{d_h(\bar s)\pi(\bar a\mid\bar s)}{P(\bar s,\bar a)}g(\bar s,\bar a,\bar s') = \frac{d^{\pi,\infty}_{\text{rob}}(\bar s,\bar a)}{P(\bar s,\bar a)}g(\bar s,\bar a,\bar s'),
\end{align*}
as desired.
\end{proof}

Finally, we can conclude that the IF in \pref{thm:if-estimand} is in fact the efficient IF (EIF) because it is in the tangent space, as the tangent space is contains all functions \citep{kallus2022efficiently}.

\section{Additional Validity Guarantees for Orthogonal Estimator}\label{app:validity-proofs}

Our orthogonal estimator has additional desirable properties such as \textit{validity} when some nuisances are misspecified. Specifically, the bounds returned by our orthogonal estimator will be asymptotically valid, though possibly loose, when some nuisances are inconsistent, \ie, do not converge to the their true values. Below, we detail conditions under which we achieve validity. 
To be concise, we focus on the $-$ case as the $+$ case is symmetric.

\paragraph{Validity with correct \texorpdfstring{$Q^\pm$}{Q±}.}
If $\wh Q = Q^\pm$, we obtain valid bounds even if $w,\beta$ are inconsistent.
\begin{restatable}{lemma}{correctQValidity}\label{lem:correct-Q-validity}
For any $w,\beta$, we have 
% $\EE[\psi(s,a,s';Q^+,\beta,w)]\geq V^+$ and
$\EE[\psi(s,a,s';Q^-,\beta,w)]\leq V_{d_1}^-$ with equality when $\beta=\beta_\tau^-$.
\end{restatable}

\paragraph{Validity with \texorpdfstring{$Q=\Tcal^\pm_{\beta}Q$}{Q=T±Q}.}
Even if $\wh Q$ is misspecified, we still have a valid bound if it solves a Bellman-type equation of the dual CVaR form. For a $\beta:\Scal\times\Acal\to\RR$, define:
\begin{align*}
    &\Tcal_\beta^\pm f(s,a) := r(s,a)+\gamma\Lambda^{-1}(s,a)\EE\bracks{f(s',\pit)\mid s,a}
    \\&\qquad\qquad\qquad+\gamma(1-\Lambda^{-1}(s,a))\EE\bracks{ \beta(s,a)+\tau^{-1}(s,a)\prns{ f(s',\pit)-\beta(s,a) }_\pm \mid s,a}.
\end{align*}
% where $\EE$ is w.r.t. the nominal kernel $P$.
\begin{restatable}{lemma}{TbetaValidity}\label{lem:Tbeta-validity}
Fix any $w,\beta$. If $Q_\beta^\pm=\Tcal_{\beta}^\pm Q_\beta^\pm$, then 
% $\EE[ \psi(s,a,s';Q_\beta^+,\beta,w) ] \geq V_{d_1}^+$ and 
$\EE[ \psi(s,a,s';Q_\beta^-,\beta,w) ]\leq V_{d_1}^-$.
\end{restatable}
\begin{remark}
\cref{lem:correct-Q-validity,lem:Tbeta-validity} are dual to each other: in \pref{lem:correct-Q-validity}, the plug-in is consistent while the debiasing correction errs in the valid direction (\ie, $\geq 0$ for $+$ and $\leq 0$ for $-$). In \pref{lem:Tbeta-validity}, the plug-in is valid while the debiasing correction has expectation zero.
\end{remark}

\subsection{Proofs for validity}
\correctQValidity*
\begin{proof}
\begin{align*}
    \EE[\psi(s,a,s';Q^-,\beta,w)] &\leq (1-\gamma)\EE_{d_1}[ V_\beta^-(s_1) ] + \EE[w(s,a)\prns{Q^-(s,a)-\Tcal_{\cvar}^-Q^-(s,a)}]
    \\&= V_{d_1}^- + 0 = V_{d_1}^-,
\end{align*}
where the inequality comes from the fact that $\beta$ is sub-optimal for $\EE[\beta(s,a)+\tau^{-1}(V^-(s')-\beta(s,a))_-]$. The same proof applies for $Q^+$.
\end{proof}

We now prove \pref{lem:Tbeta-validity}.
First, we show that the $\Tcal_\beta$ perspective gives rise to a dual definition of $Q^\pm$ (dual to \pref{eq:robust-q-primal-def}).
\begin{lemma}\label{lem:robust-q-dual-def}
\begin{equation}
    \textstyle Q^+(s,a) = \argmin_{\beta: Q_\beta=\Tcal_\beta^+ Q_\beta}Q_\beta(s,a), % \label{eq:robust-q-dual-def}
    % \\&
    \quad
    Q^-(s,a) = \argmax_{\beta: Q_\beta=\Tcal_\beta^- Q_\beta}Q_\beta(s,a). \nonumber
\end{equation}
\end{lemma}
\begin{proof}
Unroll $Q^-(s,a)=r(s,a)+\gamma\inf_{U\in\Ucal(P)}\EE_U[ r(s',a')+\gamma\inf_{U\in\Ucal(P)}\EE_U[\dots] ]$, replacing each $\inf_{U\in\Ucal(P)}$ with the convex combination of $\EE$ and $\cvar$ from \pref{lem:identification-q}. Then, write each $\cvar$ using the dual form, \ie, $\max_\beta\{\beta(s,a)+\tau^{-1}(s,a)\EE[ (\dots - \beta(s,a))_+ ]\}$. By $s,a$-rectangularity, the scalar $\max_{\beta}$ separates per $s,a$, so we can pull all the maxes out front as a $\max$ over $\beta(s,a)$ functions. 
Note that not all $\beta(s,a)$ functions have a well-defined infinite sum in this manner, as $\Tcal_\beta$ is not always a contraction. 
The condition $Q_\beta=\Tcal_\beta^-Q_\beta$ exactly characterizes when this unrolling is well-defined.
Thus, $Q^-$ is exactly the minimum $Q_\beta$ whenever this procedure of unrolling with $\beta$ is well-defined. This concludes the proof.
\end{proof}

\TbetaValidity*
\begin{proof}
\begin{align*}
% &\EE[\psi(s,a,s';Q_\beta^+,\beta,w)] \geq (1-\gamma)\EE_{d_1}[ V_\beta^+(s_1) ] + 0 \geq V_{d_1}^+,
\\&\EE[\psi(s,a,s';Q_\beta^-,\beta,w)] = (1-\gamma)\EE_{d_1}[ V_\beta^-(s_1) ] + 0 \leq V_{d_1}^-.
\end{align*}
The first equality is because the correction term is $\Tcal_\beta^- Q_\beta^- - Q_\beta^-$, which is zero since $Q_\beta^-$ is a fixed point. The inequality is due to \pref{lem:robust-q-dual-def}.
\end{proof}

\section{Additional Details for Main Experiment}
\label{apx:experiment}

\subsection{Environment}

We consider a simple MDP with a one-dimensional state space $\Scal = [0,5]$, a binary action space $\Acal = \{0,1\}$, reward function
$$r(s,a) = \frac{26 - s^2 - \I{a = 1} }{26}\,,$$
which we note takes values in the range $[0,1]$, and with transitions given by
\begin{align*}
    P(\cdot \mid s,a=0) &= \text{UnifClip}[s-0.2, \ s + 1] \\
    P(\cdot \mid s,a=1) &= \text{UnifClip}[0.2s-0.02, \ s + 0.5] \,,
\end{align*}
where $\text{UnifClip}[a,b]$ denotes a uniform distribution between $\max(a,0)$ and $\min(b,5)$. In addition, the environment always starts in initial state $s_0=2$. Essentially, this is a simple control environment, where high rewards are obtained by maintaining state as close to zero as possible, the action $a=1$ is a control action that (in expectation) moves the state closer to zero, and which occurs a small reward cost, and the action $a=0$ is a passive action that allows the state to freely drift (with an overall drift away from zero). 

\subsection{Target Policy}

We focus on estimating the worst-case policy value $V^-_{d_1}$ for the simple threshold-based target policy $\pit$ which takes action $a=1$ when $s \geq 2$, and $a=0$ whenever $s<2$.

\subsection{Logging Policy and Data Sampling Procedure}

We sample data using an evaluation policy $\pi_b$ which is an $\epsilon$-smoothed threshold policy similar to $\pi_t$. Specifically, $\pi_b$ takes action $a=1$ when $s \geq 1.5$ with probability 0.95, and takes action $a=0$ when $s < 1.5$ with probability 0.95. We obtain a dataset $\{s_i,a_i,s'_i,r_i\}$ by first rolling out with $\pi_b$ for 1000 burn-in time steps, and then sampling the tuple $(s,a,s',r)$ every 10 time steps. For each replication of our experiment, we sample 10,000 tuples in total.

\subsection{Calculation of True Worst-Case Policy Values}
\label{apx:sepsis}

A major challenge in studying robust policy value estimation is that, even with ground truth knowledge of the MDP and/or access to a simulator, it may be intractable to estimate the robust policy values $V_{d_1}^\pm$.
Fortunately, the above environment has the desirable property that we can analytically compute the best/worst-case transition distributions allowed by our sensitivity model, since no matter what policy $\pi_t$ the agent is acting with, it always strictly prefers transitions to smaller states. In detail, suppose that for some state, action pair $(s,a)$ we have $P(\cdot \mid s,a) = \text{Unif}[x, y]$, for some $0 \leq x \leq y \leq 5]$. Then, letting $\alpha = 1/(1+\Lambda(s,a))$, it is easy to verify that the worst case transition kernel is given by
\begin{align*}
    U^-(\cdot \mid s,a) &= (1 - \Lambda^{-1}(s,a)) \text{Unif}[y - \alpha(y-x),  y] + \Lambda^{-1}(s,a) \text{Unif}[x,y] \,.
\end{align*}
That is, the worst case transition kernel is given by a mixture of two uniform distributions. 
Therefore, we can easily simulate rollouts with the best/worst case transition kernels, and accurately estimate the robust policy values. This allows us to validate our methodology in this synthetic environment. Specifically, for each $\Lambda(s,a)$ we experiment with, we can compute the corresponding ground truth $V_{d_1}^-$ up to arbitrary precision via Monte Carlo sampling, by rolling out trajectories with $\pi_t$ in the adversarial MDP according to the above worst-case transition kernel.

Note as well that if one wanted to estimate the best-case policy value, analogous reasoning would give us 
\begin{align*}
    U^+(\cdot \mid s,a) &= (1 - \Lambda^{-1}(s,a)) \text{Unif}[x, x + \alpha(y-x)] + \Lambda^{-1}(s,a) \text{Unif}[x,y] \,.
\end{align*}
However, in our experiments we only concern ourselves with worst-case policy value estimation.

\subsection{Nuisance Estimation}

We instantiate slight variations of \cref{alg:robust-fqe,alg:robust-minimax} using neural nets for the classes $\Qcal$, $\Bcal$, and $\Wcal$ used for fitting $Q^-$, $\beta^-$, and $w^-$ respectively, and linear sieves for the corresponding critic class $\Qcal$ that we perform maximization over for the minimax estimation of $w^-$. Specifically, we grow the linear sieve for the critic class in a data-driven way, as follows: at each step $k$ of the respective algorithm, we compute the best response $q_k \in \Qcal$ to the previous iterate solution $w_k \in \Wcal$ by optimizing over a neural net class, and then we append this best-response function to the set of functions in our linear sieve for the corresponding critic class.
Full exact nuisance estimation details necessary for reproducibility will be available in our code release.

\subsection{Estimators}

We estimate the worst-case policy value using three different estimators:
\begin{itemize}
    \item \textbf{Q}: Direct estimator given by:
    $$\wh V^-_{d_1} = \wh Q^-(s_1,\pi_t(s_1))\,,$$
    where $s_1$ is the deterministic initial state.
    \item \textbf{W}: Importance sampling-style estimator using $\hat w^-$, which is given by:
    $$\wh V^-_{d_1} = \frac{1}{n} \sum_{i=1}^n \wh w^-(s_i,a_i) \wh \xi_i r_i\,,$$
    where
    $$\wh \xi_i = \Lambda^{-1} + (1-\Lambda^{-1}) (1+\Lambda) \I{\wh V^-(s_i') \leq \wh\beta^-(s_i,a_i)}\,.$$
    \item \textbf{Orth}: Our orthogonal estimator using EIF, given by
    $$\wh V^-_{d_1} = \frac{1}{n} \sum_{i=1}^n \psi(s_i,a_i,s'_i; \wh Q^-, \wh\beta^-, \wh w^-)\,.$$
\end{itemize}

Note as well that we used a simpler data splitting procedure rather than the cross-fitting procedure described in \cref{alg:ortho-ope}. Specifically, we used the first 10,000 tuples for estimating nuisances, and the second 10,000 tuples for the final estimators. This was done for the sake of computational ease in running experiments with many replications, and was performed in the same way for all methods.

In addition, for extra robustness, in each experiment replication we ran the nuisance estimation pipeline 5 times (on the same fixed sampled dataset), and took the 80th percentile policy value estimates, since the estimators tend to under-estimate the true policy value by design, with greater under-estimation when the nuisance estimates are less well optimized.

\section{Empirical Investigation on Medical Application}\label{sec:medical-experiments}

Here, we describe an additional empirical investigation of our methodology on medical data. Specifically, we consider the problem of sepsis management using RL. For all parts of the investigation described below, fully complete details can be obtained from our code release.

\subsection{Motivation of Investigation}

Training RL models in simulated environments derived from real-world data is an exciting avenue for leveraging AI towards critical medical use cases. However, doing this obviously has the downside that, unless one undergoes the very risky process of training an RL agent online via real medical interventions, one has to resort to training within simulators, and then has to account for the inevitable ``sim-to-real'' gap. Therefore, our robust OPE methodology provides an interesting approach for estimating worst-case performance of RL models under potential changes in dynamics when moving to real application.

\subsection{RL Environment}

Our RL environment is based on the OpenAI Gym sepsis simulator environment of \citet{kiani2019sepsis}. This RL environment allows for simulation of dynamic sepsis management, which was created by training a blackbox ML model to mimic observed transition dynamics from the real-world electronic health record-based MIMIC-III dataset \citep{johnson2016mimic}. This existing sepsis simulator is an episodic environment that continues until the agent either recovers or dies. It has a 46-dimensional state space containing various vital measurements, a discrete action space containing 24 possible actions (where an action is essentially the Cartesian product of some independent base actions). The reward function in this original simulator gives zero reward whenever an episode has not terminated, a +15 reward at termination when if the patient survives, or a -15 reward at termination if the patient dies. Please see \citet{kiani2019sepsis} and the code release linked therein for additional details.

We built an RL environment for our investigation by creating a simple wrapper around this existing sepsis simulator, in order to make it fit our setup. In particular, we made the following key changes:
\begin{enumerate}
    \item We made the environment infinite-horizon, by automatically looping to a new random starting state for a fresh patient whenever the episode in the base simulator terminates
    \item We normalized the reward function so that it lies in range $[0,1]$, where:
    \begin{enumerate}
        \item $r(s,a) = 0$ if patient dies
        \item $r(s,a) = 1$ if patient recovers and is discharged
        \item $r(s,a) = 0.5$ if treatment has not terminated for current patient
    \end{enumerate}
\end{enumerate}

In addition, for this environment, we perform all experiments with $\gamma=0.95$.

\subsection{Policies for Investigation}

We constructed RL policies for our empirical investigation by training some deep RL models using the sepsis simulator environment.

In the case of the behavioral policy $\pi_b$ used to generate the observational offline data, we trained this policy by running Proximal Policy Optimization (PPO: \citet{schulman2017proximal}) over a relatively large (16,000) number timesteps, in order to emulate a reasonably good ``current best practices'' model for creating observational data.

In the case of the target policy $\pit$ to be evaluated, we trained this policy using Deep Q Learning (DQL: \citet{mnih2013playing}), over a relatively small (1,600) number of timesteps, in order to emulate a potentially risky new candidate model.

\subsection{Creating an Offline Dataset}

Using our behavioral policy $\pi_b$ which we created as above, we generated a fixed offline dataset consisting of 20,000 observed tuples of state, action, reward, and next state. Unlike with our main empirical investigation in the main paper, we did not perform any ``thinning'' on these sampled tuples to make them more independent, so that the observed transitions are sequentially correlated as with real-world medical data.

\subsection{Nuisance Estimation}

We perform nuisance estimation almost identically as in our main empirical investigation, with the only change being a slight change to our neural network architectures to better handle the large discrete action space. Specifically, instead of training neural networks that take state as input and produce $|\Acal|$ outputs (one per action), we train neural networks that take both state and action as inputs, using a learnt low-dimensional encoding of the actions, and produce a single output. Please see our code release for details.

\subsection{Estimators}

We consider the same three estimators (\textbf{Q}, \textbf{W}, and \textbf{Orth}) as in our main empirical investigation. As in that investigation, we use these to estimate the worst-case policy value for the given $\Lambda(s,a)$. In addition, as in the main experiments, we consider these estimators for various fixed $\Lambda(s,a)$ that do not depend on $s$ or $a$. In this case, we consider $\Lambda \in \{1,2,4\}$, as these reflect a reasonable range of possible confounding strength for real application.

\subsection{Results}

\begin{table}[t]
    \centering
    \begin{tabular}{cccc}
         \multirow{2}{*}{$\Lambda$} & \multicolumn{3}{c}{Median Policy Value Estimate} \\
         & \textbf{Q} & \textbf{W} & \textbf{Orth} \\
         \hline
            1 & $.546 \pm .003$ & $.386 \pm .087$ & $.532 \pm .008$ \\
            2 & $.454 \pm .040$ & $.534 \pm .141$ & $.515 \pm .036$  \\
            4 & $.381 \pm .077$ & $.287 \pm .106$ & $.338 \pm .086$  \\
         \hline
    \end{tabular}
    \caption{Median policy value estimate for sepsis management investigation, for each estimator and value of $\Lambda$ over 5 runs of each estimator from random initial seeds. The $\pm$ values are given by half the difference between 80th and 20th percentiles.}
    \label{tab:sepsis-results}
\end{table}

Below, in \cref{tab:sepsis-results} we show the estimated policy value for all three estimators for each fixed $\Lambda \in \{1,2,4\}$. Here, we present the median policy value estimate over 5 runs of our estimators from random starting seeds after removing outliers.\footnote{Specifically, we exclude policy value estimates that lie outside the possible range of $[0,1]$, which occasionally occur due to bad optimization from the starting seed.} In addition, we present a $\pm$ spread given by half the difference between the 80th and 20th percentiles.

Although for this investigation we cannot analytically compute the ground truth ``true'' adversarial policy values to evaluate against when $\Lambda>1$, we can still analyze the trends of these estimators and compare them to those observed in our main synthetic experiment, and we can also compare their accuracy when $\Lambda=1$.

First, in the case of $\Lambda=1$, we computed the true policy value of $\pit$ to be within the range $0.532 \pm 0.002$ with 95\% confidence. 
This is almost exactly equal to the median \textbf{Orth} estimator, but far outside the spread of outputs of the \textbf{Q} estimator. That is, although the \textbf{Q} estimator has somewhat lower variance in outputs over multiple runs for $\Lambda=1$ compared with \textbf{Orth}, it appears to be far more biased.

Next, looking more broadly across all values of $\Lambda$, as in our main experiment, the \textbf{Q} and \textbf{Orth} estimators generally result in similar estimates to each other, and the \textbf{W} estimators are very variable. This may reflect the relative difficulty of estimating the $w^-$ nuisance function compared with $Q^-$ and $\beta^-$; although both \textbf{Orth} and \textbf{W} are affected by this difficulty, the \textbf{Orth} estimator has a theoretical robustness to the errors of these nuisance functions that the \textbf{W} estimator does not, as outlined in our theory.

We also observe that when $\Lambda=1$ the \textbf{Q} estimator is significantly more stable than \textbf{Orth}, but when $\Lambda>1$ the stability of \textbf{Orth} is either comparable to or superior to \textbf{Q}.
In order to understand this, we first note that unlike in our main experiments, here the repetitions are re-runs of the estimators with the same offline sepsis dataset, so these $\pm$ spreads reflect potential computational errors rather than statistical errors.
Given this, this pattern of errors could be explained by the fact that when $\Lambda=1$ the \textbf{Q} estimation is extremely simple, reducing to standard FQI, whereas when $\Lambda>1$ it requires a more complex robust FQI estimation with simultaneous estimation of $\beta^-$. That is, the difference in computational difficulty of estimating \textbf{Orth} versus \textbf{Q} may be smaller for $\Lambda>1$.

Overall, although it is hard to definitively compare the accuracy of these estimators for $\Lambda>1$ given a fundamental lack of ground truth, given both a similar pattern of results as in our synthetic experiments, as well as the far greater accuracy of \textbf{Orth} when $\Lambda=1$, it seems reasonably to believe based on these results that our proposed \textbf{Orth} estimator may be more reliable than the existing robust FQI approach of the \textbf{Q} estimator. 

Finally, we consider the  implication of our results for the problem of learning sepsis management policies from simulators. Our \textbf{Orth} estimator suggests that there is relatively little sensitivity of this environment to deviations allowed by $\Lambda=2$, but very significant deviation allowed by $\Lambda=4$. Indeed, given the reward structure described above, the worst-case results under $\Lambda=4$ imply an extremely high mortality rate. Whether worst-case deviations of this magnitude are reasonable or not is unclear, and this is something that requires further investigation for future work on RL for sepsis management.

%%%%%%%%%%%%%%%%%%%%%%%%%%%%%%%%%%%%%%%%%%%%%%%%%%%%%%%%%%%%

\newpage
\section*{NeurIPS Paper Checklist}

\begin{enumerate}

\item {\bf Claims}
    \item[] Question: Do the main claims made in the abstract and introduction accurately reflect the paper's contributions and scope?
    \item[] Answer: \answerYes{}.
    \item[] Justification: Yes, we provide complete proofs for our theorems and describe detailed empirical validation for our proposed algorithms.

\item {\bf Limitations}
    \item[] Question: Does the paper discuss the limitations of the work performed by the authors?
    \item[] Answer: \answerYes{}. 
    \item[] Justification: Yes, we discussed where our assumptions may fail and settings not captured by the current paper, which we believe are directions for future research.

\item {\bf Theory Assumptions and Proofs}
    \item[] Question: For each theoretical result, does the paper provide the full set of assumptions and a complete (and correct) proof?
    \item[] Answer: \answerYes{} 
    \item[] Justification: Yes, we provide full assumptions in the main paper and the complete proofs are written in the Appendix.

    \item {\bf Experimental Result Reproducibility}
    \item[] Question: Does the paper fully disclose all the information needed to reproduce the main experimental results of the paper to the extent that it affects the main claims and/or conclusions of the paper (regardless of whether the code and data are provided or not)?
    \item[] Answer: \answerYes{}. 
    \item[] Justification: Yes, our experimental section includes all details needed to reproduce the main experimental results. 

\item {\bf Open access to data and code}
    \item[] Question: Does the paper provide open access to the data and code, with sufficient instructions to faithfully reproduce the main experimental results, as described in supplemental material?
    \item[] Answer: \answerYes{}. 
    \item[] Justification: Yes, our code is open-sourced at \url{https://github.com/CausalML/adversarial-ope/}.
    
\item {\bf Experimental Setting/Details}
    \item[] Question: Does the paper specify all the training and test details (e.g., data splits, hyperparameters, how they were chosen, type of optimizer, etc.) necessary to understand the results?
    \item[] Answer: \answerYes{} 
    \item[] Justification: Yes, please see our experimental section and appendices for all training and evaluation details. 
    \item[] Guidelines:

\item {\bf Experiment Statistical Significance}
    \item[] Question: Does the paper report error bars suitably and correctly defined or other appropriate information about the statistical significance of the experiments?
    \item[] Answer: \answerYes{}. 
    \item[] Justification: Yes, our experiments are replicated over multiple seeds and we report the confidence intervals.

\item {\bf Experiments Compute Resources}
    \item[] Question: For each experiment, does the paper provide sufficient information on the computer resources (type of compute workers, memory, time of execution) needed to reproduce the experiments?
    \item[] Answer: \answerYes{}. 
    \item[] Justification: Yes, this paper is mostly focused on theory and our experiment is a proof of concept and can be run on a standard GPU.
    
\item {\bf Code Of Ethics}
    \item[] Question: Does the research conducted in the paper conform, in every respect, with the NeurIPS Code of Ethics \url{https://neurips.cc/public/EthicsGuidelines}?
    \item[] Answer: \answerYes{}. 
    \item[] Justification: Yes, we have reviewed the code of ethics and believe our research conforms to it.

\item {\bf Broader Impacts}
    \item[] Question: Does the paper discuss both potential positive societal impacts and negative societal impacts of the work performed?
    \item[] Answer: \answerNA{}.
    \item[] Justification: This paper is about foundational research not tied to particular applications so we do not feel the need to highlight any societal impacts.
    
\item {\bf Safeguards}
    \item[] Question: Does the paper describe safeguards that have been put in place for responsible release of data or models that have a high risk for misuse (e.g., pretrained language models, image generators, or scraped datasets)?
    \item[] Answer: \answerNA{}.
    \item[] Justification: This paper is about foundational research not tied to particular applications so we do not feel the need to highlight any risks for misuse here.
    
\item {\bf Licenses for existing assets}
    \item[] Question: Are the creators or original owners of assets (e.g., code, data, models), used in the paper, properly credited and are the license and terms of use explicitly mentioned and properly respected?
    \item[] Answer: \answerNA{}.
    \item[] Justification: The paper does not use any existing assets.

\item {\bf New Assets}
    \item[] Question: Are new assets introduced in the paper well documented and is the documentation provided alongside the assets?
    \item[] Answer: \answerNA{}.
    \item[] Justification: The paper does not release new assets.

\item {\bf Crowdsourcing and Research with Human Subjects}
    \item[] Question: For crowdsourcing experiments and research with human subjects, does the paper include the full text of instructions given to participants and screenshots, if applicable, as well as details about compensation (if any)? 
    \item[] Answer: \answerNA{}.
    \item[] Justification: The paper does not involve crowdsourcing nor research with human subjects.

\item {\bf Institutional Review Board (IRB) Approvals or Equivalent for Research with Human Subjects}
    \item[] Question: Does the paper describe potential risks incurred by study participants, whether such risks were disclosed to the subjects, and whether Institutional Review Board (IRB) approvals (or an equivalent approval/review based on the requirements of your country or institution) were obtained?
    \item[] Answer: \answerNA{}.
    \item[] Justification: The paper does not involve crowdsourcing nor research with human subjects.

\end{enumerate}

\end{document}